%% file: main.tex
\documentclass[11pt]{article}
\usepackage{amsmath}
\usepackage{amsthm,amsmath,amssymb}
\usepackage{natbib}
\usepackage{multirow}
\usepackage[pdftex]{graphicx}
\usepackage{subfigure}
\usepackage{makecell}
\usepackage{booktabs}
\usepackage{array}
\usepackage{fullpage}
\usepackage{url}
\usepackage{algorithm}
\usepackage{algorithmic}
\usepackage{bm}
\usepackage{smile}
\usepackage{wrapfig}
\usepackage{lipsum}
\usepackage{mathrsfs}
\usepackage{dsfont}
\usepackage{titling}
\usepackage{epstopdf}
\usepackage{mathtools}
\usepackage{helvet}

\usepackage{natbib}
\usepackage{multirow}
\usepackage{tikz}
\usepackage[colorlinks=true,
            linkcolor=red,
            urlcolor=blue,
            citecolor=blue]{hyperref}

\providecommand{\norm}[1]{\|#1\|}

\makeatletter
\newcommand*\rel@kern[1]{\kern#1\dimexpr\macc@kerna}
\newcommand*\widebar[1]{%
  \begingroup
  \def\mathaccent##1##2{%
    \rel@kern{0.8}%
    \overline{\rel@kern{-0.8}\macc@nucleus\rel@kern{0.2}}%
    \rel@kern{-0.2}%
  }%
  \macc@depth\@ne
  \let\math@bgroup\@empty \let\math@egroup\macc@set@skewchar
  \mathsurround\z@ \frozen@everymath{\mathgroup\macc@group\relax}%
  \macc@set@skewchar\relax
  \let\mathaccentV\macc@nested@a
  \macc@nested@a\relax111{#1}%
  \endgroup
}
\makeatother

\newcommand{\transpose}{\scriptscriptstyle \sf T}
\def\trans{^{\transpose}}

\newcommand{\dd}{\mathrm{d}}

\theoremstyle{definition}
\newtheorem{thm}{Theorem}[section]
\newtheorem{assump}{Assumption}[section]

\newtheorem{rmk}[thm]{Remark}

\newtheorem{crl}[thm]{Corollary}
\newtheorem{lem}[thm]{Lemma}
\newtheorem{exa}[thm]{Example}

\begin{document}

%\title{\huge SPARC: Optimal Estimation and Asymptotic Inference under Semiparametric Sparsity }
\title{\LARGE  Confidence Diagram of Nonparametric Ranking for Uncertainty Assessment in Large Language Models Evaluation}

\author{
    Zebin Wang$^{1, *}$ \ 
    Yi Han$^{2, *}$ \ 
    Ethan X. Fang$^{3}$ \ 
    Lan Wang$^{4}$ \
    Junwei Lu$^{1}$  \ 
}

\footnotetext[1]{Department of Biostatistics, Harvard Chan School of Public Health}
\footnotetext[2]{Department of Statistics, Columbia University}
\footnotetext[3]{Department of Biostatistics \& Bioinformatics, Duke University}
\footnotetext[4]{Department of Management Science, University of Miami}
\renewcommand{\thefootnote}{\fnsymbol{footnote}}
\footnotetext{* Equal contribution}

\date{}

\maketitle

\vspace{-0.5in}

\begin{abstract}
  We consider the inference for the ranking of large language models (LLMs). Alignment arises as a significant challenge to mitigate hallucinations in the use of LLMs. Ranking LLMs has proven to be an effective tool to improve alignment based on the best-of-\(N\) policy. In this paper, we propose a new inferential framework for hypothesis testing among the ranking for language models. Our framework is based on a nonparametric contextual ranking framework designed to assess large language models' domain-specific expertise, leveraging nonparametric scoring methods to account for their sensitivity to the prompts. To characterize the combinatorial complexity of the ranking, we introduce a novel concept of confidence diagram, which leverages a Hasse diagram to represent the entire confidence set of rankings by a single directed graph. We show the validity of the proposed confidence diagram by advancing the Gaussian multiplier bootstrap theory to accommodate the supremum of independent empirical processes that are not necessarily identically distributed. Extensive numerical experiments conducted on both synthetic and real data demonstrate that our approach offers valuable insight into the evaluation for the performance of different LLMs across various medical domains.
\end{abstract}


\input{./content/intro.tex}

\input{./content/method.tex}

\input{./content/theory.tex}

\input{./content/numerical.tex}

\bibliographystyle{ims}
\bibliography{/Users/yhan/Dropbox/2023_LongitudinalRanking/bib.bib}

\newpage

\appendix

\input{./content/appendix.tex}

\end{document}

%% file: content/intro.tex
% !TEX root = ../main.tex

\section{Introduction}
Recently, the tremendous advancement of large language models (LLMs) stimulates a broad spectrum of new applications such as automated customer service \citep{LIU2023127} and content generation \citep{NEURIPS2020_1457c0d6}. Meanwhile, as LLMs become more popular in numerous applications, we observe some instances of hallucinations, which may potentially lead to serious consequences and limit LLMs' applications, especially in trust-sensitive areas such as vision \citep{gunjal2024detecting}, laws \citep{dahl2024large} and medicine \citep{bhayana2024chatbots}. In particular, hallucinations are events in which LLMs produce coherent and grammatically correct but factually incorrect output due to the models' reliance on patterns in training data rather than factual accuracy. As a result, one of the main challenges in LLM implementation is to align the output with human values and expectations, referred to as AI alignment \citep{tonmoy2024comprehensive}. A popular and well-performing method for enhancing AI alignment is the best-of-$N$ policy \citep{nakano2022webgpt, touvron2023llama, stiennon2024learning}, where we draw outputs for the same prompt from each of the $N$ models, rank them based on a reward function, and return the one with the highest reward. However, it remains unclear how to determine the models with the highest rewards, as current methods rely on heuristics without theoretical guarantees, which undermines the reliability of rankings. Therefore, it is critical to develop a systematic approach to rank language models based on rewards that capture human preferences \citep{ouyang2022training}.

Beyond alignment, ranking LLMs in specific domains helps in selecting models with superior domain knowledge, especially for applications that prioritize reliability and trustworthiness. However, static LLM ranking models that fails to capture the contextual variability of downstream tasks can lead to the selection of suboptimal models. Such a mismatch may result in significant losses in high-stakes domains. In the legal field, while LLMs can help in analyzing complex legal documents and provide insights that align with prevailing laws \citep{dahl2024large, colombo2024saullm7b}, their susceptibility to errors poses substantial risks.  Notably, \citet{dahl2024large} found the hallucinate rates are notably higher in contexts involving lower levels of the federal judiciary and in more conservative Circuits of Appeals. In the medical field, studies \citep{pal2023med,yan2024med} show that the hallucinate rates are significantly different across different medical domains. These findings raise both legitimate and scientific concerns in LLMs' ability to adapt to different contexts. Addressing these limitations calls for the development of a context-aware LLM ranking model. Such a model is expected to identify the most reliable LLMs for specific legal tasks, enabling these supportive tools to complement the domain experts in their workflows.

Furthermore, uncertainty quantification is critical to ensure integrity and reproducibility in using language models. In the medical domain, unstable LLM rankings may lead to inconsistent choice of language models and potentially provide misleading clinical diagnoses \citep{nori2023capabilities}. On the other hand, consistent LLM ranking enables decision makers to select the most suitable models and further guide the fine-tuning for future language models in medicine \citep{ji2023mitigating, tian2023finetuning}. While previous studies have made significant progress in estimating the rankings for different models \citep{shah2016stochastically,chen2018optimal,pananjady2017worst}, inference on contextual ranking models remains largely unexplored. Conducting such an inference task is considered challenging due to the nonparametric nature of the score functions, which introduces additional challenges in inference, such as estimating the quantiles of the score functions, as we will discuss later.

In general, besides ranking LLMs, ranking problems find many applications, such as sports competition \citep{xia2018network, pelechrinis2016sportsnetrank}, web search and information retrieval \citep{bouadjenek2013sopra,guo2020deep}, gene ranking \citep{kolde2012robust, kim2015hydra}, among many others. Given the wide applications, various models have been developed to estimate the rankings. For instance, the Thurstone model for paired comparisons \citep{thurstone1927law} and the Plackett-Luce model for rankings \citep{plackett1975analysis}. In our work, we extend the celebrated Bradley-Terry-Luce (BTL) model \citep{bradley1952rank, luce1959individual}, one of the most widely used parametric models for ranking. In the classical BTL model, each item has a fixed latent preference score that determines its rank among all items. However, assuming each LLM has a fixed score is not enough in the context of language models. This is because the performance of each model varies substantially for different domain-specific prompts due to different training and tuning processes, and their performances vary over time. The classical BTL model with fixed scores across all prompts fails to capture this contextual variability. Additionally, the parametric score functions, such as linear models, have a limited ability to accommodate the variability in performance across prompts in different domains. Therefore, it is crucial to develop a new ranking model that captures the varying contexts of the prompts, referred to as the contextual ranking model.

Motivated by these challenges, we consider a contextual ranking model and an inferential framework with nonparametric score functions to provide reproducible and trustable evaluations of LLMs. In particular, given $n$ models, we define the true latent preference score function $\theta^*_{i}(\xb)$ over the set of $n$ models, and prompt $\xb$ within a domain $\Omega$. Let $\Theta^* = \{\theta_{i}^*(\xb)| i \in [n], \xb \in \Omega\}$, where $[n] = \{1,\ldots,n\}$.  We aim to test whether $\Theta^*$ satisfies some properties of interest based on partial pairwise comparison observations. 
Specifically, 
we test for $\Theta^*$ on a general set of properties that
$${H}_{0}: \Theta^* \text{ does not satisfy the property}   \text{ \ \ v.s. \ \  } {H}_{a}: \Theta^* \text{  satisfies the property}.$$

We develop a nonparametric model based on kernel smoothing methods and estimate the preference scores through a regularized maximum likelihood estimation. Due to the nonparametric nature of the score functions in our proposed model, most quantile estimation methods, which rely on closed-form statistics, are not applicable. Therefore, we extend the Gaussian multiplier bootstrap \citep{chernozhukov2013gaussian} to estimate the supremum of non-identically distributed empirical processes. Previous works \citep{chernozhukov2013gaussian, chernozhukov2014central, chernozhukov2022improved} show that the distribution of the maximum of a sum of random vectors with unknown covariance matrices can be consistently estimated by the distribution of the maximum of a sum of conditional Gaussian random vectors obtained by the Gaussian multiplier bootstrap method. These works typically address the finite maximum of identically distributed processes. However, in our context, the score functions are independent across different models but not identically distributed. Additionally, in the application of LLMs, the continuous nature of the prompt domain requires analysis of a supremum over a continuous space rather than a finite maximum.
Consequently, classical Gaussian multiplier bootstrap methods and theories cannot be applied directly. To address this gap, we propose a novel approach to discretize the continuous prompt space. This discretization step enables the application of the Gaussian multiplier bootstrap theory to a finite-dimensional process. We further show that the error introduced by discretization can be theoretically derived and controlled by the complexity of a tube-like structure.

Furthermore, to provide a global understanding of the combinatorial relationships and performance hierarchies among LLMs, we introduce a novel concept in ranking inference, called the {confidence diagram}. This diagram is inspired by the celebrated Hasse diagram, commonly used in the field of lattice theory and order theory \citep{birkhoff1940lattice}, which is a mathematical tool for visualizing a finite partially ordered set. The confidence diagram identifies a subset of potential orders that includes the true ranking order, providing a visual and analytical representation of possible rankings. The confidence diagram-based approach improves our understanding of the relative performance and relationships among different LLMs.

\subsection{Motivating Applications}
The inference in ranking problems finds many applications. For example, when two large language models, such as \textsc{Llama 2} \citep{touvron2023llama} and \textsc{GPT-3}, are available for use, it is practically important for legal professionals to determine whether \textsc{GPT-3} is preferred over \textsc{Llama 2} in their specific legal practice area. This scenario represents a typical pairwise ranking inference problem, which aligns with the framework that we propose and will discuss further below.

\begin{exa}[Pairwise Inference]
     Consider a pairwise test that evaluates whether the true preference score of a model is greater than that of another model over all prompts within a specific domain. This can be formulated as the following hypothesis testing problem:
    $${H}_{0}: \text{Model }i \text{ is ranked lower than model } j \text{ for {\bf some} prompts in the domain.}   $$ $$ {H}_{a}: \text{Model }i \text{ is ranked higher than model } j \text{ across {\bf all} prompts in the domain}.$$
\end{exa} 

Another important application of inference in ranking problems is the top-$K$ inference. As discussed previously, the best-of-$N$ policy \citep{nakano2022webgpt, touvron2023llama, stiennon2024learning} identifies the model with the highest reward to enhance AI alignment. The top-$K$ inference is essential in the best-of-$N$ policy to solve the problem of hallucination in LLMs. In addition, top-$K$ inference is crucial for identifying the top-ranked models in a domain to facilitate trustworthy decision-making. For instance, in medical diagnostics, only a small subset of top-ranked language models is informative, making it essential to identify them for reliable results. Consider a scenario where medical practitioners want to identify the top two performers from the available model pool including \textsc{GPT-4}, \textsc{BERT} \citep{devlin2018bert}, \textsc{T5}, and \textsc{Gemini}. Selecting the top-performing models ensures that unreliable information is excluded while enabling the medical practitioners' access to accurate supportive information generated by LLMs. We summarize the top-$K$ inference problem in the following example.

\begin{exa}[Top-$K$ Inference]
    We test whether the true preference score of a model is among the top-$K$ models across all prompts in a specific domain. This is equivalent to the following hypothesis testing problem:
    $${H}_{0}: \text{Model }i \text{ is not among the top-}K \text{ models for {\bf some} prompts in the domain.}   $$ $$ {H}_{a}: \text{Model }i \text{ is among the top-}K \text{ models across {\bf all} prompts in the domain}.$$
\end{exa}

Furthermore, we introduce a global ranking property inference framework. This framework identifies a set of ranking orders that includes the actual ranking order with a desired confidence level, providing deeper insights into the combinatorial structures of the ranking system. For instance, we consider the task of determining the overall ranking of LLMs in linguistics, such as \textsc{Whisper} \citep{radford23a}, \textsc{AudioPaLM2} \citep{rubenstein2023audiopalm}, \textsc{BigTranslate} \citep{yang2023bigtranslate}, and \textsc{SeamlessM4T} \citep{communication2023seamlessm4t}. Inferring the global ranking provides LLM developers with insight on areas for refinement in their tuning process. We formalize the global ranking property inference problem as follows.

\begin{exa}[Confidence Diagram] 
  Let $\theta^*$ denote the true preference score, and $R(\theta^*)$ denote the true ranking based on $\theta^*$. Consider the problem of identifying a set of ranking orders that covers  $R(\theta^*)$. To achive this, we construct a confidence diagram $  \widehat{\mathbf{H}} $ with level $1-\alpha$ such that
    $$P\left(R(\theta^*) \in  \widehat{\mathbf{H}} \right) \geq 1-\alpha,$$
    where $\widehat{\mathbf{H}}$ is a set of possible orders represented by a Hasse diagram.

\end{exa}

\subsection{Major Contributions}
To the best of our knowledge, this work provides the first nonparametric inferential framework for ranking systems with applications to LLMs. Our proposed method potentially provides a foundation for more transparent, reliable, and reproducible evaluations of LLMs, ensuring their alignment across diverse application domains. The key contributions of this paper include
\begin{itemize}
    \item Methodologically, we develop a novel contextual ranking model tailored for evaluating LLMs' domain-specific knowledge, addressing their contextual sensitivity with nonparametric score functions.  
    \item We introduce the confidence diagram, a new tool for visualizing ranking inference based on a Hasse diagram, which facilitates a more thorough understanding of the combinatorial relationships and performance hierarchies among LLMs. 
    \item Theoretically, we develop the first inferential framework for ranking LLMs with nonparametric score functions, aiming for more transparent and reproducible evaluations, filling a gap in the previous literature that focuses solely on parametric models.
    \item We advance the Gaussian multiplier bootstrap theory to the supremum of independent but not identically distributed empirical processes, bridging a crucial gap in existing work while offering an independent theoretical interest.
\end{itemize}

\subsection{Literature Review} 

A series of papers have studied the estimation or inference of different ranking models.  \citet{negahban2016rank}, \citet{azari2013generalized}, \citet{maystre2015fast}, and \citet{jang2018top} develop new algorithms for estimating the ranking systems. In particular, under parametric setup, \citet{negahban2016rank} develop a new method called \textit{rank centrality}, which uses an iterative process based on the $\ell_2$-loss to estimate a set of normalized scores. Expanding on this area, \citet{chen2019spectral} study how well one can estimate these scores in terms of the $\ell_{\infty}$-error. They prove that both spectral methods and regularized maximum likelihood estimators achieve the minimax optimal rate of convergence. Further, \citet{chen2020partial} show that the optimal rate of convergence can be obtained without regularization on the likelihood function.

Beyond parametric methods, several works consider nonparametric methods for ranking problems. For example, \citet{shah2017simple} analyze a nonparametric counting algorithm, showing its optimality and robustness for rank estimation. \citet{shah2016stochastically}, \citet{chen2018optimal}, and \citet{pananjady2017worst} consider the strong stochastically transitive model for pairwise comparisons. However, all the above works focus solely on the estimation problem and do not consider uncertainty quantification and inferential methods in ranking.

Following work on estimation, several works make recent progress on uncertainty quantification in more general ranking models. For example, \citet{han2020asymptotic} show the asymptotic normality of the maximum likelihood estimator of the BTL model. \citet{gao2023uncertainty} introduce a \textit{leave-two-out} technique to derive the asymptotic distributions for the estimators of ranking scores, revealing the asymptotic normality of both the maximum likelihood and spectral estimators. Meanwhile, \citet{liu2022lagrangian} introduce a Lagrangian debiasing approach to derive asymptotic distributions for ranking scores that accommodate sparse comparison graphs with $p \asymp 1/n$, and also study the novel combinatorial inference to infer the combinatorial properties of ranking systems. 
\citet{Fan2024ranking} study the ranking problem for the multiway comparison model that accounts for the uncertainty in the ranking process based on multiway comparisons. \citet{han2023unified} further extend the settings in \citet{Fan2024ranking} by considering the setting where the comparisons are generated from a mixture of Erdős-Rényi graphs with different sizes or a hypergraph stochastic block model. \citet{fan2024spectral} conduct a comprehensive study for the performance of the spectral method in preference score estimation and provide insights into the asymptotic distribution of the estimated scores. Additionally, they explore one-sample and two-sample inferences on ranks. However, a crucial limitation of all the above-mentioned inferential works is the assumption that the scores of compared items are fixed and the models do not incorporate the items' variable features. An exception is \citet{fan2024uncertainty}, where the authors extend existing research by linearly incorporating covariate information into the BTL model. In contrast, we provide a more general framework that allows the scores of compared items to be nonparametric functions that vary with their attributes.

\subsection{Paper Organization} The rest of our paper is organized as follows. In Section~\ref{sec:setup}, we introduce the preliminary setup for the contextual ranking model under the BTL framework. In Section~\ref{sec:method}, we present our score estimation method and the confidence band for the proposed estimator. We then provide two hypothesis testing methods and a confidence diagram as a global ranking property. In Section~\ref{sec:theory}, we offer a theoretical guarantee for the convergence rate of the estimator and show the validity of the inferential procedure. We evaluate the performance of our method through numerical results in Section~\ref{sec:numerical}.

\subsection{Notations} Let $[n]$ represent the set of $\{1, \ldots, n\}$ for $n \in \mathbb{Z}^{+}$. For vector $v=\left(v_{1}, \ldots, v_{d}\right)^\top \in \mathbb{R}^{d}$ and $1 \leq q \leq \infty$, we define norm of $v$ as $\|v\|_{q}=\big(\sum_{i=1}^{d}\left|v_{i}\right|^{q}\big)^{1 / q}$. In particular,
$\|v\|_{\infty}=\max _{1 \leq i \leq d}\left|v_{i}\right|$. For a matrix $M=\left[M_{ij}\right]$, 
let $\ell_1$-norm $\|M\|_{1}=\max _{j} \sum_{i}\left|M_{ij}\right|$, $\ell_\infty$-norm $\|M\|_{\infty}=\max _{i} \sum_{j}\left|M_{i j}\right|$, and the operator norm
$\|M\|_{2}=\sigma_{\max }(M)$
where $\sigma_{\max }(M)$ represents the largest singular value of matrix $M$. In addition, $a_n=O(b_n)$ or $a_n \lesssim b_n$ means there exists a constant $C>0$ such that $a_n \leq Cb_n$, and $a_n=o(b_n)$ means
$\lim_{n \rightarrow \infty} \frac{a_n}{b_n}=0$. we write $a_{n} \asymp b_{n}$ if $C \leq a_{n} / b_{n} \leq C^{\prime}$ for some $C, C^{\prime}>0$. {For a sequence of random
variables $\{X_n\}$, we write $X_n \overset{d}{\longrightarrow} X$ if $X_n$ converges in distribution to the random variable $X$. }

%%%%%%%%%%%%%%%%%%%%%%%%%%%%%%%%%%%%%%%%%%%%%%
%% Single Appendix:                         %%
%%%%%%%%%%%%%%%%%%%%%%%%%%%%%%%%%%%%%%%%%%%%%%
%\begin{appendix}

%% file: content/method.tex
% !TEX root = ../main.tex

\section{Model and Problem Formulation}\label{sec:setup}

We discuss the model and problem formulation of our contextual ranking task. As we discussed in the introduction of our study, we extend the celebrated BTL pairwise comparison model to a contextual framework. We first introduce some important terminologies and subsequently discuss the data-generating mechanism.

\subsection{Preference scores} 
We let $\bX = [0,1]^d$ be the entire prompt domain and let the working prompt domain be $\Omega \subset \bX$. In the context of contextual ranking models, given a prompt $\mathbf{x}\in\Omega$, language model $i$ is assigned a true latent score $\omega_i^*(\bx)$.  We denote the true preference scores for the $n$ models as
$
    \boldsymbol{\omega}^*(\mathbf{x}) = \big(\omega^*_1(\mathbf{x}), \omega^*_2(\mathbf{x}), \ldots, \omega^*_n(\mathbf{x})\big).
$
We assume that the scores are bounded and positive such that for all $i\in[n]$, 
$
    \omega_i^*(\mathbf{x}) \in [\omega_{\text{min}}^*, \omega_{\text{max}}^*], \text{ and }\omega_{\text{min}}^*, \omega_{\text{max}}^*> 0.
$
Here, if $\omega_i(\bx) > \omega_j(\bx)$, it implies that model $i$ provides better response than model $j$ for prompt $\bx$.
Then, we let the ratio of the condition number $\kappa$ be the ratio of the upper bound to the lower bound that
$
\kappa = {\omega^*_{\max}}/{\omega^*_{\min} }.
$
To facilitate our analysis, we take the logarithm transformation of the scores. In particular, letting $\theta_i^*(\mathbf{x}) = \log w_i^*(\mathbf{x})$, we define the logged preference scores as
$
\boldsymbol{\theta}^*(\mathbf{x}) = \big(\theta^*_1(\mathbf{x}), \theta^*_2(\mathbf{x}), \ldots, \theta^*_n(\mathbf{x})\big).
$
Since the logarithm is monotone, it preserves the ordering of the $n$ models.
\subsection{Comparison random graph}
In practice, comparing all possible pairs of LLMs of interest is infeasible due to the exponential growth in the number of pairs, totaling $2^n$ for $n$ models. Instead, we typically observe pairwise comparison results for a subset of all possible pairs. Following earlier works \citep{negahban2016rank, pmlr-v32-rajkumar14, NIPS2014_f0e52b27}, we assume that the pairwise comparisons between models follow an Erdős–Rényi random graph $\mathcal{G} = (\mathcal{V}, \mathcal{E})$. This graph is characterized by its vertex set $\mathcal{V} = [n]$ and its edge set $\mathcal{E}$, which is defined as
\[
\mathcal{E} := \{(i,j) : i < j \text{ and the $i^{\text{th}}$ and $j^{\text{th}}$ vertices are connected}\}.
\]
On the graph, if there is an edge between item $i$ and item $j$, we have a pairwise comparison of model~$i$ and model $j$. 
In the Erdős–Rényi model $\mathcal{G} = (n,p)$, each potential edge $(i,j)$ between a pair of vertices presents in the graph with probability $p$ independently.

\subsection{Pairwise comparisons}
For any pair $(i,j)$ in $\mathcal{E}$, we observe $L_{ij}$ independent comparisons between model $i$ and model $j$, and let $y_{ij}^\ell(\mathbf{x})$
be the $\ell$-th outcome of the comparison. Here, $y_{ij}^\ell(\mathbf{x}) = 1$ indicates that model $j$ is preferred over model $i$ for the prompt $\mathbf{x}$ in the $\ell$-th comparison. In particular, for any models $i$ and $j$, the outcome of their comparison, $y_{ij}^\ell(\mathbf{x})$ follows a Bernoulli distribution that
\[
 y_{ij}^\ell(\mathbf{x}) = \begin{cases}
 1, \text{ with probability }   \frac{\exp{\{\theta_j^*(\mathbf{x})\}}}{\exp{\{\theta_i^*(\mathbf{x})\}} + \exp{\{\theta_j^*(\mathbf{x})\}}}, \\
 0, \text{ with probability } \frac{\exp{\{\theta_i^*(\mathbf{x})\}}}{\exp{\{\theta_i^*(\mathbf{x})\}} + \exp{\{\theta_j^*(\mathbf{x})\}}}.
 \end{cases}
\]

The distribution of the comparison prompts over $\Omega$ is denoted as $f_{\bX}$. For any model pair $(i,j)$, we denote a collection of independent prompts $\{\bX_{ij} ^ \ell\}_{\ell = 1} ^ {L_{ij}}$ that is distributed according to $f_{\bX} (\cdot)$, with $L_{ij}$ denoting the number of comparisons between  pair $(i,j)$. Thus, we have the observed data, denoted as $\mathcal{D}$, where
\begin{equation*}
    \mathcal{D} = \big \{\bX_{ij} ^ \ell, y_{ij} (\bX_{ij} ^ \ell), \mathcal{E}_{ij} \big| (i,j) \in [n] \times [n], \ell \in [L_{ij}]\big\}.
\end{equation*}

In what follows, we develop novel methods to estimate and quantify the uncertainty of the latent preference scores from the observed data and infer the combinatorial structures of the ranking systems. 

\section{Method}\label{sec:method}

We propose our methods to address the contextual ranking problem. We begin by estimating the preference score from the observed data via a kernel-smoothed likelihood function. We then present a Gaussian multiplier bootstrap approach to construct confidence bands for the rankings and outline a hypothesis testing procedure. Finally, we introduce a novel confidence diagram to infer global ranking properties. 

\subsection{Score Estimation}\label{sec:trajectory_estim}
We first estimate the latent preference scores of the $n$ LLMs, $\boldsymbol{\theta}^*(\bx) = \big(\theta_1^*(\bx),\ldots,\theta_n^*(\bx)\big)^\top$, for prompt $\bx \in \Omega$. Given that the observed data is from a subset of the entire working domain $\Omega$, our first objective is to extend our model's applicability to the full domain. To achieve this, we adopt the kernel smoothing techniques, enabling the inference of latent preferences in unobserved regions of~$\Omega$. In particular, we have the following log-likelihood function $\mathcal{L}(\boldsymbol{\theta}, \xb)$  that
\begin{equation*} 
    \mathcal{L}(\btheta; \xb)
    = \frac{1}{n^2pL} \sum_{(i,j) \in \mathcal{E}}
    \sum_{\ell_{ij} \in [L_{ij}]}
    \mathcal{K}_h\big(\bX_{ij}^\ell - \xb \big)
    \Lambda_{ij} ^ \ell \big(\btheta(\bX_{ij}^\ell) \big),
\end{equation*}
where {$\mathcal{K}_h(\cdot) = \frac{1}{h}K\big(\frac{\cdot}{h}\big) $ is a multivariate kernel function derived from the univariate kernel $K(\cdot)$, with its details and assumptions introduced in Section~\ref{sec:assumptions}}; $h>0$ is the bandwidth, and $\Lambda_{ij}^\ell(\btheta(\bX_{ij}^\ell))$ is defined as
\begin{equation}
    \Lambda_{ij}^\ell(\btheta(\bX_{ij}^\ell))= - \log \frac{\exp(\theta_i(\bX_{ij}^\ell))}{\exp(\theta_i(\bX_{ij}^\ell)) + \exp(\theta_j(\bX_{ij}^\ell))} - y_{ij}\big(\bX_{ij}^\ell \big)\big[\theta_j(\bX_{ij}^\ell) - \theta_i(\bX_{ij}^\ell) \big].
\end{equation}

We then apply the regularized maximum likelihood estimator (MLE) approach. The MLE approach provides an estimator for the latent preference scores by solving the following convex optimization problem that
\begin{equation}\label{equ:regularized_mle}
    \widehat{\boldsymbol{\theta}}(\xb) = \argmin_{\boldsymbol{\theta} \in \mathbb{R} ^ n} \mathcal{L}_{\lambda}(\boldsymbol{\theta}, \xb) =  {\mathcal{L}(\boldsymbol{\theta}; \xb) + \frac{1}{2} \lambda \|\boldsymbol{\theta}\|_2^2},
\end{equation} 
where $\lambda>0$ is a tuning parameter. The statistical rate of $\widehat{\boldsymbol{\theta}}(\xb)$ is stated in Theorem~\ref{thm:estimate}. 

\subsection{Confidence Band for Ranking}\label{sec:inf_BTL}

We first construct a $(1-\alpha)$-level confidence interval for the estimated score $\widehat{\theta}_i(\xb)$ of model $i \in [n]$. We propose a Gaussian multiplier bootstrap approach, which approximates the supremum of a sum of empirical processes by the empirical quantiles of the Gaussian supremum. In particular, we consider  a random variable $T$ defined as
\begin{equation}\label{equ:Def_TT0}
    T = \sup_{i \in [n], \xb \in \Omega} \sqrt{h^d \Xi} \cdot \big| \widehat{\theta}_i(\xb) - \theta^*_i(\xb) \big|,
\end{equation}
where we denote the effective sample size as $\Xi = \sum_{i \neq j} \mathcal{E}_{ij} L_{ij}$. Here $\cE_{ij}$ is an independent Bernoulli random variable, and $L_{ij}$ is the number of comparisons we have for models $i$ and $j$.

Since $\theta^*_i(\xb)$ is unknown in practice, $T$ is not an empirical process. We estimate the quantile of~$T$ by Gaussian multiplier bootstrap techniques introduced in Section~\ref{sec:GMB}. We further introduce a statistic $W$ from the Gaussian multiplier bootstrap, denoted as
\begin{equation}\label{equ:Def_W}
    W = \sup_{i \in [n], \xb \in \Omega} \sqrt{h^d \Xi}\cdot \bigg| \frac{  \overline{\mathcal{G}}_{i} (\xb)}{\overline{\mathcal{V}}_{i} (\xb)} \bigg|, \text{ where }
\end{equation}
\begin{equation}\label{eqn:Def_G}
    \overline{\mathcal{G}}_{i} (\xb) = \frac{1}{npL} \sum_{i \neq j} \sum_{\ell \in [L_{ij}]} \xi_{ij}^l \mathcal{E}_{ij} \mathcal{K}_h(\bX_{ij} ^ \ell - \xb) \Big(-y_{ij}(\bX_{ij} ^ \ell) + \psi(\hat{\theta}_j (\bX_{ij} ^ \ell) - \hat{\theta}_i(\bX_{ij} ^ \ell)) \Big)
\end{equation}
\begin{equation}\label{eqn:Def_V}
   \text{ and } \overline{\mathcal{V}}_{i} (\xb) = \frac{1}{npL} \sum_{i \neq m} \sum_{\ell \in [L_{ij}]} \mathcal{E}_{ij} \mathcal{K}_h(\bX_{ij} ^ \ell - \xb) \psi'(\hat{\theta}_m (\bX_{ij} ^ \ell) - \hat{\theta}_i(\bX_{ij} ^ \ell)).
\end{equation}
Here, $\xi_{ij}^\ell$ denotes independent standard normal variables, and the functions $\psi$ and $\psi'$ represent the logistic function and its derivative, respectively.

We then estimate the conditional quantile of $W$ given the observable data $\mathcal{D}$. Let $\mathbb{P}$ be the empirical distribution of $W$ conditioning on the data $\mathcal{D}$. We denote the estimated quantile by
\begin{equation}\label{equ:W_quantile}
    % \hat{c}_0(\alpha) = \inf\{t \in \mathbb{R}: \mathbb{P}(W > t|\mathcal{D}) \leq \alpha\}.\\
    \hat{c}_0(\alpha) = \inf\{t \in \mathbb{R}: \mathbb{P}(W \leq t|\mathcal{D}) \geq \alpha\}.
\end{equation}
By Corollary~\ref{col:GMB_validity}, we have a $(1-\alpha) \times 100\%$ confidence band for the true scores $\theta^*_i(\xb)$ such that
\begin{equation}\label{equ:confidence_band}
    \mathbb{P} \Big( {\theta}^*_i(\xb) \in \Big[ \hat{\theta}_i(\xb) - \frac{\hat{c}_0(1-\alpha)}{\sqrt{h^d \Xi}}, \ \hat{\theta}_i(\xb) + \frac{\hat{c}_0(1-\alpha)}{\sqrt{h^d \Xi}} \Big] \Big) \geq 1-\alpha,
\end{equation}
where $\hat{c}_0(1-\alpha)$ denotes the $(1-\alpha)$-quantile of $W$ conditioning on data $\mathcal{D}$ as defined in \eqref{equ:W_quantile}. See Section~\ref{sec:band} for detailed arguments. 
We summarize the method to provide a $(1 - \alpha) \times100 \%$ confidence band for the true ${\boldsymbol{\theta}} ^ *(\xb)$ in Algorithm~\ref{al:confidence_band}. 
 
\begin{algorithm}[h]
\caption{Construction of $(1 - \alpha) \times 100 \%$ Confidence Band for ${\boldsymbol{\theta}} ^ *(\xb)$}\label{al:confidence_band}
\begin{algorithmic}
\STATE \textbf{Input}: Estimator $\widehat{\boldsymbol{\theta}}(\xb)$, observed data $\mathcal{D}$, bandwidth $h$.
\STATE \textbf{Conditional Resampling}: Resample the statistic $W$ multiple times through Gaussian multiplier bootstrapping via resampling $\xi_{ij} ^ \ell \sim N(0,1)$ and conditioning on the data $\mathcal{D}$.
\STATE \textbf{Quantile Evaluation}: Evaluate the $(1-\alpha)$-quantile of $W$ conditioning on the data $\mathcal{D}$, denoted as $\widehat{c} (1-\alpha)$.
\STATE \textbf{Output}: $(1 - \alpha) \times 100 \%$ confidence band for $\theta ^*_i(\xb)$:\\
 $$\Big[\widehat{\theta}_i(\xb) - \frac{\widehat{c}(1-\alpha)}{\sqrt{h^d \Xi}}, \widehat{\theta}_i(\xb) + \frac{\widehat{c}(1-\alpha)}{\sqrt{h^d \Xi}}  \Big]$$
\end{algorithmic}
\end{algorithm}

\subsection{Hypothesis Testing}\label{sec:hypothesis_testing}
We propose our hypothesis testing method to infer the ranking of language models. We consider two types of hypothesis tests on the preference scores, the pairwise test and the top-$K$ test. To facilitate our discussion, we denote by
\begin{equation}\label{eqn:Ti_Wi}
    T_i(\xb) = \sqrt{h^d \Xi} \cdot \big( \widehat{\theta}_i(\xb) - {\theta}^*_i(\xb)\big)  \ \text{ and } \ W_i(\xb) =  - \frac{\sqrt{h^d \Xi} \cdot  \overline{\mathcal{G}}_{i} (\xb)}{\overline{\mathcal{V}}_{i} (\xb)},
\end{equation}
where $\overline{\mathcal{G}}_{i} (\xb)$ and $\overline{\mathcal{V}}_{i} (\xb)$ are defined in \eqref{eqn:Def_G} and \eqref{eqn:Def_V}, respectively.

\subsubsection{Pairwise Test} We consider the pairwise comparison between model $i$ and model $j$. We aim to test if model $i$ is ranked higher than model $j$ for all prompts in the working domain. In particular, we consider the following hypothesis testing problem that 
\begin{equation}\label{hyp:pairwise_uniform}
    H_0: \theta^*_i(\xb) - \theta^*_j(\xb) \leq 0 \text{ for some } \xb \in \Omega\quad \text{ v.s. }\quad H_a: \theta^*_i(\xb) - \theta^*_j(\xb) > 0 \text{ for all } \xb \in \Omega.
\end{equation}

To make the test valid, we consider a test statistic $T_{ij}$ that 
\begin{equation}\label{eqn:Tij}
    T_{ij} = \inf_{\xb \in \Omega} \sqrt{h^d \Xi} \cdot \big( \widehat{\theta}_i(\xb) - \widehat{\theta}_j(\xb)\big).
\end{equation}
We then compute the Gaussian multiplier bootstrap statistic 
\begin{equation}\label{equ:Wij}
    W_{ij} = \sup_{\xb \in \Omega} \big[W_i(\xb) - W_j(\xb) \big],
\end{equation}
where $W_i(\xb)$ is defined in \eqref{eqn:Ti_Wi}. We next estimate the conditional $(1-\alpha)$-quantile  of $W_{ij}$ given observable data $\mathcal{D}$ and denote the estimated quantile as
$
    \hat{c}_{ij}(\alpha) = \inf\{t \in \mathbb{R}: \mathbb{P}(W_{ij} \leq t|\mathcal{D}) \geq \alpha\},
$
where $\mathbb{P}$ is the empirical distribution of $W_{ij}$ conditioning on the data $\mathcal{D}$.  We reject the null hypothesis at the significance level $\alpha$ if $ T_{ij} > \hat{c}_{ij}(1-\alpha). $

\subsubsection{Top-$K$ Test}

We generalize the pairwise test to test whether a given language model $i$ ranks within the top-$K$ language models in the working domain. We first note that this is equivalent to test if model $i$ is preferred over the $(K+1)$-th most preferred model for all \(\bx\) within  domain~\(\Omega\). In particular, for some \(\bx \in \Omega\), consider the order statistics for a given score \(\boldsymbol{\theta}(\bx)\) such that \( \{\theta_{(1)}(\bx), \theta_{(2)}(\bx), \ldots, \theta_{(n)}(\bx) \}\), where
$\theta_{(1)}(\bx) \ge \theta_{(2)}(\bx) \ge \ldots \ge \theta_{(n)}(\bx)$.   We consider the following hypothesis testing problem that
\begin{equation}\label{hyp:topk_uniform}
    H_0: \theta_i^* (\xb) - \theta^*_{(K+1)} (\xb) \leq 0 \text{ for some }  \xb \in \Omega\ \  \text{ v.s. }\  \ H_a: \theta_i^* (\xb) - \theta^*_{(K+1)} (\xb) > 0  \text{ for all }  \xb \in \Omega,
\end{equation}
where $\theta^*_{(K+1)} (\xb)$ is the \((K+1)\)-th highest score for some prompt \(\xb\). If model \(i\)'s score is higher than the \((K+1)\)-th largest score for any \(\xb \in \Omega\), we consider that model \(i\) is ranked among the top-$K$ language models.

For the Top-$K$ test, we consider  test statistics $T_{i}$ such that 
\begin{equation}\label{eqn:Ti}
    T_{i} = \inf_{\xb \in \Omega} \sqrt{h^d \Xi} \cdot \big( \widehat{\theta}_i(\xb) - \widehat{\theta}_{(K+1)}(\xb)\big).
\end{equation}
To estimate the critical value, by Gaussian multiplier bootstrap, we first compute
\begin{equation} 
    W^{(i)} = \sup_{j \neq i, \xb \in \Omega} \big[W_i(\xb) - W_j(\xb) \big].
\end{equation}
We then estimate the conditional $(1-\alpha)$-quantile probability of $W^{(i)}$ given observable data $\mathcal{D}$ by
$
    \hat{c}_{i}(\alpha) = \inf\{t \in \mathbb{R}: \mathbb{P}(W^{(i)} \leq t|\mathcal{D}) \geq \alpha\},
$
where $\mathbb{P}$ is the empirical distribution of $W^{(i)}$ conditioning on the data $\mathcal{D}$.  
We reject the null hypothesis at significance level $\alpha$ if $ T_i > \hat{c}_{i}(1-\alpha). $
The validity of the testing procedure is shown in Theorem~\ref{thm:pairwise_test}. 

\subsection{Confidence Diagram for Ranking}\label{sec:confidence_diagram}

We extend the inferential framework to a novel concept of confidence diagrams as a global ranking property. As discussed in the previous subsections, the pairwise tests and Top-$K$ tests focus on the partial ranking properties, which test the performances of one or two models. Here, the confidence diagram aims to provide a global understanding of the combinatorial relationships and performance hierarchies among all LLMs of interest.

First, we present the notion of the confidence diagram for rankings. For $n$ language models, there exist $n$! potential sequences of rankings. Our objective is to identify a subset of these potential orders that covers the actual order with the desired probability. This subset can be efficiently represented through Hasse diagrams in the field of lattice theory and order theory \citep{birkhoff1940lattice}. We integrate each potential sequence into the Hasse diagram as a linear extension, simplifying the necessity to examine each possible sequence within the subset.

The Hasse diagram is essentially a partial order among all models. In our setting, we define a partial order ``\(<\)'' on all models such that for all \(i, j \in [n]\), we have \(i < j\) if and only if \(\theta^*_i(\bx) < \theta^*_j(\bx)\) for all \(\bx \in \Omega\).  
For example, in Figure~\ref{fig:hasse1}, the performances of eight models are represented by a Hasse diagram. The directed edges between the two models indicate the partial order ``\(<\)'' based on the pairwise comparison. Note that we connect two elements $i$ and $j$ with an edge if there is no element~$k$ such that $i < k < j$. This implies $i$ is directly above or below $j$ in the order, without any intermediate elements in between.

\usetikzlibrary {arrows.meta,automata,positioning,shadows}
\begin{figure}[H]\label{fig:hasse1}
    \centering
    \begin{tikzpicture}[shorten >=1pt,auto,node distance=1.5cm,on grid,thick,
                        every state/.style={fill=blue, draw=none,circular drop shadow,text=white,font=\bfseries\small\sffamily}]
        \node[state, fill=blue!55] (A) {\textbf{1}}; 
        \node[state, fill=blue!55] (B) [above=of A] {\textbf{2}}; 
        \node[state, fill=blue!55] (D) [below right=of A] {\textbf{4}}; 
        \node[state, fill=blue!55] (C) [right=of A] {\textbf{3}}; 
        \node[state, fill=blue!55] (E) [below =of D] {\textbf{5}}; 
        \node[state, fill=blue!55] (F) [left=of A] {\textbf{6}}; 
        \node[state, fill=blue!55] (G) [left=of D] {\textbf{7}}; 
        \node[state, fill=blue!55] (H) [left=of G] {\textbf{8}}; 
        \path[draw=gray!60, line width=2.4pt] 
            (E) edge node {} (D)
            (G) edge node {} (F)
            (H) edge node {} (F)
            (D) edge node {} (A)
                edge node {} (C)
            (C) edge node {} (B)
            (A) edge node {} (B)
            (F) edge node {} (B);
    \end{tikzpicture}
    \caption{Nodes 1 through 8 represent eight models, with directed edges indicating the strict partial order ``$<$'' based on the pairwise comparison. } 
\end{figure}
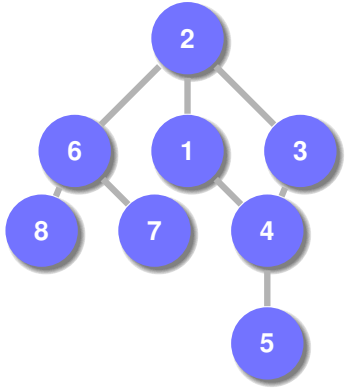
    
Algorithm~\ref{alg:confidence_with_stepdown} implements the construction of the confidence diagram in the form of Hasse diagrams. The algorithm is motivated by the step-down method \citep{romano2005exact}, which is designed to control the family-wise error of general multiple hypothesis tests. In each iteration, our method excludes all pairs that were rejected in the previous iterations.

In our practice, we define the partial order, ``$<$'', by introducing pairwise comparisons between models. We denote the set of the pairs that have been rejected in the $\ell$-th iteration as $\mathcal{S}_\ell$ and initialize $\mathcal{S}_0$ as an empty set. We consider the following pairwise test:
\[H_0: \theta^*_i(\xb) - \theta^*_j(\xb) \geq 0 \text{ for some } \xb \in \Omega \text{ v.s. } H_a: \theta^*_i(\xb) - \theta^*_j(\xb) < 0 \text{ for all } \xb \in \Omega.\]

To construct the reject region, recall that we have $T_{ij}$ defined in \eqref{eqn:Tij} such that 
$
    T_{ij} = \inf_{\xb \in \Omega} \sqrt{h^d \Xi} \cdot \big( \widehat{\theta}_i(\xb) - \widehat{\theta}_j(\xb)\big).
$
We then define the test statistic $T$ in the $\ell$-th iteration such that
$
    T = \inf_{(i,j) \in \mathcal{S}_\ell^C} T_{ij}.
$
To estimate the quantile of $T$, we first compute
$
    W = \sup_{(i,j) \in \mathcal{S}_\ell^C, \xb \in \Omega} \big[W_i(\xb) - W_j(\xb) \big],
$
where $W_i(\xb)$ is defined in \eqref{eqn:Ti_Wi}. We then estimate the conditional $(1-\alpha)$-quantile  of \(W\) given the observable data \(\mathcal{D}\) through
$
    \hat{c}_\ell(\alpha) = \inf\{t \in \mathbb{R}: \mathbb{P}(W \leq t|\mathcal{D}) \geq \alpha\},
$
where $\mathbb{P}$ is the empirical distribution of $W$ conditioning on data $\mathcal{D}$. We reject the null hypothesis if $T > \hat{c}_\ell(1-\alpha)$.

Our algorithm iteratively adds pairs to the set $L$, thereby introducing new edges to the confidence diagram. The process repeats until no more edges can be added. Therefore, by exploiting the nested structure, the algorithm increases the precision of the confidence diagram. We let $\widehat{\boldsymbol{\theta}}$ be the collection of the estimated preference scores for all $\xb \in \Omega$, i.e., $\widehat{\boldsymbol{\theta}} = \{\widehat{\boldsymbol{\theta}}(\xb): \xb \in \Omega\}$.

\begin{algorithm}[h]
    \caption{Construction of Confidence Diagrams}\label{alg:confidence_with_stepdown}
    \begin{algorithmic}
    \STATE \textbf{Input}: Estimator $\widehat{\boldsymbol{\theta}}$, observed data $\mathcal{D}$, bandwidth $h$, level $\alpha$.
    \STATE Initialize all nodes' levels to $\infty$, and the process list $Q \gets [1,2,\ldots,n]$.
    \STATE Initialize the index of iteration $\ell \gets 0 $.
    \STATE Initialize the set of pairs $\mathcal{S}_0 \gets \emptyset $. % 
    \STATE Initialize the quantile $\hat{c}_0(\alpha) = \inf\{t \in \mathbb{R}: \mathbb{P}(\sup_{(i,j) \in \mathcal{S}_0^C, \xb \in \Omega} \big[W_i(\xb) - W_j(\xb) \big] \leq t|\mathcal{D}) \geq \alpha\}$
    
    \WHILE{$\mathcal{S}_\ell$ is not empty or $\ell$=0}
        \STATE $\mathcal{S}_\ell \gets \emptyset$ 
        \WHILE{$Q$ is not empty}
            \STATE $i \gets Q[1]$ and $\text{level}(i) \gets 1$
            \STATE $Q \gets Q \backslash Q[1]$, i.e., remove the first element from $Q$
            
            \FORALL{nodes $k$ where $T_{ki} > \hat{c}_l(1-\alpha)$}
            \STATE $\mathcal{S}_\ell \gets \mathcal{S}_\ell \cup \{(k,i)\}$, i.e., add the pair $(k,i)$ to $\mathcal{S}_\ell$.
                \IF{$\text{level}(k) < \text{level}(i) + 1$ or $\text{level}(k) = \infty$}
                    \STATE $\text{level}(k) \gets \text{level}(i) + 1$
                    \STATE $Q \gets [Q,k]$, i.e., append $k$ to the end of $Q$
                \ENDIF
            \ENDFOR
            
        \ENDWHILE
    
    \STATE $\ell \gets \ell + 1$
    \STATE $\mathcal{S}_\ell^C \gets \{(i,j) \in [n]\times[n] \} \backslash \mathcal{S}_\ell$
    \STATE Update $\hat{c}_\ell(\alpha) = \inf\{t \in \mathbb{R}: \mathbb{P}(\sup_{(i,j) \in \mathcal{S}_\ell^C, \xb \in \Omega} \big[W_i(\xb) - W_j(\xb) \big] \leq t|\mathcal{D}) \geq \alpha\}$
    \ENDWHILE
    \STATE \textbf{Output}: Confidence Diagram $\widehat{\mathbf{H}}$.
    
    \end{algorithmic}
\end{algorithm}

Denote $R(\boldsymbol{\theta}^*)$ as the true ranking of the language models. Then from Algorithm~\ref{alg:confidence_with_stepdown}, we obtain a $(1-\alpha) \times 100 \%$ confidence diagram of the ranking such that 
\begin{equation}\label{eqn:confidence_diagram}
    \mathbb{P} \Big( R(\boldsymbol{\theta}^*) \in \widehat{\mathbf{H}} \Big) \geq 1-\alpha,
\end{equation}
where $\widehat{\mathbf{H}}$ is output from Algorithm~\ref{alg:confidence_with_stepdown} and ``$\in$'' indicates that \( R(\boldsymbol{\theta}^*) \) is a linear extension of the diagram. Linear extension refers to a total ordering of the models that is consistent with the partial order represented by the diagram. In other words, it is an arrangement of the models into a sequence such that if one model precedes another in the partial order, it also precedes it in the sequence.

%% file: content/theory.tex
% !TEX root = ../main.tex

\section{Theory}\label{sec:theory}
We provide theoretical justifications for the inferential methods for the contextual ranking proposed in the previous section. We first introduce the assumptions for our analyses and then derive the convergence rate for the proposed estimator. Next, we introduce the nonparametric bootstrap theory and apply it to our inferential methods. We conclude this section by providing theoretical guarantees for the converge probability of the confidence diagrams.

\subsection{Assumptions}\label{sec:assumptions}

Assumption~\ref{ass:erdos_p} is crucial to ensure the connectedness of Erdős-Rényi graphs by defining a lower bound for the probability $p$.

\begin{assump}[Condition on the Probability $p$ and Sample Size $L$]\label{ass:erdos_p}
    The probability $p$ for the Erdős–Rényi random graph $\mathcal{G}(\mathcal{V}, \mathcal{E})$ satisfies that $p \gtrsim (\log n) ^ {3/2} n ^ {\epsilon - 1}$ for some $\epsilon \in (0,1)$. Furthermore, the sample size for each pair $L$ satisfies that $Lh^d = \Omega (n^{\widetilde{\epsilon}})$ for some positive constant $\widetilde \epsilon$, where $h$ is the kernel radius.
\end{assump}
 Assumptions~\ref{ass:time_independence}-\ref{ass:time_distribution} focus on the prompt $\bm X_{ij} ^ \ell \in \Omega$. 

\begin{assump}[Independence of $\bX_{ij} ^ \ell$]\label{ass:time_independence}
    The variables $\bX_{ij} ^ \ell$ are independent across all indices $i, j, \text{and } \ell$. Furthermore, $\bX_{ij} ^ \ell$ is assumed to be independent of the pairing assignments given by the Erdős-Rényi graph $\mathcal{G}(\mathcal{V}, \mathcal{E})$.

\end{assump}

\begin{assump}[Distribution Regularity of $f_{\bX}(\cdot)$]\label{ass:time_distribution}
    Recall that $f_{\bX}(\cdot)$ represent the distribution of $\bX_{ij} ^ \ell$, with its support contained within $\Omega$. It is assumed that $f_{\bX}(\cdot)$ is uniformly bounded away from zero by a constant $c > 0$, and is sufficiently smooth for some finite constant $C$ such that 
    $$\big\| f_{\bX} \big\|_{W^{2,\infty}} \leq C \text{\, where \ } \big\|\cdot \big\|_{W^{2,\infty}} \text{ is the Sobolev norm.}$$
\end{assump}

Assumption~\ref{ass:similar_comparison} imposes the uniformity in the number of comparisons $L_{ij}$ for all pairs $(i,j) \in \mathcal{E}$.

\begin{assump}[Uniform Comparison Counts]\label{ass:similar_comparison}
    The comparison counts $L_{ij}$ are constrained within the interval $[L_{m}, L_{M}]$ and bounded by two constants $c < C$ such that $c L \leq L_{m} \leq L_{M} \leq C L$ for any $i, j \in [n]$.
\end{assump}

Assumption~\ref{ass:constraint_lps} bound variability of the latent preference scores $\{\omega^*_i(\xb) \}_{i = 1} ^ n$.

\begin{assump}[Bounded Latent Preference Scores]\label{ass:constraint_lps}
    The latent preference scores $\omega^*_i(\xb)$ are assumed to be confined within the bounds $[c, C]$ for positive constants $c < C$, across all $\xb \in \Omega$.
\end{assump}

We introduce Assumptions~\ref{ass:constraint_theta}-\ref{ass:regularity_conditions} for the identifiability of $\btheta^*(\xb)$.

\begin{assump}[Constraint on $\btheta^*(\cdot)$]\label{ass:constraint_theta}
    The condition $\mathbf{1} ^ \intercal \btheta ^ * (\xb) = 0$ holds for every $\xb \in \Omega$.
\end{assump}

\begin{assump}[Smoothness of $\btheta^*(\xb)$]\label{ass:regularity_conditions}
    $\btheta^*(\xb)$ is twice continuously differentiable and there exists a positive constant $C$ such that $\big\| \btheta^*(\xb) \big\|_{W^{2,\infty}} \leq C$.
\end{assump}

We introduce Assumptions~\ref{ass:kernel_general}-\ref{ass:kernel_smooth} to the multivariate kernel function $\mathcal{K}(\cdot)\colon \mathbb{R}^d \rightarrow \mathbb{R}$.

\begin{assump}[General Properties of the Kernel Function $\mathcal{K}(\cdot)$]\label{ass:kernel_general}
The multivariate kernel function $\mathcal{K}(\cdot)$ is defined as
$
    \mathcal{K} (\mathbf{u}) = \prod_{i = 1} ^ d K(u_i),
$
where each univariate kernel $K(u_i)$ is symmetric, bounded, unimodal, and has a compact support.
\end{assump}

\begin{assump}[Integrability of the Kernel Function $K(\cdot)$]\label{ass:kernel_integ}
    For any integer $\ell \in \{1,2,3,4\}$, the kernel function $K(\cdot)$ satisfies

    \begin{equation}
        \int_{\mathbb{R}} K(u) \dd u = 1, \quad \int_{\mathbb{R}} u K(u) \dd u = 0, \quad \int_{\mathbb{R}} u^\ell K(u) \dd u < \infty, \quad \int_{\mathbb{R}} K^\ell(u) \dd u < \infty,
    \end{equation}
    ensuring its integrability and the boundedness of its moments up to the fourth order.
\end{assump}

\begin{assump}[Smoothness of the Kernel Function $K(\cdot)$]\label{ass:kernel_smooth}
    The univariate kernel function $K(\cdot)$ exhibits smoothness, characterized by a finite total variation:
    \begin{equation*}
        \|K\|_{\mathrm{TV}} = \int_{\mathbb{R}} |K'(u)| \, \dd u < \infty.
    \end{equation*}
\end{assump}

\begin{rmk}
    One example of multivariate kernel function that fulfills Assumptions~\ref{ass:kernel_general}-\ref{ass:kernel_smooth} has been the multiplicative Epanechnikov kernel
    \( \mathcal{K}(\ub) = \prod_{i = 1} ^ d (1 - u_i ^ 2) 1_{\{|u_i| \leq 1\}}, \)
    with \( 1_{\{|u_i| \leq 1\}}\) being an indicator function that takes value 1 if \(|u_i| \leq 1\) and 0 elsewhere.
\end{rmk}

\begin{rmk}
    Assumption~\ref{ass:erdos_p} relaxes the conventional conditions 
    \( p \gg \frac{(\log n) ^ \alpha}{n} \) where \(\alpha \in [1, \infty], L = O(\text{poly} \log n)\),
    which are commonly cited in extant literature concerning the inference of static ranking models \citep{chen2019spectral, gao2023uncertainty, fan2022PL}. Furthermore, Assumption~\ref{ass:erdos_p} guarantees a $O(n^{-\tau})$ tail probability when we apply concentration inequalities on the summation $\Xi$ across language models. It also guarantees the existence of $h = O(n^{-\epsilon/4})$ that further guarantees
    $
        h^2 + \sqrt{{\log(n h^{d/2-1})}/{nh^d}} \lesssim n^{-\epsilon / 2}\sqrt{\log{n}}.
    $
\end{rmk}

\subsection{Convergence rate of the Estimator}\label{sec:theory_estimate}

The following theorem provides the statistical rate of $\widehat{\btheta}(\xb)$ in \eqref{equ:regularized_mle}.
\begin{thm}\label{thm:estimate} 
    Suppose that the assumptions in Section~\ref{sec:assumptions} hold. There exists a positive constant~$C$ such that the regularized MLE $\widehat{\btheta}(\xb)$ satisfies that
    \begin{equation*}
        \sup_{\xb \in \Omega} \big\|\widehat{\btheta}(\xb) - \btheta ^ *(\xb) \big\|_\infty \leq C \bigg(h^2 + \sqrt{\frac{\log (nh ^ {d/2-1})}{npLh^d}} \bigg).
    \end{equation*}
\end{thm}

\begin{proof} [Proof Sketch:]
     We outline the sketch of the proof here. The detailed proof is deferred to Appendix~\ref{pf:thm:estimate} in the Supplementary Material. To bound the distance $\big\| \widehat{\btheta}(\xb) - \btheta ^*(\xb) \big\|_\infty$, we consider applying a gradient descent algorithm  detailed in Appendix~\ref{pf:thm:estimate} in the Supplementary Material. We denote the output of the gradient descent algorithm by $\btheta ^ {\Gamma} (\xb)$. After sufficiently large iteration steps $\Gamma$, we bound $\big\| \btheta ^ \Gamma(\xb) - \btheta ^*(\xb) \big\|_\infty$ and $\big\| \widehat{\btheta}(\xb) - \btheta ^ \Gamma(\xb) \big\|_\infty$ separately. We use index $t\in [\Gamma]$ to track the iteration.  
     
     First,  for $\big\|\widehat{\btheta}(\xb) - \btheta ^ \Gamma (\xb) \big\|_\infty$, we have that  for some sufficiently large $\Gamma \gtrsim n^2 L$, it holds that with probability at least $1 - O(n^{-10})$,
    \begin{equation*}
        \sup_{\xb\in \Omega} \big\|\btheta ^ \Gamma(\xb) - \widehat{\btheta}(\xb) \big\|_\infty \leq C \bigg( h^2 + \sqrt{\frac{\log(nh ^ {d/2-1})}{npLh^d}} \bigg).
    \end{equation*}
    
    Then, to bound $\big\|\btheta^\Gamma(\xb) - \btheta^*(\xb)\big\|_\infty$, we apply the leave-one-out framework by \cite{chen2019spectral} and have that
    \begin{equation*} 
        \sup_{\xb \in \Omega} \big\| \btheta^{t + 1}(\xb) - \btheta ^ * (\xb) \big \|_\infty \leq C \bigg(h ^ 2 + \sqrt{\frac{\log (nh^{d/2-1})}{npLh ^ d}} \bigg).
    \end{equation*}
    We then iteratively apply the inequality for $\Gamma$ iterations and letting $\btheta^0(\xb) = \btheta ^*(\xb)$. Combining the results above, our claim holds as desired.
 \end{proof}   

\subsection{Gaussian Multiplier Bootstrap}\label{sec:GMB}
Building upon \cite{chernozhukov2013gaussian}, we develop a generalized Gaussian multiplier bootstrap theorem that fits into the contextual ranking framework of our interest. 

Existing literature on the nonparametric bootstrap justifies the approximation of the maximum of a sum of i.i.d empirical processes by a Gaussian process \citep{chernozhukov2013gaussian, cck2014band}. However, our problem is more challenging since (1) The ranking of large language models consists of score functions that are independent among different models but do not satisfy the identically distributed assumption; (2) due to the continuous nature of the prompt domain, the supremum of our interest is over a continuous space rather than a finite maximum. Therefore, existing theoretical results on the Gaussian multiplier bootstrap do not directly apply to our contextual ranking of LLMs, and we necessitate the development of new theoretical advancements.

To facilitate our discussion, we first introduce some notations. Consider \(\bX\) as a metric space equipped with metric \(\|\cdot\|\). Let \(\Omega \subset \bX\), and let \(X_1, \ldots, X_n\in   \Omega\) be i.i.d. random variables. We introduce a function class \(\mathcal{F}:\Omega\rightarrow\RR\), and let \(f(\xb, \cdot) \in \mathcal{F}\) for every \(\xb \in \Omega\). 
Moreover, we require that for any function \( f(\xb, \cdot) \in \mathcal{F}\), the expectation \(E_{X_i}\big(f(\xb,X_i)\big)=0\).

Let \(\{f_i(\xb, \cdot) \}_{i \in [n], \xb \in \Omega}\) be a set of unknown functions within \(\mathcal{F}\). Our goal is to obtain a distributional approximation for the statistic $T_0$ defined as
\begin{equation*} 
    T_0 = \sup_{\xb \in \Omega} \frac{1}{\sqrt{n}} \sum_{i = 1}^n f_i(\xb, X_i).
\end{equation*}

Next, we impose Assumptions~\ref{ass:GMB_dominate}-\ref{ass:GMB_covering} to ensure the validity of our Gaussian multiplier bootstrap method that are essential to control the discretization error, which are commonly imposed in prior work \citep{imaizumi2021gaussian}.
\begin{assump}[Existence and Requirements of Envelope Function]\label{ass:GMB_dominate}
    There exists an envelope function \( F \) for the function family \( \mathcal{F} \) such that
    \[
    \sup_{f \in \mathcal{F}, \xb \in \Omega} |f(\xb, x)| \leq F(x) \quad \text{for all } x \in \Omega.
    \]
    Additionally, we assume that \( \|F\|_{\infty} \leq \eta < \infty \) for some \( \eta > 0 \). 
\end{assump}

\begin{assump}[Uniform Lipschitz Property]\label{ass:GMB_Lipschitz}
    For any $\xb, \xb' \in \Omega$ and $f(\xb, \cdot), f(\xb', \cdot) \in \mathcal{F}$, there exists some  positive constant $K_{\mathcal{F}}$ such that for any $ x \in \Omega$, we have
    $
        |f(\xb, x) - f(\xb', x)| \leq K_{\mathcal{F}} \|\xb - \xb'\|.
    $
\end{assump}

\begin{assump}[Property of Covering Number]\label{ass:GMB_covering}
    For any arbitrary $\xb \in \Omega$ and $u, \epsilon > 0$, we have the upper bound on the covering number of $u$-net on $B_{\|\cdot\|}(\xb, \epsilon)$, denoted as $N(B_{\|\cdot\|}(\xb, \epsilon), \|\cdot\|, u)$, such that
    $
        N(B_{\|\cdot\|}(\xb, \epsilon), \|\cdot\|, u) \leq N(\bX, \|\cdot\|, {u}/({2\epsilon})),
    $
    where $B_{\|\cdot\|}(\xb, \epsilon)$ denote a ball in $\bX$ with center $\xb$ and radius $\epsilon$ given the metric space $(\bX, \|\cdot\|)$.
\end{assump}

\begin{rmk}\label{rmk:covering}
    If \(\bX = [0,1]^d\), then for any point \(\xb \in \bX\), the ball \(B_{\|\cdot\|}(\xb, {1}/{2})\) has smaller volume than the cube \(\bX\). This indicates that we need fewer balls to cover the ball than to cover the cube. Therefore, for any \(u, \epsilon > 0\), we have  
    $
    N(B_{\|\cdot\|}(\xb, {1}/{2}), \|\cdot\|, {u}/({2\epsilon})) \leq N(\bX, \|\cdot\|, {u}/({2\epsilon})).
    $
     Meanwhile, since the covering number is invariant under a certain scaling of the ball's and the cover's radius, we have  that Assumption~\ref{ass:GMB_covering} is satisfied that
    $
    N(B_{\|\cdot\|}(\xb, \epsilon), \|\cdot\|, u) \leq N(\bX, \|\cdot\|, {u}/({2\epsilon})).
    $
\end{rmk}

The next theorem justifies our Gaussian multiplier bootstrap approach that accommodates independent but not identically distributed functions. The proof of the theorem leverages the concept of the \(\epsilon\)-net, denoted as \(\mathcal{N}_\epsilon\), to discretize the continuous domain. The key observation is that the discretization error between the supremum over the continuous space and the maximum over the \(\epsilon\)-net is closely related to the complexity of a tube $\bX_\epsilon^2$, defined as 
\begin{equation}\label{def:complexity_tube}
    \bX_\epsilon^2 = \{(\xb, \xb'): \xb, \xb' \in \bX; \|\xb- \xb'\| \leq \epsilon \}.
\end{equation}

\begin{thm}\label{thm:GMB_validity}
    Let $T_0$ and $W_0$ be the statistics for the empirical process and the respective Gaussian process that
    \begin{equation}\label{equ:def_T0}
    T_0 = \sup_{\xb \in \Omega} \frac{1}{\sqrt{n}} \sum_{i = 1}^n f_i(\xb, X_i), 
    \end{equation}
    and
    \begin{equation}\label{equ:def_W0}
        W_0 = \sup_{\xb \in \Omega} \frac{1}{\sqrt{n}} \sum_{i = 1}^n \xi_i f_i(\xb, X_i),
    \end{equation}
    where each $\xi_i$ is an independent standard Gaussian random variable. 

    Given Assumptions \ref{ass:GMB_dominate}--\ref{ass:GMB_covering},  suppose that some test statistics $T$ and $W$ satisfy
    \begin{equation}\label{assump:GMB_error}
        (1).\  \mathbb{P}(|T - T_0| > \zeta_1) < \zeta_2 \text{\ \ and\ \ } (2).\  \mathbb{P}(\mathbb{P}(|W - W_0| > \zeta_1| X_{1:n}) > \zeta_2) < \zeta_2,
    \end{equation}
    for some $\zeta_1 \geq 0$, $\zeta_2 \geq 0$. Moreover, assume that
    \begin{equation*}
        \Big(\zeta_1 + \int_0^{\eta / n} \sqrt{\log N(\bX, \|\cdot\|, t)} dt + \frac{1}{\sqrt{n}} \Big) \sqrt{\log{N(\bX, \|\cdot\|, 1/n)}} + \zeta_2 \leq C_0 n^{-c_0}.
    \end{equation*}
    for some positive constants $c_0$ and $C_0$. We denote the quantile for the empirical distribution of~$W$ as $c_W(\alpha)$. 
    Then, for any significance level $\alpha$, the quantile estimator for $T$ through its Gaussian approximate $W$ is valid for sufficiently large $n$. In particular, we have
    $
        \sup_{\alpha \in (0,1)} |\mathbb{P}(T \leq c_W(\alpha)) - \alpha| \leq C n^{-c},
    $
    where $c$, $C$ are positive constants.
\end{thm}

\begin{proof} [Proof Sketch:]
    { We present the proof idea here and defer the details to Appendix~\ref{pf:GMB_validity} in the Supplementary Material. We first discretize the continuous space \(\Omega\) into an \(\epsilon\)-net and then control the discretization error. Specifically, the discretization allows us to apply the classical Gaussian multiplier bootstrap to the resulting finite-dimensional process under some regularity conditions. One such condition is that we need to show that the error between the statistic of interest \(T\) and the discretized empirical process is well controlled. By  \eqref{assump:GMB_error}, it suffices to show that the error between~$T_0$ and its discretized version is well controlled. 
    We also reveal that the discretization error can be controlled by studying the complexity of a tube $\bX_\epsilon^2$ defined in \eqref{def:complexity_tube}.

    We then present the proof sketch for controlling the discretization error of $T_0$. We first denote the discretized version of $T_0$ as
    where $\mathcal{N}_\epsilon$ is an $\epsilon$-net  on $\bX$. 

    We introduce $\bX_\epsilon^2$ to control the discretization error. For any \(\xb \in \bX\), we can identify a corresponding \(\xb' \in \mathcal{N}_\varepsilon\) with \(\|\xb- \xb'\| \leq \epsilon\) such that
    \begin{equation*}
    \frac{1}{n} \sum_{i=1}^n \left[ f_i(\xb, X_i) - \max_{\xb_q \in \mathcal{N}_\epsilon} f_i(\xb_q, X_i) \right] \leq \sup_{(\xb,\xb') \in \bX_\epsilon^2} \frac{1}{n} \sum_{i=1}^n \left[ f_i(\xb, X_i) - f_i(\xb', X_i) \right] .
    \end{equation*}

    By bounding the complexity of the tube, we arrive at
    \begin{equation*}
        \mathbb{P} \left( T_0 - \widetilde{T}_0 \geq K_1 \int_0^{4\eta/n} \sqrt{1+\log N(\bX_\epsilon^2, d_{\|\cdot\|}, t)} dt + \frac{K_2}{\sqrt{n}} \right) \leq 1-e^{-\xi},
    \end{equation*}
    where the distance \( d_{\|\cdot\|}(\cdot) \) on \( \bX_\epsilon^2 \) is defined as \( d_{\|\cdot\|}((\xb, \xb'), (\yb, \yb')) = \|\xb - \yb\| + \|\xb' - \yb'\| \).

    By Assumption~\ref{ass:GMB_covering}, we have
    \begin{equation*}
        \mathbb{P} \left( T_0 - \widetilde{T}_0 \geq K_1 \int_0^{\eta/n} \sqrt{1+\log N(\bX, \|\cdot\|, t)} dt + \frac{K_2}{\sqrt{n}} \right) \leq 1-e^{-\xi},
    \end{equation*}
    which completes our proof for controlling the discretization error of $T_0$. Plugging the numbers in, we have \(\zeta_1' \geq 0\) and \(\zeta_2' \geq 0\) such that
    $
        \mathbb{P}(| T_0 - \widetilde{T}_0| > \zeta_1') < \zeta_2'. 
    $
By similar arguments, we have that this claim also holds for  $W_0$. 

By controlling the discretization error for $T_0$ and $W_0$, we demonstrate the validity of applying the classical Gaussian multiplier bootstrap on the discretized versions to approximate the distribution of $T$ through $W$. Specifically, for sufficiently large $n$, it holds that
    $
    \sup_{\alpha \in (0,1)} |\mathbb{P}(T \leq c_W(\alpha)) - \alpha| \leq C n^{-c},
    $
    which completes the proof.}
\end{proof} 
  
Following Theorem~\ref{thm:GMB_validity}, the next corollary considers another scenario where
\begin{equation}\label{equ:X_definition}
    \bX = \Big\{ \big\{ \mathcal{I}_m \big\} \big| m \in [n], \mathcal{I}_m = [0,1]^d \Big\}, \ \  \|\cdot\| = \|\cdot\|_2.
\end{equation}

\begin{crl}\label{col:GMB_validity}
    Let $T_0$ and $W_0$ be the statistic for the empirical process and the respective Gaussian process defined in \eqref{equ:def_T0} and \eqref{equ:def_W0}, and let  metric space $(\bX, \|\cdot\|)$ be defined in \eqref{equ:X_definition}. Given Assumptions~\ref{ass:GMB_dominate} and \ref{ass:GMB_Lipschitz}, suppose that  test statistics $T$ and $W$ satisfy
    \begin{equation*} %\label{assump:GMB_error}
        (1).\  \mathbb{P}(|T - T_0| > \zeta_1) < \zeta_2 \text{\ \ and\ \ } (2).\  \mathbb{P}(\mathbb{P}(|W - W_0| > \zeta_1| X_{1:n}) > \zeta_2) < \zeta_2,
    \end{equation*}
    for some $\zeta_1, \zeta_2 \geq 0$. Moreover, suppose that  $\zeta_1 \sqrt{\log n} + \zeta_2 \leq C_0 n ^ {-c_0}$ for some positive constants $c_0$ and $C_0$. We denote the quantile for the empirical distribution of  $W$ as $c_W(\alpha)$. 
    Then, for any significance level $\alpha$, the quantile estimator for $T$ through its Gaussian approximate $W$ is valid for sufficiently large $n$ that we have, for some positive constants $c$ and $C$,
    $
        \sup_{\alpha \in (0,1)} |\mathbb{P}(T \leq c_W(\alpha)) - \alpha| \leq C n^{-c}.
    $
\end{crl}

\subsection{Validity of Confidence Band}\label{sec:band}
We justify the confidence band of $\theta^*_i(\xb)$ provided in Section~\ref{sec:inf_BTL} as an application of the Gaussian multiplier bootstrap proposed in Section~\ref{sec:GMB}. 
Recall that we consider $T$ in \eqref{equ:Def_TT0} such that
$
    T = \sup_{i \in [n], \xb \in \Omega} \sqrt{h^d \Xi} \cdot \big| \widehat{\theta}_i(\xb) - \theta^*_i(\xb) \big|,
$
where $\Xi = \sum_{i \neq j} \mathcal{E}_{ij} L_{ij}$ is the effective sample size, and the quantile of the statistics $W$ in \eqref{equ:W_quantile}:
$
    \hat{c}_0(\alpha) = \inf\{t \in \mathbb{R}: \mathbb{P}(W \leq t|\mathcal{D}) \geq \alpha\}.
$
Following  Corollary~\ref{col:GMB_validity}, we justify the validity of our confidence band for $\theta^*_i(\xb)$ in \eqref{equ:confidence_band} in the next theorem. 

\begin{thm}\label{thm:band}

    Suppose that Assumptions \ref{ass:erdos_p}--\ref{ass:kernel_smooth} hold, and  bandwidth $h$ satisfies \(h \leq \big (\frac{npL}{\log n} \big)^{-\frac{1}{d+4}}\). 
    By Corollary~\ref{col:GMB_validity}, we have that, for some positive constants $c$ and $C$,
    \[
        \sup_{\alpha \in (0,1)} |\mathbb{P}(T \leq \hat{c}_0(1 - \alpha)) - (1 - \alpha)| \leq C n^{-c}.
    \] Furthermore, for any $i \in [n]$, we have the $(1 - \alpha)\times 100\%$ confidence band for $\theta^*_i(\xb)$ such that
   \[ 
    \mathbb{P} \Big( {\theta}^*_i(\xb) \in \Big[ \hat{\theta}_i(\xb) - \frac{\hat{c}_0(1-\alpha)}{\sqrt{h^d \Xi}}, \ \hat{\theta}_i(\xb) + \frac{\hat{c}_0(1-\alpha)}{\sqrt{h^d \Xi}}\Big] \Big) \geq 1-\alpha.
   \]
\end{thm}

\begin{proof}
    See Appendix~\ref{pf:thm:band} in the Supplementary Material for detailed proof.
\end{proof}

\subsection{Validity of Hypothesis Testing}\label{sec:theory_hypothesis}

We justify the hypothesis testing procedure for the contextual ranking model discussed in Section~\ref{sec:hypothesis_testing}. Recall that we consider two types of hypothesis tests on the ranking, the pairwise test and the top-$K$ test. 
For the pairwise test, 
we consider the  test that
\begin{equation*} %\label{hyp:pairwise_uniform}
    H_0: \theta^*_i(\xb) - \theta^*_j(\xb) \leq 0 \text{ for some } \xb \in \Omega \text{\ \ \  v.s.\ \  } H_a: \theta^*_i(\xb) - \theta^*_j(\xb) > 0 \text{ for any } \xb \in \Omega.
\end{equation*}
We have the test statistic $T_{ij}$ defined in \eqref{eqn:Tij}.
We reject the null hypothesis if \( T_{ij} > \hat{c}_{ij}(1-\alpha), \)
where $\hat{c}_{ij}(1-\alpha)$ is the $(1-\alpha)$-quantile of the Gaussian multiplier bootstrap of $T_{ij}$.

For the top-$K$ test, 
we consider the  test  that
\begin{equation*} %\label{hyp:topk_uniform}
    H_0: \theta_i^* (\xb) - \theta^*_{(K+1)} (\xb) \leq 0 \text{ for some }  \xb \in \Omega \text{\ \ \ v.s.\ \  } H_a: \theta_i^* (\xb) - \theta^*_{(K+1)} (\xb) > 0  \text{ for any }  \xb \in \Omega.
\end{equation*}
We use the test statistic $T_{i}$ defined in \eqref{eqn:Ti}.
We reject the null hypothesis if \( T_{i} > \hat{c}_{i}(1-\alpha), \)
where $\hat{c}_{i}(1-\alpha)$ is the $(1-\alpha)$-quantile of the Gaussian multiplier bootstrap of $T_{i}$.

The next theorem shows that the proposed testing procedure controls the Type I error, and  the power goes to 1 asymptotically.

\begin{thm}\label{thm:pairwise_test}
    Under the same assumptions as in Corollary~\ref{col:GMB_validity}, both of the pairwise test and top-$K$ test in \eqref{hyp:pairwise_uniform} and \eqref{hyp:topk_uniform} are valid for any significance level $\alpha \in (0,1)$. In particular, we have that, as $n$, $L \xrightarrow[]{} \infty$,
    \[ 
        \sup_{\theta^* \in H_0} \mathbb{P}_{\theta^*}(\text{reject } H_0) \leq \alpha. \]
    If  kernel bandwidth $h$ satisfies \(h = O \Big( \big(\frac{npL}{\log n} \big)^{-\frac{1}{d+4}} \Big) \), then for the pairwise test, we have \[
        \inf_{\forall \xb, \theta^*_i(\xb) - \theta^*_j(\xb) >\delta} \mathbb{P}_{\theta^*}(\text{reject } H_0) \xrightarrow[]{} 1,
    \] and for the top-$K$ test, we have \[
        \inf_{\forall \xb, \theta^*_i(\xb) - \theta^*_{(K+1)}(\xb) > \delta} \mathbb{P}_{\theta^*}(\text{reject } H_0) \rightarrow 1,
    \]
    where $\delta = c \big (\frac{npL}{\log n} \big)^{-\frac{2}{d+4}} $,  and $c$ is a positive constant.
\end{thm}
\begin{proof}
    See Appendix~\ref{pf:thm:pairwise_test} in the Supplementary Material for detailed proof. 
\end{proof}

\subsection{Validity of Confidence Diagram}\label{sec:diagram}
We justify the validity of the confidence diagram for the ranking in Section~\ref{sec:confidence_diagram}. 
In Algorithm~\ref{alg:confidence_with_stepdown}, we construct the confidence diagram $\widehat{\mathbf{H}}$ by iteratively rejecting the pairs of language models based on the hypotheses tests. Recall that we denote the true ranking of the language models by $R(\boldsymbol{\theta}^*)$.  We have the following theorem for the confidence diagram of the ranking.

\begin{thm}\label{thm:confidence_diagram}
    Under the same assumptions as in Corollary~\ref{col:GMB_validity}, the confidence diagram $\widehat{\mathbf{H}}$ satisfies that, for any significance level $\alpha$, as $n$, $L \xrightarrow[]{} \infty$,
    \[
        \mathbb{P} \Big( R(\boldsymbol{\theta}^*) \in \widehat{\mathbf{H}} \Big) \geq 1-\alpha,
    \]
    where ``$\in$'' indicates that \( R(\boldsymbol{\theta}^*) \) is a linear extension of the diagram, where linear extension refers to a total ordering of the models that is consistent with the partial order represented by the diagram. 

\end{thm}

\begin{proof}
    See Appendix~\ref{pf:thm:confidence_diagram} in the Supplementary Material for detailed proof. 
\end{proof}

%% file: content/numerical.tex
% !TEX root = ../main.tex

\section{Numerical Results}\label{sec:numerical}
We conduct extensive numerical studies to test the empirical performance of the proposed methods using both synthetic and real-world datasets. 

\subsection{Synthetic Data}
Using synthetic data, we evaluate our methods proposed in Section~\ref{sec:method}. Specifically, we set the dimension $d=3$ and the prompt domain $\bX =[0,1]^3$. Let $\xb = (x_1,x_2,x_3)$. We assume that for any two models $i$ and $j$ where $\mathcal{E}_{ij} = 1$, they are compared $L$ times. 

We first evaluate the performance of the proposed estimators in Section~\ref{sec:trajectory_estim}. The values of $\xb$ are generated uniformly over $\bX$, and the true latent score $\theta^*(\xb)$ is defined as $ \log \theta_i^*(\xb) = 0.01i\times\sum_{1}^{3}{x_i}$, where~$i$ is the model index.
Following the procedure developed in Section \ref{sec:setup}, we repeat the generating scheme 50 times and present the mean squared error (MSE) of the estimators averaging over all models. Figure~\ref{fig:mse} (A) displays the MSE for the proposed estimator with $p=0.5$ across varying values of $n$ and $L$. Figure~\ref{fig:mse} (B) displays the MSE with $n=20$ across different $L$ and $p$. These two figures illustrate that the MSE decreases as $n,p$ and $L$ increase, consistent with our results in Theorem~\ref{thm:estimate}.

\begin{figure}[H]\label{fig:mse}
	% \vskip-10pt
	\begin{center}
		\includegraphics[height=0.25\textwidth]{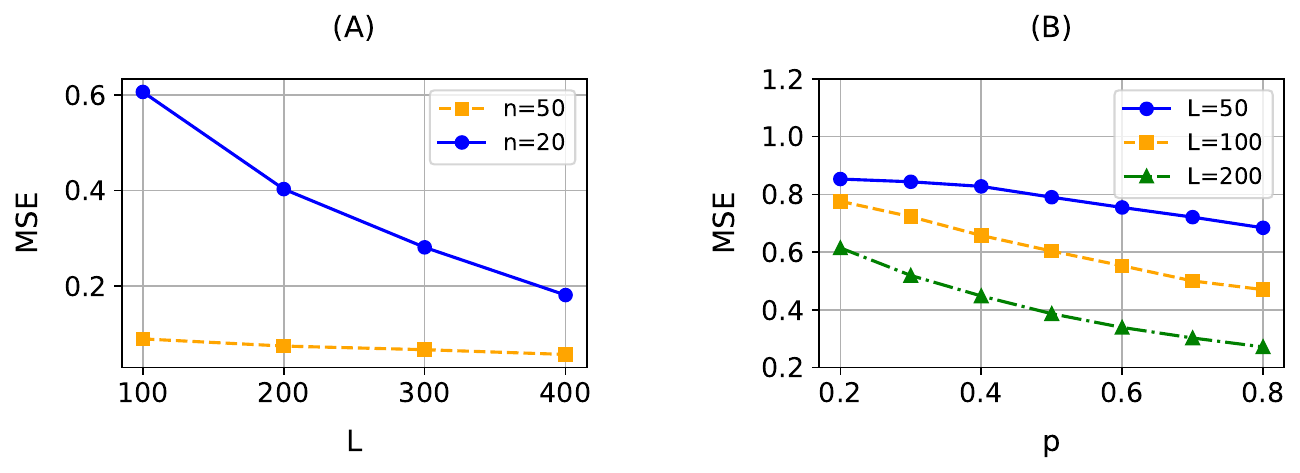}
	\end{center}
	\vskip-20pt
	\caption{The performance of the proposed estimators, comparing the MSE with different $n$, $L$ and $p$. (A) We set $p=0.5$ and vary $n$ and $L$. (B) We set $n=20$ and vary $L$ and $p$.} 
\end{figure}

Next, we examine the validity of the confidence bands in Section~\ref{sec:band} by evaluating the empirical coverage probabilities of the bands.
Figure~\ref{fig:coverage} compares the upper and lower confidence bands to the ground truth with different $n$, $L$, $p$, and the model index $i$. For this analysis, we generate $x_1$ and $x_2$ uniformly over $[0,1]$, while fixing $x_3 = 0.4$. We set the significance level as $\alpha=0.1$. We observe that the confidence band covers the ground truth almost everywhere, demonstrating the validity of the proposed method. 
The only exception locates in Figure~\ref{fig:coverage} (C), where the ground truth is slightly above the upper confidence band when both $x_1$ and $x_2$ are close to 1. 

\begin{figure}[H]\label{fig:coverage}
	\begin{center}
		\includegraphics[height=.6\textwidth]{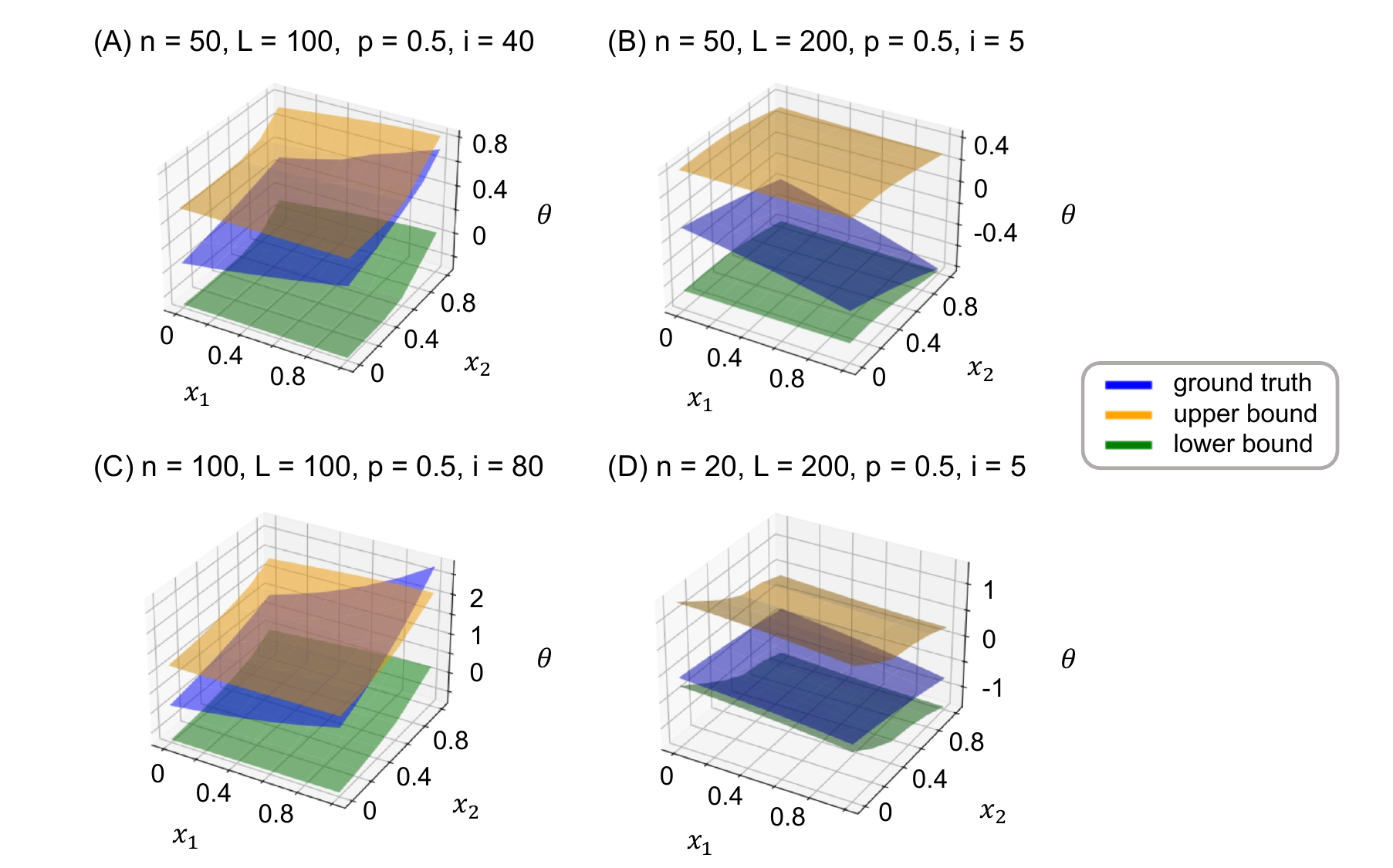}
	\end{center}
	% \vskip-10pt
	\caption{The converge of the confidence band over the ground truth with different $n$, $L$, $p$ and the model index $i$. We fix $x_3 = 0.4$ and vary $x_1$ and $x_2$ uniformly over $[0,1]$.}
\end{figure}

Finally, we examine the empirical performance of the confidence diagram for ranking, as proposed in Section~\ref{sec:confidence_diagram}. For this evaluation, we set $n=20$, $L=100$, $p = 0.2$, and the significance level $\alpha=0.1$. The latent score  
$\theta^*(\xb)$ is defined as
$
\theta_i^*(\xb) = i e^{\sum_{1}^{3}{x_i}} + i.
$
Figure~\ref{fig:hasse} shows an example of the confidence diagram using Algorithm~\ref{alg:confidence_with_stepdown} on synthetic data. The confidence diagram shows the possible ranks of each item by presenting a partial order among all items, which, in our case, correspond to LLM models. For example, in Figure~\ref{fig:hasse}, \textsc{Model}~20 is the only model at the highest level (level 3), indicating that it ranks higher than all other models, with a possible rank of $\{1\}$. For \textsc{Model} 15 at level 2, it is ranked lower than the \textsc{Model}~20 at level 3 but above all the 12 models at level 1. Therefore, the possible ranks for \textsc{Model} 15, as well as all other models at level 2, are $\{2, 3, 4, 5, 6, 7, 8\}$. Meanwhile, \textsc{Model} 1, located at level 1, is ranked lower than all models at levels 2 and 3, with possible ranks of $\{9, 10, 11, 12, 13, 14, 15, 16, 17, 18, 19, 20\}$. Following the reasoning, we can determine the possible ranks for all models.

To assess the reliability of the confidence diagrams, we repeat the generating scheme 50 times. Figure~\ref{fig:heatmap} shows the frequency of possible ranks indicated by the confidence diagram. The heatmap shows that the possible ranks are concentrated along the diagonal, indicating that the confidence diagrams cover the true rankings of the models with high probability. 

\begin{figure}[htb]\label{fig:hasse}
	\begin{center}
		\includegraphics[height=.4\textwidth]{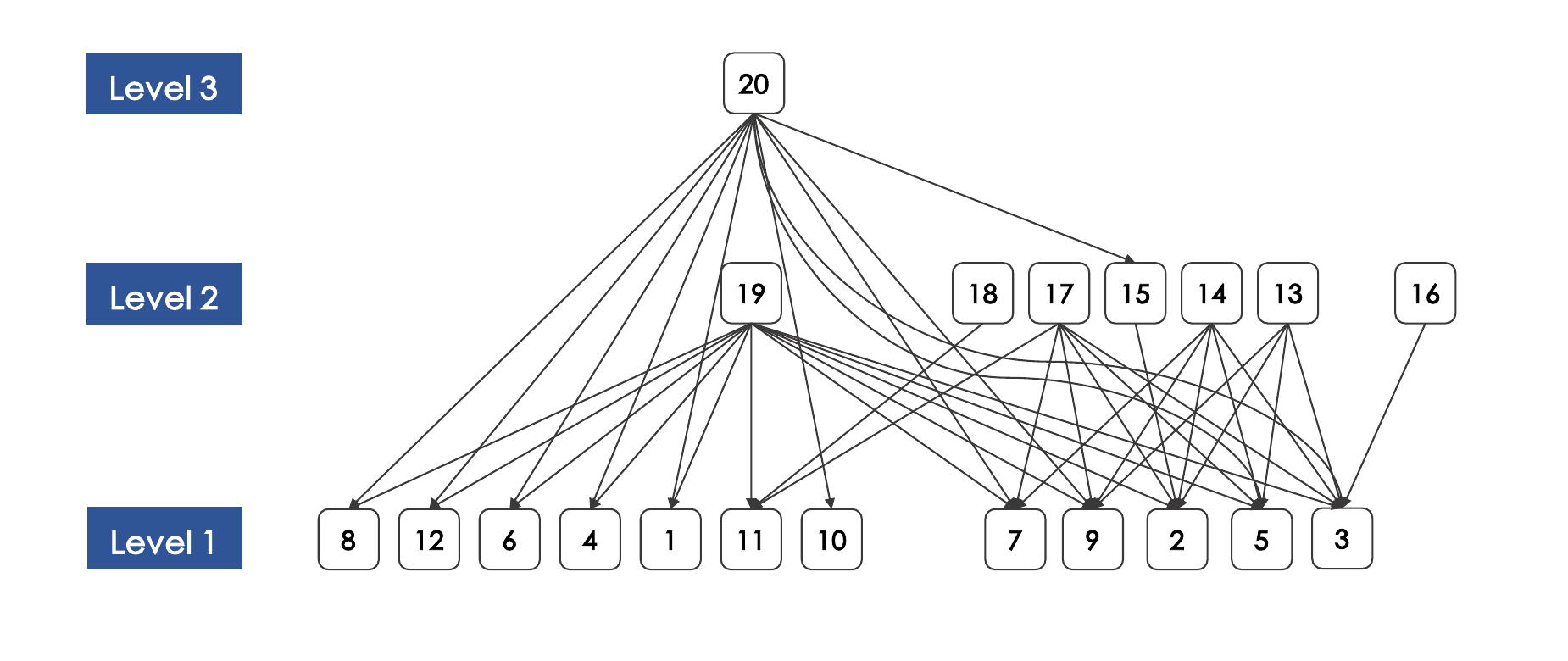}
	\end{center}
	% \vskip-30pt
	\caption{Example of the constructed Hasse diagram with $n=20$, $L=100$ and $p=0.2$. }
\end{figure}

\begin{figure}[h]\label{fig:heatmap}
	\begin{center}
		\includegraphics[height=.65\textwidth]{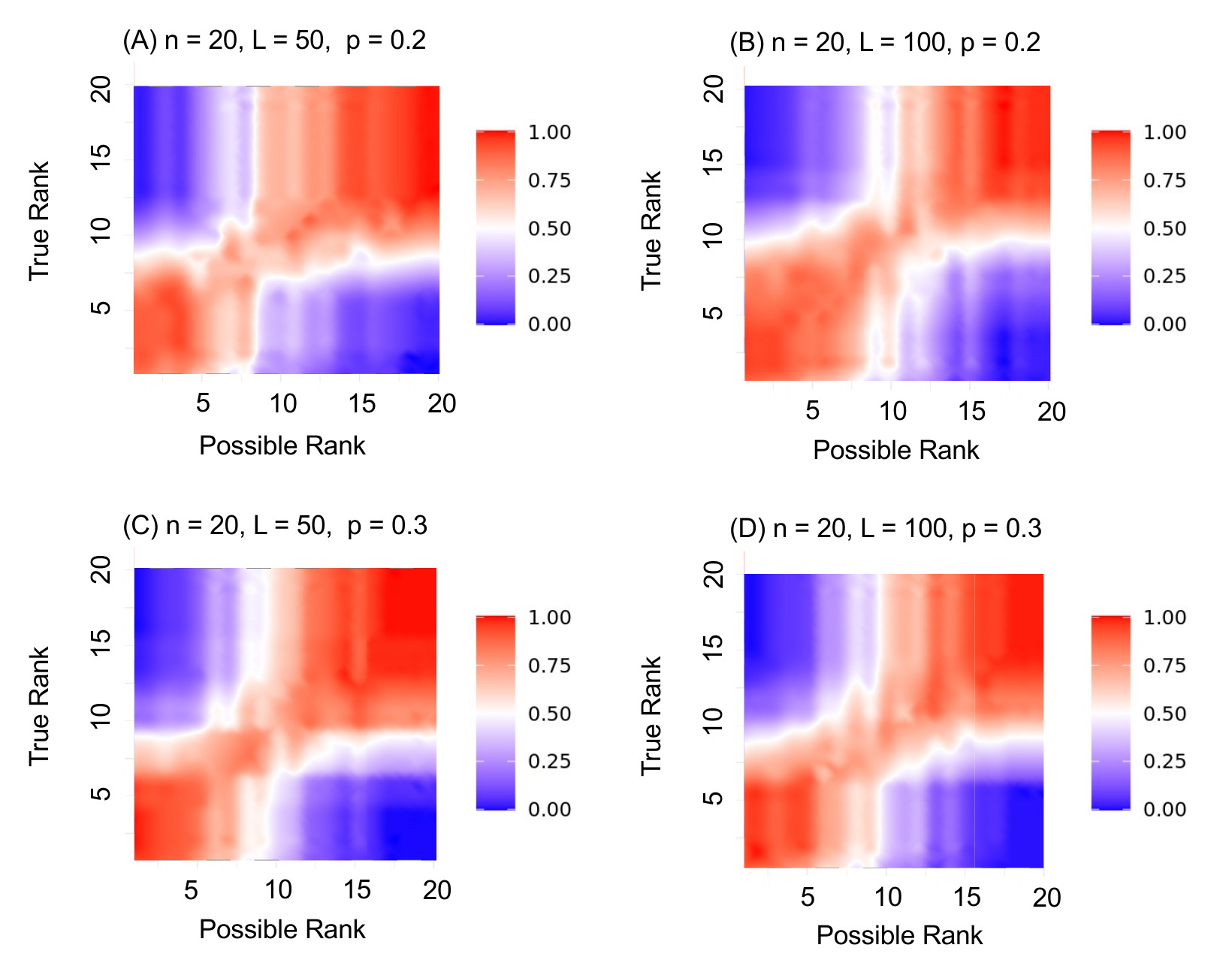}
	\end{center}
	% \vskip-20pt
	\caption{Heatmap for the possible ranks indicated by the confidence diagram, comparing the frequency of the possible ranks with the true ranks. (A) We fix $n=20$, $L=50$ and $p=0.2$. (B) We fix $n=20$, $L=100$ and $p=0.2$.}
\end{figure}

\subsection{Real Data}
We apply our method to rank multiple language models on the Massive Multitask Language Understanding (MMLU) dataset \citep{hendrycks2021measuring}.  In particular, we conduct experiments on four datasets: MMLU Anatomy, MMLU Clinical Knowledge, MMLU College Biology, and MMLU Medical Genetics. These datasets contain multiple-choice questions on various medical subjects. We include five language models in the ranking: \textsc{LLaMA-1} \citep{touvron2023llamaopenefficientfoundation}, \textsc{LLaMA-2} \citep{touvron2023llama}, \textsc{Alpaca} \citep{alpaca}, \textsc{GPT-3.5}, and \textsc{GPT-4o mini}. Specifically, we generate an Erdős-Rényi random graph with an edge probability of $p=0.5$ to determine the comparison pairs of the five models. For each pair of datasets presented in the graph, we conduct $L=100$ sets of pairwise comparisons, and then use \textsc{GPT-4} to evaluate the pairwise responses of the five models. 
Figure~\ref{fig:tr_pairs} shows the comparison counts between the set of models using \textsc{GPT-4}. 
\begin{figure}[htb]\label{fig:tr_pairs}
	\begin{center}
		\includegraphics[height=.25\textwidth]{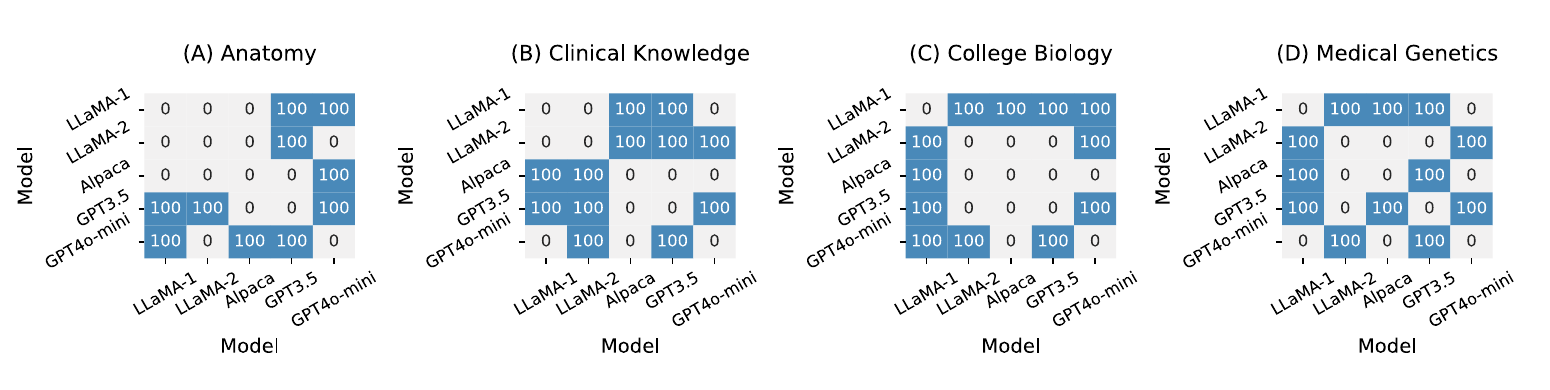}
	\end{center}
	\vskip-20pt
	\caption{Comparison count between the set of models over different subjects. The comparison pairs are selected through an Erdős-Rényi random graph with an edge probability of $p=0.5$.}
\end{figure}

We first transform the prompts into representation vectors using OpenAI's text embedding model, \textsc{text-embedding-3-small} \citep{openai2023textembedding}. We then employ principal component analysis (PCA) \citep{jolliffe2002principal} to reduce the embedding dimension from 1536 to 16 in our estimation to the latent score. Figure~\ref{fig:mmlu} illustrates the pairwise winning probabilities between the five language models on all the subjects over 50 prompts. Across the four subjects, \textsc{GPT-4o mini} consistently outperforms the other models, while \textsc{LLaMA-1} shows the poorest overall performance. As for the other models, their performance varies across different subjects. For example, \textsc{LLaMA-2} and \textsc{GPT-3.5} show dominance over \textsc{Alpaca} in MMLU Anatomy and MMLU Clinical Knowledge. In contrast, \textsc{Alpaca} performs well in MMLU College Biology, and \textsc{GPT-3.5} performs especially well in MMLU Medical Genetics. Figure~\ref{fig:hasse_llm} shows the confidence diagram using Algorithm~\ref{alg:confidence_with_stepdown}, providing a visualized guide to the choice of language models in different subjects.

\begin{figure}[H]\label{fig:mmlu}
	\begin{center}
		\includegraphics[height=.7\textwidth]{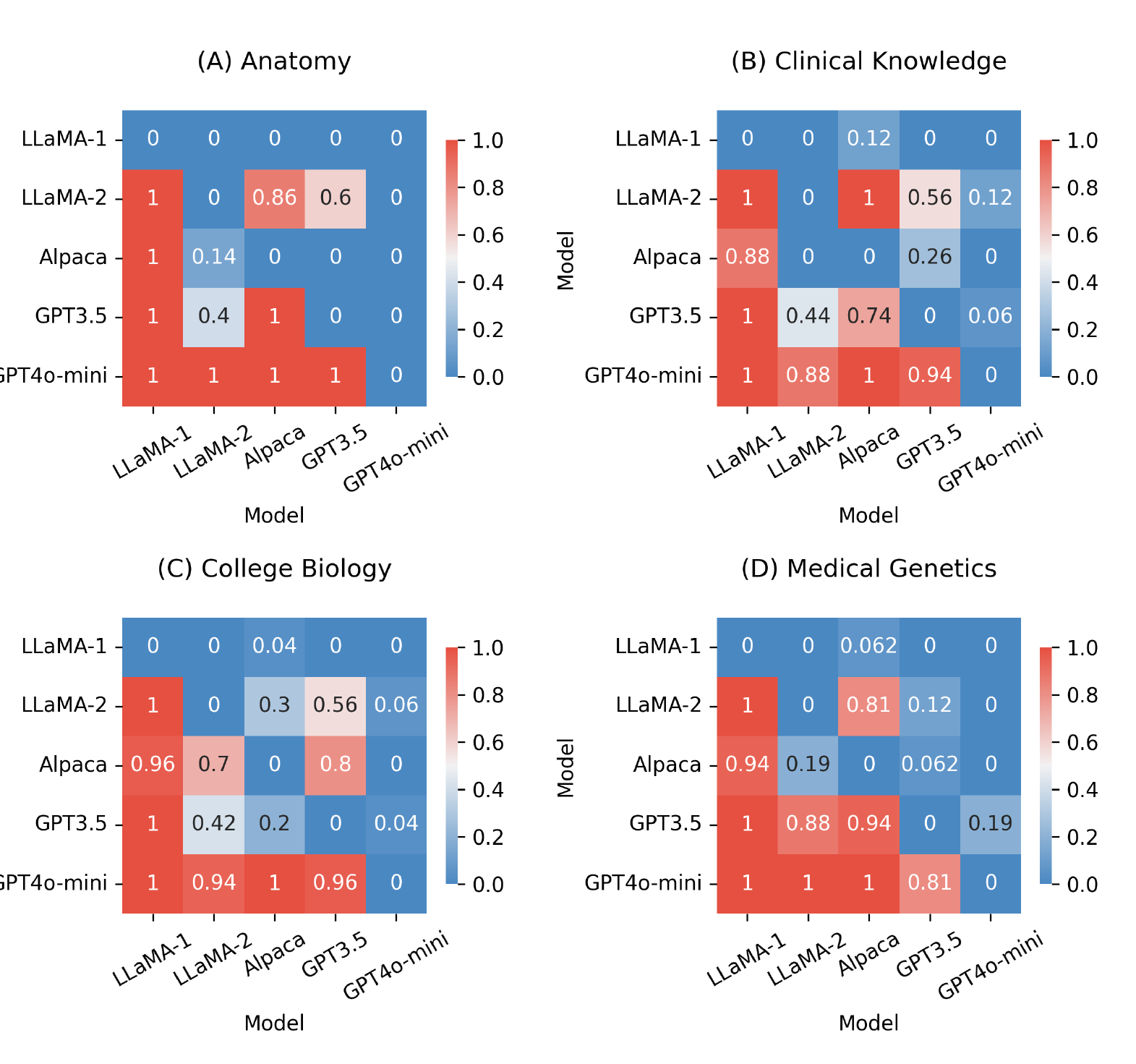}
	\end{center}
	\vskip-20pt
	\caption{Application of our estimation method on MMLU Dataset across different subjects. The four panels display the pairwise win rate between the set of models based on our proposed estimation method.}
\end{figure}

\begin{figure}[h]\label{fig:hasse_llm}
	\begin{center}
		\includegraphics[height=.25\textwidth]{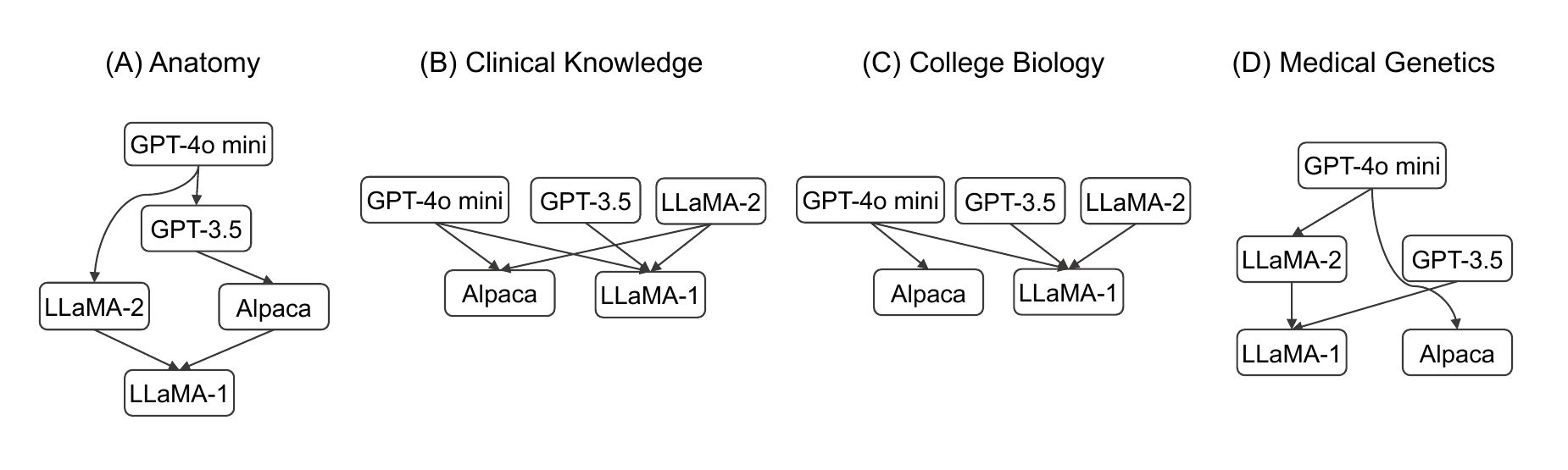}
	\end{center}
	\vskip-20pt
	\caption{Application of our proposed confidence diagram on the ranking of the five compared language models by Algorithm~\ref{alg:confidence_with_stepdown}. The four panels display the hierarchical ranking of the five language models on different subjects.}
\end{figure}

%% file: content/appendix.tex
% !TEX root = ../main.tex

\section{Proof of Theorems in Section~\ref{sec:theory}}\label{sec:pf_theory}

\subsection{Proof of Theorem~\ref{thm:estimate} }\label{pf:thm:estimate}

We provide the proof for Theorem \ref{thm:estimate}, the main theorem concerning the rates of our proposed estimator.

First, we introduce the gradient descent algorithm used to compute the estimator. The proposed algorithm is as follows:

\begin{algorithm}[h]
\caption{Gradient Descent}\label{al:gradient_descent}
\begin{algorithmic}
\STATE \textbf{Input}: Regularized likelihood function $\mathcal{L}_{\lambda}(\cdot, \cdot)$, gradient function $\nabla \mathcal{L}_{\lambda}(\cdot, \cdot)$
\STATE \textbf{Initialization}: Set $\btheta^0(\xb) = \btheta^*(\xb)$
\FOR{$\gamma = 0$ \textbf{to} $\Gamma$}
    \STATE $\btheta^{(\gamma + 1)}(\xb) = \btheta^\gamma(\xb) - \eta \nabla \mathcal{L}_\lambda (\btheta^\gamma ; \xb)$
\ENDFOR
\STATE \textbf{Output}: The estimator $\btheta^\Gamma(\xb)$
\end{algorithmic}
\end{algorithm}

For the validity of the proof, we set the penalty term as $\lambda \asymp \frac{1}{n} (h^2 + \sqrt{\frac{\log (nh^{d/2-1})}{npLh^d}})$ and the stepsize as $\eta \lesssim \min \{ \frac{1}{\lambda + \frac{C^*}{n}}, 1 \}$ for some sufficiently large positive constant $C^*$. It is important to note that, as pointed out by \citet{chen2019spectral},  initializing with $\btheta^0(\xb) = \btheta^*(\xb)$ is not feasible since $\btheta^*(\xb)$ is unknown in advance. We choose this initialization only for the purpose of analysis.

\begin{proof}

To prove Theorem \ref{thm:estimate}, we need to bound the distance $\big\| \widehat{\btheta}(\xb) - \btheta^*(\xb) \big\|_\infty$. We introduce the intermediate estimator $\btheta^\Gamma(\xb)$, obtained from Algorithm \ref{al:gradient_descent} after a sufficiently large number of iterations $\Gamma$. We then bound $\big\| \btheta^\Gamma(\xb) - \btheta^*(\xb) \big\|_\infty$ and $\big\| \widehat{\btheta}(\xb) - \btheta^\Gamma(\xb) \big\|_\infty$ separately. We use index $\gamma \in [\Gamma]$ to track the current step of iteration.

For the bound of $\big\| \btheta ^ \Gamma(\xb) - \btheta ^*(\xb) \big\|_\infty$, we establish our induction-based proof in the following lemma.

\begin{lem}[Induction for Distance between $\btheta^{\gamma}(\xb)$ and $\btheta ^*(\xb)$]\label{lem:infty_whole}
    Given that the inequalities \eqref{equ:two_whole}-\eqref{equ:two_loo} hold for some positive $\gamma$, there exists some positive constant $C$, such that
    \begin{equation}\label{equ:infty_whole}
    \sup_{\xb \in \Omega} \big\| \btheta^{\gamma + 1}(\xb) - \btheta ^ * (\xb) \big \|_\infty \leq C \bigg(h ^ 2 + \sqrt{\frac{\log (nh^{d/2-1})}{npLh ^ d}} \bigg)
    \end{equation}
with probability at least $1 - O(n ^ {-10})$.
\end{lem}
    
The induction process is based on Lemmas \ref{lem:two_whole}--\ref{lem:two_loo} under the assumptions such that for some positive integers $C_1, C_2$, and $C_3$, the following inequalities hold with high probability:

\begin{equation}\label{equ:two_whole}
    \sup_{\xb \in \Omega} \big \| \btheta^\gamma(\xb) - \btheta ^ * (\xb) \big \|_2 \leq C_1 \sqrt{n} \bigg(h ^ 2 + \sqrt{\frac{\log (nh ^ {d/2-1})}{npLh^d}} \bigg);
\end{equation}

\begin{equation}\label{equ:abs_loo}
    \sup_{\xb \in \Omega} \max_{1 \leq m \leq n} \big| \theta^{\gamma, (m)}_m(\xb) - \theta_m ^ * (\xb) \big| \leq C_2 \bigg(h ^ 2 + \sqrt{\frac{\log (nh ^ {d/2-1})}{npLh^{d}}} \bigg);
\end{equation}

\begin{equation}\label{equ:two_loo}
    \sup_{\xb \in \Omega} \max_{1 \leq m \leq n} \big\| \btheta^\gamma(\xb) - \btheta ^ {\gamma, (m)} (\xb) \big \|_2 \leq C_3 \bigg(h ^ 2 + \sqrt{\frac{\log (nh ^ {d/2-1})}{npLh^d}} \bigg).
\end{equation}
Here, $\boldsymbol\theta^{\gamma,(m)}$ represents the estimate updated at the $\gamma$-th iteration following the leave-one-out technique inspired by \citet{chen2019spectral}, and we provide  more details about the leave-one-out technique in Section \ref{sec:four_ineqs}. Finally, our setup as stated in Algorithm~\ref{al:gradient_descent} guarantees that the inequalities \eqref{equ:two_whole}-\eqref{equ:two_loo} hold for $\gamma = 0$.

For the bound of $\big\| \widehat{\btheta}(\xb) - \btheta ^ \Gamma(\xb) \big\|_\infty$, we apply Lemma~\ref{lem:distance_hat} for sufficiently large $\Gamma$. 

\begin{lem}[Distance between $\widehat{\btheta}(\xb)$ and $\btheta^\Gamma(\xb)$]\label{lem:distance_hat} 
    Consider the estimator $\btheta ^ \Gamma(\xb)$ after $\Gamma$ steps of iteration through Algorithm \ref{al:gradient_descent}. There exists some sufficiently large number of iteration steps $\Gamma = \Omega \Bigg(\frac{\sqrt{n}}{h^4 + \frac{\log(nh ^ {d/2-1})}{npLh ^ d}} \Bigg)$ that guarantees the existence of some positive constant $C$ such that
    \begin{equation*}
        \sup_{\xb\in \Omega} \big\|\btheta ^ \Gamma(\xb) - \widehat{\btheta}(\xb) \big\|_\infty \leq C \bigg( h^2 + \sqrt{\frac{\log(nh ^ {d/2-1})}{npLh^d}} \bigg),
    \end{equation*}
    with probability at least $1 - O(n^{-10})$.
\end{lem}

The proofs of Lemma \ref{lem:infty_whole} and Lemma \ref{lem:distance_hat} are provided in Section \ref{pf:lem_estimate}. Combining the two lemmas directly implies that Theorem \ref{thm:estimate} holds.
\end{proof}

\subsection{Proof of Theorem \ref{thm:GMB_validity}}\label{pf:GMB_validity}
\begin{proof}
    We first define the discretized version of $T_0$ and $W_0$ as $\tilde{T}_0$ and $\tilde{W}_0$. Let
    \begin{equation}\label{equ:def_Ttilde}
        \tilde{T}_0 = \max_{\xb_q \in \mathcal{N}_\epsilon}  \frac{1}{\sqrt{n}} \sum_{i = 1}^n f_i(\xb_q, X_i),
    \end{equation}
    and
    \begin{equation}\label{equ:def_Wtilde}
        \tilde{W}_0 = \max_{\xb_q \in \mathcal{N}_\epsilon} \frac{1}{\sqrt{n}} \sum_{i = 1}^n \xi_i f_i(\xb_q, X_i),
    \end{equation}
    where $\mathcal{N}_\epsilon$ is an $\epsilon$-net  on $\bX$.

   Suppose that  there exist \(\zeta_1, \zeta_2, \zeta_1' \geq 0\) and \(\zeta_2' \geq 0\) such that if

    \begin{equation}\label{eqn:orig_cdn}
        (1).\ \mathbb{P}(|T - T_0| >  \zeta_1) < \zeta_2 \quad \text{and} \quad (2).\  \mathbb{P}(\mathbb{P}(| W - W_0 | > \zeta_1|X_{1:n}) > \zeta_2) < \zeta_2,
    \end{equation}
    we have that
    \begin{equation}\label{eqn:new_cdn}
        (1)'. \mathbb{P}(| T - \widetilde{T}_0| > \zeta_1') < \zeta_2' \quad \text{and} \quad (2)'. \mathbb{P}(\mathbb{P}(| W - \widetilde{W}_0 | > \zeta_1'|X_{1:n}) > \zeta_2') < \zeta_2'.
    \end{equation}

    Then by Theorem 3.2 of \citet{chernozhukov2013gaussian} and \eqref{eqn:new_cdn}, there exist some constants $c$ and $C$ such that
    \begin{equation}
        \sup_{\alpha \in (0,1)} \big|\PP(T \leq c_w (\alpha))-\alpha\big| \leq Cn^{-c}.
    \end{equation}

    Therefore, to prove our claim, it is sufficient to show that \eqref{eqn:orig_cdn} is a sufficient condition of \eqref{eqn:new_cdn}. The proof is inspired by \citet{imaizumi2021gaussian}. By the triangle inequality, we have
    \begin{equation}\label{eqn:T_tri}
        \mathbb{P}(|T - \widetilde{T}_0| > \zeta_1 + |T_0 - \widetilde{T}_0|) < \zeta_2 ,
    \end{equation}
    and 
    \begin{equation}\label{eqn:W_tri}
        \mathbb{P}( \mathbb{P}( |W - \widetilde{W}_0| >  \zeta_1 +|W_0 - \widetilde{W}_0|  | X_{1:n}) \geq \zeta_2 ) \leq \zeta_2.
    \end{equation}
Let the region $\bX_\epsilon^2$ be
    \begin{equation}\label{equ:def_Tepsilon}
        \bX_\epsilon^2 = \{(\xb, \xb'): \xb, \xb' \in \bX; \|\xb- \xb'\| \leq \epsilon \}.
    \end{equation}
    For any  \(\xb \in \bX\), there exists some \(\xb' \in \mathcal{N}_\varepsilon\) with \(||\xb- \xb'|| \leq \epsilon\) and

    $$
    \begin{aligned}
    \frac{1}{n} \sum_{i=1}^n [ f_i(\xb, X_i) - \max_{\xb_q \in \mathcal{N}_\epsilon} f_i(\xb_q, X_i) ] \leq &\  \frac{1}{n} \sum_{i=1}^n [f_i(\xb, X_i) - f_i(\xb', X_i)]  \\
    \leq & \sup_{(\xb,\xb') \in \bX_\epsilon^2} \frac{1}{n} \sum_{i=1}^n [ f_i(\xb, X_i) - f_i(\xb', X_i) ]. 
    \end{aligned}
    $$
    Define 
    \begin{equation}\label{equ:delta_T}
        \Delta T^{\max}_\epsilon = \sup_{(\xb,\xb') \in \bX_\epsilon^2} \frac{1}{n} \sum_{i=1}^n \big[ f_i(\xb, X_i) - f_i(\xb', X_i) \big],
    \end{equation}
    and
    \begin{equation}\label{equ:delta_W}
        \Delta W^{\max}_\epsilon = \sup_{(\xb,\xb') \in \bX_\epsilon^2} \frac{1}{n} \sum_{i=1}^n \big[\xi_i f_i(\xb, X_i) - \xi_i f_i(\xb', X_i) \big].
    \end{equation}
    Since the inequality holds for any  \(\xb \in \bX\), we have 
    $$ \frac{1}{\sqrt{n}} (T_0 - \widetilde{T}_0) \leq \Delta T^{\max}_\epsilon. $$
    Similarly, for \(|W_0 - \widetilde{W}_0|\), we have
    $$ \frac{1}{\sqrt{n}} (W_0 - \widetilde{W}_0) \leq \Delta W^{\max}_\epsilon. $$
    By the Talagrand's inequality in \citet{boucheron2013ineq}, there exists \( \xi_1 \) such that
    \begin{equation*}
        \mathbb{P} ( |\Delta T^{\max}_\epsilon| \geq \mathbb{E}[|\Delta T^{\max}_\epsilon|]+ \frac{\epsilon \sqrt{2\xi_1}}{\sqrt{n}} + \frac{2 \xi_1 \eta + \sqrt{4 \xi_1 \eta^2}}{n} ) \leq 1-e^{-\xi_1}.
    \end{equation*}
    By the Borel-TIS inequality, there exists \( \xi_2 \) such that
    \begin{equation*}
        \mathbb{P} (|\Delta W^{\max}_\epsilon| \geq \mathbb{E}[|\Delta W^{\max}_\epsilon|] + \frac{\epsilon \eta \sqrt{2\xi_2}}{\sqrt{n}} | X_{1:n} ) \leq 1-e^{-\xi_2},
    \end{equation*}
Letting \( \phi(\epsilon) = \max (\mathbb{E}[|\Delta T^{\max}_{\epsilon}|], \mathbb{E}[|\Delta W^{\max}_{\epsilon}|])\), we have,     for some positive \( \xi \),
    \begin{equation}\label{equ:T_tilda_phi}
        \mathbb{P} ( T_0 - \widetilde{T}_0 \geq \sqrt{n} \phi(\epsilon) + \epsilon \sqrt{2\xi} + \frac{2 \xi \eta + \sqrt{4 \xi \eta^2}}{\sqrt{n}} ) \leq 1-e^{-\xi},
    \end{equation}
    and
    \begin{equation}\label{equ:W_tilda_phi}
        \mathbb{P} ( W_0 - \widetilde{W}_0 \geq \sqrt{n} \phi(\epsilon) + \epsilon \eta \sqrt{2\xi} | X_{1:n} ) \leq 1-e^{-\xi}.
    \end{equation}

Then, we bound $\phi(\epsilon)$. Let \[
    Z_{\xb, \xb'} = \frac{1}{n} \sum_{i=1}^n [ f_i(\xb, X_i) - f_i(\xb', X_i) ], \quad Y_{\xb, \xb'} = \frac{1}{n} \sum_{i=1}^n [ \xi_i f_i(\xb, X_i) - \xi_i f_i(\xb', X_i) ].
    \]
By \eqref{equ:delta_T} and \eqref{equ:delta_W}, we have \(\Delta T^{\max}_\epsilon = \sup_{(\xb,\xb') \in \bX_\epsilon^2} Z_{\xb, \xb'} \text{ and } \Delta W^{\max}_\epsilon = \sup_{(\xb,\xb') \in \bX_\epsilon^2} Y_{\xb, \xb'}\).

    We consider a centered version of $T_0$ and $W_0$, \( \overline{T}_0 \text{ and } \overline{W}_0 \) where $$\overline{\Delta T_{\varepsilon}^{\max}} = \sup_{(\xb,\xb') \in \bX_\epsilon^2} [ Z_{\xb, \xb'} - \mathbb{E}( Z_{\xb, \xb'})] \quad \text{and}\quad \overline{\Delta W_{\varepsilon}^{\max}} = \sup_{(\xb,\xb') \in \bX_\epsilon^2} [Y_{\xb, \xb'} - \mathbb{E}( Y_{\xb, \xb'})]. $$ 
    By  Assumption~\ref{ass:GMB_Lipschitz}, we have
    \begin{equation}\label{equ:delta_T_c}
    0 \leq \Delta T_{\varepsilon}^{\max} \leq \overline{\Delta T_{\varepsilon}^{\max}} + \sup_{(\xb,\xb') \in \bX_\epsilon^2} Z_{\xb, \xb'} \leq \overline{\Delta T_{\varepsilon}^{\max}} +  K_{\cF} \varepsilon,
    \end{equation}
    and
    \begin{equation}\label{equ:delta_W_c}
    0 \leq \Delta W_{\varepsilon}^{\max} \leq \overline{\Delta W_{\varepsilon}^{\max}} + \sup_{(\xb,\xb') \in \bX_\epsilon^2} Y_{\xb, \xb'} \leq \overline{\Delta W_{\varepsilon}^{\max}} +  K_{\cF} \varepsilon.
    \end{equation}

    We first bound $\mathbb{E} [ |\Delta W_{\varepsilon}^{\max}||{X_{1:n}} ]$. Let $ J(\epsilon) = \int_0^\epsilon \sqrt{1+\log N(\bX, \|\cdot\|, t)} \dd t$. By the maximal inequality in Corollary 2.2.8 in \citet{van1996weak},we have
    \begin{equation*}
    \mathbb{E} [ |\overline{\Delta W_{\varepsilon}^{\max}}||{X_{1:n}} ] = \mathbb{E} \sup_{(\xb,\xb') \in \bX_\epsilon^2} \frac{1}{n} \sum_{i=1}^n [\xi_i f_i(\xb, X_i) - \xi_i f_i(\xb', X_i) ] \leq \frac{1}{\sqrt{n}} J(\eta \epsilon).
    \end{equation*}
    Plugging \eqref{equ:delta_W_c} into the above, we have
    \begin{equation}\label{eqn:W_delta_J}
    \mathbb{E} [ |\Delta W_{\varepsilon}^{\max}||{X_{1:n}} ] \leq \frac{1}{\sqrt{n}} J(\eta \epsilon) +  K_{\cF} \epsilon.
    \end{equation}
    We then bound $\mathbb{E} [ |\Delta T_{\varepsilon}^{\max}| ]$. By Assumption \ref{ass:GMB_Lipschitz}, we first have that for any   \((\xb_1, \xb_2), (\xb_1', \xb_2') \in \bX_\epsilon^2\), 
    \begin{equation}\label{eqn:zxxx}
    \begin{aligned}
    |Z_{\xb_1,\xb_1'} - Z_{\xb_2,\xb_2'}| \leq &\  | \frac{1}{n} \sum_{i=1}^n [ f_i(\xb_1, X_i) - f_i(\xb_2, X_i)] | + | \frac{1}{n} \sum_{i=1}^n [ f_i(\xb_1', X_i) - f_i(\xb_2', X_i) ] | \\
    \leq & \ K_{\cF} (\norm{\xb_1 - \xb_2} + \norm{\xb_1' - \xb_2'}) \\ 
    = & \  K_{\cF} \cdot d_{\|\cdot\|}((\xb_1,\xb_1'), (\xb_2,\xb_2')), \\
    \end{aligned}
    \end{equation}
    where the distance \( d_{\|\cdot\|}(\cdot) \) on \( \bX_\epsilon^2 \) is defined as \( d_{\|\cdot\|}((\xb, \xb'), (\yb, \yb')) = \|\xb - \yb\| + \|\xb' - \yb'\| \). 

Meanwhile, by Lemma \ref{lem:covering_ttprime}, we have
    \begin{equation*}
        \log N(\bX_\epsilon^2, d_{\|\cdot\|}, \delta) \leq \log N(\bX, \|\cdot\|, \delta/2) + \log N(B_{\|\cdot\|}(\xb, \epsilon + \delta), \|\cdot\|, \delta/2),
    \end{equation*}
and by Assumption \ref{ass:GMB_covering} and letting \( \delta + \epsilon < 1 \), it follows that
    \begin{equation}\label{eqn:covering}
        \log N(\bX_\epsilon^2, d_{\|\cdot\|}, \delta) \leq 2 \log N(\bX, \|\cdot\|, \delta/4).
    \end{equation}
    Let \( \mathcal{L}_{\epsilon} = \{ Z_{\xb, \xb'} - \mathbb{E}(Z_{\xb, \xb'}) | (\xb,\xb') \in \bX_\epsilon^2 \} \). By Assumption \ref{ass:GMB_dominate}, there exists an envelope function \( F_Z (x) \) such that \( \norm{F_Z}_\infty \leq 2\eta \). Following \eqref{eqn:zxxx}, by the standard Rademacher complexity technique and Assumption~\ref{ass:GMB_Lipschitz}, there exists a constant \( K_1 >0\) such that
    \begin{align*}
        &\mathbb{E}_T [ \sup_{\mathcal{L}_{\epsilon}} [Z_{\xb, \xb'} - \mathbb{E}( Z_{\xb, \xb'})] ]    \leq K_1 \int_{0}^{2\eta} \sqrt{\frac{\log N(\mathcal{L}_{\epsilon}, \|\cdot\|_{L^\infty}, \delta)}{n}} \dd \delta \\
        & \leq K_1 \int_{0}^{2\eta\epsilon} \sqrt{\frac{\log N(\bX_\epsilon^2, \|\cdot\|,  K_{\cF} \delta)}{n}} \dd \delta \\
        & \leq K_1 \int_{0}^{2\eta \epsilon} \sqrt{\frac{2\log N(\bX, \|\cdot\|,  K_{\cF} \delta/4)}{n}} \dd \delta \\
        & \leq \frac{1}{\sqrt{n}} K_1 J( K_{\cF} \eta \epsilon/2),
    \end{align*}
    where the third equality follows \eqref{eqn:covering}.
    Therefore, together with \eqref{equ:delta_T_c}, we have 
    \begin{equation}\label{eqn:T_delta_J}
    \mathbb{E} [ | \Delta T_{\varepsilon}^{\max}| ] \leq \frac{1}{\sqrt{n}} K_1 J( K_{\cF} \eta \epsilon/2) +  K_{\cF} \epsilon.
    \end{equation}

    Combining \eqref{eqn:W_delta_J} and \eqref{eqn:T_delta_J}, we have 
    \begin{align*}
        \phi(\epsilon)  \leq \max \Big( \frac{1}{\sqrt{n}} J(\eta \epsilon) +  K_{\cF} \epsilon , K_1 \frac{1}{\sqrt{n}} J( K_{\cF} \eta \epsilon/2) +  K_{\cF} \epsilon\Big)  \leq K_1  \Big(  \frac{J(\eta \epsilon) + J( K_{\cF} \eta \epsilon/2)}{\sqrt{n}} +  K_{\cF} \epsilon  \Big),
    \end{align*}
    which completes the proof for bounding $\phi(\epsilon)$.

    Now we apply the bound of $\phi(\epsilon)$ to \eqref{equ:T_tilda_phi} and \eqref{equ:W_tilda_phi}. Taking \( \epsilon = 1/n \), we have
$$
        \mathbb{P} \Big( T_0 - \widetilde{T}_0 \geq K_1 J(\eta/n) + \frac{K_2}{\sqrt{n}} \Big) \leq 1-e^{-\xi},
$$
    and
$$
        \mathbb{P} \Big( \mathbb{P} \Big( W_0 - \widetilde{W}_0 \geq K_1 J(\eta/n) + \frac{K_2}{\sqrt{n}}  \Big| X_{1:n} \Big) \geq 1-e^{-\xi} \Big) \leq 1-e^{-\xi},
$$
    for some positive \( \xi \).

Applying the triangle inequality to \eqref{eqn:T_tri} and \eqref{eqn:W_tri}, we have
  $$
        \mathbb{P} \Big(|T - \widetilde{T}_0| > \zeta_1 + K_1 J(\eta/n) + \frac{K_2}{\sqrt{n}} \Big) \leq \zeta_2 (1-e^{-\xi}),
$$
    and 
   $$
        \mathbb{P}\Big( \mathbb{P} \Big( |W - \widetilde{W}_0| >  \zeta_1 +K_1 J(\eta/n) + \frac{K_2}{\sqrt{n}}  \Big| X_{1:n}\Big) \geq \zeta_2 (1-e^{-\xi}) \Big) \leq \zeta_2 (1-e^{-\xi}).
 $$

    Therefore, there exist \( \zeta_1' = \zeta_1 + K_1 J(\eta/n) + \frac{K_2}{\sqrt{n}} \text{ and } \zeta_2' = \zeta_2 (1-e^{-\xi}) \) such that 
    $$
        \mathbb{P}(| T - \widetilde{T}_0| > \zeta_1') < \zeta_2' \quad \text{and} \quad  \mathbb{P}(\mathbb{P}(| W - \widetilde{W}_0 | > \zeta_1'|X_{1:n}) > \zeta_2') < \zeta_2'.
    $$
Thus, we have shown that \eqref{eqn:orig_cdn} is a sufficient condition of \eqref{eqn:new_cdn},    which completes the proof.
\end{proof} 

\subsection{Proof of Theorem \ref{thm:band}}\label{pf:thm:band}
\begin{proof}
We first derive the non-asymptotic expansion of $T$. Let $\mathcal{U}_{m}(\xb)$ and $\mathcal{V}_{m}(\xb)$ be
    \begin{equation}\label{equ:Gao_derivative}
        \mathcal{U}_{m} (\xb) = \frac{1}{npL} \sum_{i \neq m} \sum_{\ell \in [L_{im}]} \mathcal{E}_{im} \mathcal{K}_h(\bX_{im} ^ \ell - \xb) \Big(-y_{im}(\bX_{im} ^ \ell) + \psi(\theta^*_m(\bX_{im} ^ \ell) - \theta^*_i(\bX_{im} ^ \ell)) \Big), 
    \end{equation} 
    and
    \begin{equation*}
        \mathcal{V}_{m} (\xb) = \mathcal{V}_{m} (\xb) = \frac{1}{npL} \sum_{i \neq m} \sum_{\ell \in [L_{im}]} \mathcal{E}_{im} \mathcal{K}_h(\bX_{im} ^ \ell - \xb) \psi'(\theta_m^*(\bX_{im} ^ \ell) - \theta_i^*(\bX_{im} ^ \ell)) ,
    \end{equation*}
    where $\psi(x) = \frac{e^x}{1 + e^x}$ and $\mathcal{E}_{im}=1$ when model $i$ and model $m$ are selected and $\mathcal{E}_{im}=0$ otherwise.

We apply the following lemma on the non-asymptotic expansion of $\widehat{\btheta}(t)$, { It is the rephrased version of Theorem 2.2 in \citet{gao2023uncertainty}. We provide it here for self-completeness.} 

\begin{lem}\label{thm:LTO_asymptotic} [Non-Asymptotic Expansion for $\widehat{\btheta}(\xb)$]
    Given Assumptions \ref{ass:erdos_p}, \ref{ass:constraint_lps}, \ref{ass:constraint_theta}, \ref{ass:kernel_general} and \ref{ass:kernel_smooth}, 
    for any $m \in [n]$, 
    \begin{equation}\label{equ:nonasymptotic_expansion}
        \widehat{\theta}_m(\xb) - \theta ^ *_m (\xb) = -(1 + \epsilon_{1,m}(\xb))\frac{\mathcal{U}_{m} (\xb)}{\mathcal{V}_{m} (\xb)} + \epsilon_{2,m}(\xb),
    \end{equation}
    holds with
    \begin{equation*}
        \|\bepsilon_1 \|_\infty = o(1), \ \ \ \|\bepsilon_2 \|_\infty = o\bigg(dh^2 + \sqrt{\frac{\log(n h^{d/2-1})}{npLh^d}} \bigg),
    \end{equation*}
    with probability at least $1 - O(n^{-10})$. Specifically, define the infinite norm of the error term as
    \[ 
        \|\bepsilon_1 \|_\infty =  \sup_{m \in [n], \xb \in \Omega} \bigg| \epsilon_{1,m}(\xb) \bigg|, \ \ \ \|\bepsilon_2 \|_\infty =  \sup_{m \in [n], \xb \in \Omega} \bigg| \epsilon_{2,m}(\xb) \bigg|.
    \]
\end{lem}

By Lemma~\ref{thm:LTO_asymptotic}, we have the non-asymptotic expansion of $T$ such that
\begin{equation}\label{equ:T0}
T_0 = \sup_{m \in [n], \xb \in \Omega} \sqrt{h^d \Xi}\cdot \bigg|  \frac{  \mathcal{U}_{m} (\xb)}{\mathcal{V}^*_{m} (\xb)} \bigg|,
\end{equation}
where $\mathcal{V}^*_{m}(\xb)$ is defined as
\begin{equation}\label{equ:Gao_Hessian}
    \mathcal{V}^*_{m} (\xb) = \mathbb{E}[\mathcal{V}_{m} (\xb)] = \frac{1}{npL} \mathbb{E}\Big[\sum_{i \neq m} \sum_{\ell \in [L_{im}]} \mathcal{E}_{im} \mathcal{K}_h(\bX_{im} ^ \ell - \xb) \psi'(\theta_m^*(\bX_{im} ^ \ell) - \theta_i^*(\bX_{im} ^ \ell)) \Big].
\end{equation}

To facilitate the proof, we let
\begin{equation}\label{equ:Gao_T0}
    T^{m}_0 (\xb) = \sqrt{h^d \Xi}\cdot \bigg|  \frac{  \mathcal{U}_{m} (\xb)}{\mathcal{V}^*_{m} (\xb)} \bigg| = \max \bigg(\sqrt{h^d \Xi}\cdot \frac{  \mathcal{U}_{m} (\xb)}{\mathcal{V}^*_{m} (\xb)},-\sqrt{h^d \Xi}\cdot \frac{  \mathcal{U}_{m} (\xb)}{\mathcal{V}^*_{m} (\xb)} \bigg).
\end{equation}
We also define the statistic $W_0$ as the approximation of $T_0$ by
\begin{equation}\label{equ:Def_W0}
    W_0 = \sup_{m \in [n], \xb \in \Omega} \sqrt{h^d \Xi}\cdot \bigg| \frac{  \mathcal{G}_{m} (\xb)}{\mathcal{V}^*_{m} (\xb)} \bigg|,
\end{equation}
where $\mathcal{G}{m}(\xb)$ is defined as
\[
    \mathcal{G}_{m} (\xb) = \frac{1}{npL} \sum_{i \neq m} \sum_{\ell \in [L_{im}]} \xi_{im}^\ell \mathcal{E}_{im} \mathcal{K}_h(\bX_{im} ^ \ell - \xb) \Big(-y_{im}(\bX_{im} ^ \ell) + \psi(\theta^*_m(\bX_{im} ^ \ell) - \theta^*_i(\bX_{im} ^ \ell)) \Big).
\]

Then, following Corollary~\ref{col:GMB_validity}, we have confidence band for $\boldsymbol{\theta}^*(\xb)$. To ensure the validity of Corollary~\ref{col:GMB_validity}, it remains to  check the following conditions required in the corollary:
\begin{enumerate}
    \item There exist $\zeta_1, \zeta_2$ such that $\mathbb{P}(|T - T_0| > \zeta_1) < \zeta_2, \quad \mathbb{P}(\mathbb{P}(|W - W_0| > \zeta_1 | y) > \zeta_2) < \zeta_2 \text{ with } \zeta_1\sqrt{\log (npL)} + \zeta_2 \leq C_0 (npL)^ {-c_0}$, for some positive $c_0, C_0$.
    \item We have $c \leq \mathbb{E}(T^{m}_0 (\xb)^2) \leq C$ for some positive $c, C$.
\end{enumerate}

For the first condition, by Lemmas \ref{lem:zeta_T} and \ref{lem:zeta_W}, we have that there exist \( \zeta_1 =o \Big(\sqrt{\log (npLh ^ {d/2 - 1} )} \Big)\) and \( \zeta_2  = O(n^{-10})\) such that the required conditions hold. 

Then, we consider the second condition. First, by Lemmas \ref{lem:mean_estim} and \ref{lem:var_estim}, we have $\mathcal{V}_m^* (\xb)$ is $O(1)$. For $\mathcal{U}_{m} (\xb)$, we have
\begin{align*}
 &   \Var \big(\mathcal{U}_{m} (\xb) \big) = \frac{1}{n^2 p^2 L^2} \Var \bigg(  \sum_{i \neq m} \sum_{\ell \in [L_{im}]} \mathcal{E}_{im} \mathcal{K}_h(\bX_{im} ^ \ell - \xb) \Big(-y_{im}(\bX_{im} ^ \ell) + \psi(\theta^*_m(\bX_{im} ^ \ell) - \theta^*_i(\bX_{im} ^ \ell)) \Big) \bigg) \\
    &\ \ = \frac{1}{n^2 p^2 L^2} \sum_{i \neq m} \Var \bigg( \mathcal{E}_{im} \sum_{\ell \in [L_{im}]} \mathcal{K}_h(\bX_{im} ^ \ell - \xb) \Big(-y_{im}(\bX_{im} ^ \ell) + \psi(\theta^*_m(\bX_{im} ^ \ell) - \theta^*_i(\bX_{im} ^ \ell)) \Big) \bigg)\\
    & \ \ = \frac{1}{n^2 p^2 L^2} \sum_{i \neq m} \bigg( \mathbb{E} \bigg[\Var \Big( \mathcal{E}_{im} \sum_{\ell \in [L_{im}]} \mathcal{K}_h(\bX_{im} ^ \ell - \xb) \Big(-y_{im}(\bX_{im} ^ \ell) + \psi(\theta^*_m(\bX_{im} ^ \ell) - \theta^*_i(\bX_{im} ^ \ell)) \Big) \big| \mathcal{E}_{im} \Big) \bigg] \\
    &\ \  \quad + \Var \bigg(\mathbb{E} \Big[ \mathcal{E}_{im} \sum_{\ell \in [L_{im}]} \mathcal{K}_h(\bX_{im} ^ \ell - \xb) \Big(-y_{im}(\bX_{im} ^ \ell) + \psi(\theta^*_m(\bX_{im} ^ \ell) - \theta^*_i(\bX_{im} ^ \ell)) \Big) \big| \mathcal{E}_{im} \Big] \bigg)\bigg),
\end{align*}
where the second equality holds by the independence of \(\mathcal{E}_{im}\) and \(\bX_{im} ^ \ell\) with respect to $i$ and the third holds by law of total variance.

Note that $$\mathbb{E} \Big[ \mathcal{E}_{im} \sum_{\ell \in [L_{im}]} \mathcal{K}_h(\bX_{im} ^ \ell - \xb) \Big(-y_{im}(\bX_{im} ^ \ell) + \psi(\theta^*_m(\bX_{im} ^ \ell) - \theta^*_i(\bX_{im} ^ \ell)) \Big) \big| \mathcal{E}_{im} \Big] = 0.$$ Therefore, we have
{\small
\begin{align*}
&    \Var \big(\mathcal{U}_{m} (\xb) \big)  = \frac{1}{n^2 p^2 L^2} \sum_{i \neq m} \bigg( \mathbb{E} \bigg[\Var \Big( \mathcal{E}_{im} \sum_{\ell \in [L_{im}]} \mathcal{K}_h(\bX_{im} ^ \ell - \xb) \Big(-y_{im}(\bX_{im} ^ \ell) + \psi(\theta^*_m(\bX_{im} ^ \ell) - \theta^*_i(\bX_{im} ^ \ell)) \Big) \big| \mathcal{E}_{im} \Big) \bigg] \bigg)\\
    &\quad = \frac{1}{n^2 p^2 L^2} \sum_{i \neq m} \bigg( \mathbb{E} \bigg[ \mathcal{E}_{im}^2 \sum_{\ell \in [L_{im}]} \Var \Big(  \mathcal{K}_h(\bX_{im} ^ \ell - \xb) \Big(-y_{im}(\bX_{im} ^ \ell) + \psi(\theta^*_m(\bX_{im} ^ \ell) - \theta^*_i(\bX_{im} ^ \ell)) \Big) \Big) \bigg] \bigg).
\end{align*}}

{
We further apply Theorem 1 of \citet{hardle1997kernel} such that
\begin{equation*}
    \sup_{\xb \in \Omega} \Var \Big(  \mathcal{K}_h(\bX_{im} ^ \ell - \xb) \Big(-y_{im}(\bX_{im} ^ \ell) + \psi(\theta^*_m(\bX_{im} ^ \ell) - \theta^*_i(\bX_{im} ^ \ell)) \Big) \Big) \leq \frac{C}{h^d} \|K\|_2^2 \,
\end{equation*}
for some constant $C$. Therefore, we can conclude that with probability at least $1 - O(n^{-10})$, 
\begin{equation*}
    \Var \big(\mathcal{U}_{m} (\xb) \big) = O\bigg(\frac{npL}{n^2 p^2 L^2 h^d} \bigg) = O\bigg(\frac{1}{npLh^d}\bigg).
\end{equation*}
}
Since for any $(i,m)$, we have $\mathbb{E}\Big[\mathcal{E}_{im} \mathcal{K}_h(\bX_{im} ^ \ell - \xb) \Big(-y_{im}(\bX_{im} ^ \ell) + \psi(\theta^*_m(\bX_{im} ^ \ell) - \theta^*_i(\bX_{im} ^ \ell)) \Big) ]=0$,
 \(T^{m}_0 (\xb)\) is unbiased. Therefore, we have $\mathbb{E}(T^{m}_0 (\xb)^2) = O(1)$, which completes the proof.
\end{proof}

\subsection{Proof of Theorem~\ref{thm:pairwise_test}}\label{pf:thm:pairwise_test}
    \begin{proof}
        We show that for the pairwise test, the Type I error can be well-controlled at the desired level, and the power is asymptotically one with some required signal strength. The proof for the top-$K$ test is similar, and we skip it here to avoid repetition.
        
        First, for the Type I error, we have
        \begin{align*}
           & \sup_{\exists \xb, \theta^*_i(\xb) - \theta^*_j(\xb) \leq 0 } \mathbb{P}_{\theta^*}(\text{reject } H_0) 
             \leq \sup_{\exists \xb,  \theta^*_i(\xb) - \theta^*_j(\xb) \leq 0 } \mathbb{P}_{\theta^*} \Big( \inf_{\xb \in \Omega} \sqrt{h^d \Xi}(\widehat{\theta}_i(\xb) - \widehat{\theta}_j(\xb)) > \hat{c}_{ij}(1-\alpha) \Big) \\
            & \qquad \leq \sup_{\exists \xb,  \theta^*_i(\xb) - \theta^*_j(\xb) \leq 0 } \mathbb{P}_{\theta^*} \bigg( \inf_{\theta^*_i(\xb) - \theta^*_j(\xb) \leq 0} \sqrt{h^d \Xi} \bigg[ \widehat{\theta}_i(\xb)- \widehat{\theta}_j(\xb)- \theta^*_i(\xb) + \theta^*_j(\xb) \bigg] > \widehat{c}_{ij}(1-\alpha) \bigg) \\
            &\qquad  \leq \sup_{\exists \xb,  \theta^*_i(\xb) - \theta^*_j(\xb) \leq 0 } \mathbb{P}_{\theta^*} \bigg( \sup_{\xb \in \Omega} \sqrt{h^d \Xi} \bigg[ \widehat{\theta}_i(\xb)- \widehat{\theta}_j(\xb)- \theta^*_i(\xb) + \theta^*_j(\xb) \bigg] > \widehat{c}_{ij}(1-\alpha) \bigg).
        \end{align*}

        Let \( \tilde{T} _{ij}= \sup_{\xb \in \Omega}\sqrt{h^d \Xi}  \bigg[ \widehat{\theta}_i(\xb)- \widehat{\theta}_j(\xb)- \theta^*_i(\xb) + \theta^*_j(\xb) \bigg] \). 
        We show that there exists $\zeta_1 = O(n^{-\epsilon/4} \sqrt{\log n})$ and $\zeta_2 = O(n^{-10})$ such that the $(1-\alpha)\%$ quantile of $W_{ij}$ in \eqref{equ:Wij} accurately estimates the $(1-\alpha)\%$ quantile of $\tilde{T}_{ij}$. The proof is similar to the proof of Lemma \ref{lem:zeta_T} and \ref{lem:zeta_W}. We skip it here to avoid repetition. Specifically, by Corollary~\ref{col:GMB_validity}, we have that 
        \[
            \sup_{\alpha \in (0,1)}|\mathbb{P}(\tilde{T} _{ij} \leq \hat{c}_{ij}(1-\alpha)) - (1-\alpha)| \xrightarrow[]{} 0  \text{ as } n \xrightarrow[]{} \infty.
        \]
        
        Combining the two inequalities above, we have that
        \begin{equation*}
            \sup_{\exists \xb,  \theta^*_i(\xb) - \theta^*_j(\xb) \leq 0 } \mathbb{P}_{\theta^*}(\text{reject } H_0) \leq \alpha + o(1),
        \end{equation*}
        which shows that the Type I error is controlled as desired.

        Then, we show that the power goes to one asymptotically when the signal strength is strong enough. In particular, for a fixed $\delta>0$, we have
        \begin{align*}
            \mathbb{P}(\text{reject } H_0) & = \mathbb{P}  \Big( \inf_{\xb \in \Omega} \sqrt{h^d \Xi}(\widehat{\theta}_i(\xb) - \widehat{\theta}_j(\xb)) > \hat{c}_{ij}(1-\alpha) \Big)  \\
            & \geq \mathbb{P} \bigg( \inf_{\xb \in \Omega} \sqrt{h^d \Xi} \Big( \theta^*_i(\xb) - \theta^*_j(\xb) \Big) - T_{ij} > \hat{c}_{ij}(1-\alpha) \bigg) \\
            & \geq \mathbb{P} \bigg( \inf_{\xb \in \Omega} \sqrt{h^d \Xi} \Big( \theta^*_i(\xb) - \theta^*_j(\xb) \Big) >2\delta,  T_{ij} < \delta,  \hat{c}_{ij}(1-\alpha)<\delta  \bigg).
        \end{align*}

        By Theorem~\ref{thm:estimate}, we have
        \begin{equation}\label{equ:Tij_bound}
            \mathbb{P} \Big(T_{ij} \leq 2\sqrt{h^d \Xi} \Big(h^2 + \sqrt{\frac{\log (nh ^ {d/2-1})}{npLh^d}} \Big) \Big) \geq 1 - O(n^{-10}).
        \end{equation} 
        Meanwhile, by Corollary~\ref{col:GMB_validity}, we have for any fixed \( \alpha \) and sufficiently large \( L, n \), 
        \begin{equation}\label{equ:cij_bound}
            \mathbb{P} \Big(T_{ij} \geq \hat{c}_{ij}(1-\alpha) \Big) \geq \alpha/2.
        \end{equation} 
        
        Combining \eqref{equ:Tij_bound} and \eqref{equ:cij_bound}, we have
        \[
            \hat{c}_{ij}(1-\alpha) \leq 2\sqrt{h^d \Xi} \Big(h^2 + \sqrt{\frac{\log (nh ^ {d/2-1})}{npLh^d}} \Big),
        \]
        for some sufficiently large constant \( C_0 \) with probability \( 1 - O(n^{-10}) \). 

        Therefore, letting
        \(
            \delta = C \Big(h^2 + \sqrt{\frac{\log (nh ^ {d/2-1})}{npLh^d}} \Big)
        \)
        for some sufficiently large constant \( C \) and  taking $h \leq ({npLd^2})^{-\frac{1}{d+4}}$, we have
        \[
        \inf_{ \forall \xb, \theta^*_i(\xb) - \theta^*_j(\xb) > \delta } P(\text{reject } H_0) > 1 - O(n^{-10}),
        \]
        which completes the proof.     
    \end{proof}
  
\subsection{Proof of Theorem~\ref{thm:confidence_diagram}}\label{pf:thm:confidence_diagram}
\begin{proof}
    Let $\widehat{\mathbf{H}}$ be the output of Algorithm~\ref{alg:confidence_with_stepdown} after $\Gamma$ iterations. For $\ell \in [\Gamma]$, we let the set of pairs that have been rejected by be algorithm be  $\mathcal{S}_\ell$, and let the quantile be $\hat{c}_\ell(\alpha)$. Since the confidence diagram is  based on the pairwise test, we first have
    \begin{align*}
        & \mathbb{P} \Big( R(\boldsymbol{\theta}^*) \in \widehat{\mathbf{H}} \Big)  = \mathbb{P}  \Big( \inf_{\ell \in [\Gamma]} \inf_{(i,j) \in \mathcal{S}_\ell^C, T_{ij} > \hat{c}_\ell(1-\alpha)} (\theta^*_i - \theta^*_j) > 0  \Big) \\
        & \quad = \mathbb{P}  \Big( \sup_{\ell \in [\Gamma]} \sup_{(i,j) \in \mathcal{S}_\ell^C, \theta^*_i - \theta^*_j \leq 0} (T_{ij} - \hat{c}_\ell(1-\alpha) ) \leq 0  \Big) \\
        & \quad \geq \mathbb{P}  \Big( \sup_{\ell \in [\Gamma]} \sup_{(i,j) \in \mathcal{S}_\ell^C, \theta^*_i - \theta^*_j \leq 0} \Big[ \inf_{\xb \in \Omega} \sqrt{h^d \Xi} \Big(\widehat{\theta}_i(\xb)- \widehat{\theta}_j(\xb)- \theta^*_i(\xb) + \theta^*_j(\xb) \Big) - \hat{c}_\ell(1-\alpha) \Big] \leq 0  \Big) 
        % &\quad  \geq 1 - \mathbb{P}  \Big( \sup_{\ell \in [\Gamma]} \sup_{(i,j) \in \mathcal{S}_\ell^C} \Big[ \inf_{\xb \in \Omega} \sqrt{h^d \Xi} \Big(\widehat{\theta}_i(\xb)- \widehat{\theta}_j(\xb)- \theta^*_i(\xb) + \theta^*_j(\xb) \Big) - \hat{c}_\ell(1-\alpha) \Big] > 0  \Big). 
    \end{align*}
    Recall that in each iteration, we leave out all the pairs that are rejected in the previous iterations and update the quantile, $\hat{c}_\ell(\alpha)$. Therefore, we have 
    \begin{equation}\label{equ:stepdown_order}
        \hat{c}_0(\alpha) \geq \hat{c}_\ell(\alpha) \geq \hat{c}_\Gamma(\alpha).
    \end{equation}

    Moreover, let \( \tilde{T}_\ell = \sup_{(i,j) \in \mathcal{S}_\ell^C} \sup_{\xb \in \Omega} \sqrt{h^d \Xi} \Big(\widehat{\theta}_i(\xb)- \widehat{\theta}_j(\xb)- \theta^*_i(\xb) + \theta^*_j(\xb) \Big) \). By Corollary~\ref{col:GMB_validity}, the $\alpha$-quantile of the corresponding statistic $W$ derived via the Gaussian multiplier bootstrap accurately estimates the $\alpha$-quantile of $\tilde{T}_\ell$  that 
    \begin{equation}\label{equ:stepdown_convergence}
        \sup_{\alpha \in (0,1)}|\mathbb{P}(\tilde{T}_\ell  \leq \hat{c}_\ell(\alpha)) - \alpha| \xrightarrow[]{} 0 \text{ as } n \xrightarrow[]{} \infty.
    \end{equation}
    
    Following \eqref{equ:stepdown_order} and \eqref{equ:stepdown_convergence}, \citet{romano2005exact} shows the strong control of the family-wise error rate such that
    \begin{equation}
        \mathbb{P}  \Big( \sup_{\ell \in [\Gamma]} \sup_{(i,j) \in \mathcal{S}_\ell^C} (\inf_{\xb \in \Omega} \sqrt{h^d \Xi} \Big(\widehat{\theta}_i(\xb)- \widehat{\theta}_j(\xb)- \theta^*_i(\xb) + \theta^*_j(\xb) \Big) - \hat{c}_\ell(1-\alpha) ) \leq 0  \Big) \geq 1-\alpha+o(1).
    \end{equation}
    Therefore, we have
    \begin{equation*}
        \mathbb{P} \Big( R(\boldsymbol{\theta}^*) \in \widehat{\mathbf{H}} \Big) \geq 1-\alpha + o(1),
    \end{equation*}
    which completes the proof.
\end{proof}
  
\section{Proof of Lemmas in Section~\ref{pf:thm:estimate}}\label{pf:lem_estimate}

\begin{proof}[Proof of Lemma~\ref{lem:infty_whole}]
    Consider any   $\xb \in \Omega$ and $m \in [n]$, by the triangle inequality, we have
    \begin{align*}
        \big|\theta_m^{\gamma + 1}(\xb) - \theta ^ *_m (\xb) \big| &\leq \big|\theta_m^{\gamma + 1}(\xb) - \theta_m^{\gamma + 1, (m)}(\xb) \big| + \big|\theta_m^{\gamma + 1, (m)}(\xb) - \theta ^ *_m (\xb) \big| \\
        & \leq \sup_{\xb \in \Omega} \big\| \theta_m^{\gamma + 1}(\xb) - \theta_m^{\gamma + 1, (m)}(\xb) \big\|_2 + \sup_{\xb \in \Omega}  \big|\theta_m^{\gamma + 1, (m)}(\xb) - \theta ^ *_m (\xb) \big|.
    \end{align*}
    Together with Lemmas  \ref{lem:abs_loo} and  \ref{lem:two_loo}, we finish the proof.
\end{proof}

\begin{proof}[Proof of Lemma~\ref{lem:distance_hat}]
    First, we consider the convergence property of the gradient descent algorithm under the RSC/RSS condition for $\mathcal{L}_\lambda$ as specified in Lemma \ref{lem:RSC_RSS}. Applying Theorem 3.10 of \citet{bubeck2015convex} yields that for any   $\xb \in \Omega$,
    \begin{equation*}
        \big \|\btheta ^ \Gamma(\xb) - \widehat{\btheta}(\xb) \big\|_{2} \leq \zeta ^ \Gamma \big\|\btheta ^ 0(\xb) - \widehat{\btheta}(\xb) \big\|_2,
    \end{equation*}
    where $\zeta = 1 - \frac{\lambda}{\lambda + \frac{C}{n}}$. We then apply Lemma \ref{lem:distance_star} and consider the case that $\btheta^0(\xb) = \btheta^*(\xb)$. Specifically, we have 
    \begin{equation*} 
    \begin{split}
        \sup_{\xb \in \Omega} \big\|\btheta ^ \Gamma(\xb) - \widehat{\btheta}(\xb) \big \|_2 &\leq \sup_{\xb \in \Omega} \zeta ^ \Gamma \big \|\btheta ^ 0(\xb) - \widehat{\btheta}(\xb) \big \|_2 = \sup_{\xb \in \Omega} \zeta ^ \Gamma \big \|\btheta^* (\xb) - \widehat{\btheta}(\xb) \big\|_2 \\ &
        \lesssim \sqrt{n} \bigg(1 - \frac{\lambda}{\lambda + \frac{C}{n}} \bigg) ^ \Gamma \leq \sqrt{n} \bigg(1 - \widetilde{C} \bigg(h^2 + \sqrt{\frac{\log (nh ^ {d/2-1})}{npLh ^ d}} \bigg) \bigg) ^ \Gamma.
    \end{split}
    \end{equation*}
    Here $\widetilde{C}$ is some positive constant, and the last inequality holds as the term $\frac{C}{n}$ dominates the denominator of $\frac{\lambda}{\lambda + \frac{C}{n}}$, given that $\lambda \asymp \frac{1}{n} \Big( h^2 + \sqrt{\frac{\log (nh ^{d/2-1})}{npLh^d}} \Big)$. In this case, a sufficiently large number of iterations $\Gamma$ such that
    \begin{equation} \label{equ:gamma_condition}
        \Gamma \geq \frac{\log \Big(h^2 + \sqrt{\frac{\log(nh^{d/2-1})}{npLh^d}} \Big) - \log \sqrt{n}}{\log \Big(1 - \widetilde{C} \Big(h^2 + \sqrt{\frac{\log(nh^{d/2-1})}{npLh^d}} \Big)  \Big)} 
    \end{equation}
    guarantees that
    \begin{equation*}
        \sup_{\xb \in \Omega} \big\|\btheta ^ \Gamma(\xb) - \widehat{\btheta}(\xb) \big \|_2 \lesssim  \sqrt{n} \bigg(1 - \widetilde{C} \bigg(h^2 + \sqrt{\frac{\log (nh ^ {d/2-1})}{npLh ^ d}} \bigg) \bigg) ^ \Gamma \leq h^2 + \sqrt{\frac{\log (nh ^{d/2-1})}{npLh^d}}.
    \end{equation*}
    
    Then, we check  condition \eqref{equ:gamma_condition}. We consider the numerator and the denominator separately. We first consider the expansion of the denominator {\small
    \begin{equation} \label{equ:log_expansion}
        \log \bigg(1 - \widetilde{C} \Big(h^2 + \sqrt{\frac{\log(nh^{d/2-1})}{npLh^d}} \Big) \bigg) = - \widetilde{C} \bigg(h^2 + \sqrt{\frac{\log(nh^{d/2-1})}{npLh^d}} \bigg)  + O\bigg( \Big(h^2 + \sqrt{\frac{\log(nh^{d/2-1})}{npLh^d}} \Big)  ^ 2\bigg),
    \end{equation} }
    and then compare \eqref{equ:log_expansion} with the numerator $\log \bigg( \frac{1}{\sqrt{n}}\Big(h^2 + \sqrt{\frac{\log(nh^{d/2-1})}{npLh^d}} \Big) \bigg)$. Here, we have
    \begin{equation*}
    \begin{aligned}
      & \frac{1}{\sqrt{n}} \bigg(h^2 + \sqrt{\frac{\log(nh^{d/2-1})}{npLh^d}} \bigg) ^ 2   \frac{\log \Big(h^2 + \sqrt{\frac{\log(nh^{d/2-1})}{npLh^d}} \Big) - \log \sqrt{n}}{\log \Big(1 - \widetilde{C} \Big(h^2 + \sqrt{\frac{\log(nh^{d/2-1})}{npLh^d}}\Big)  \Big)} \xrightarrow[]{} 0 \\
       & \text{ since } \frac{1}{\sqrt{n}}\bigg(h^2 + \sqrt{\frac{\log(nh^{d/2-1})}{npLh^d}} \bigg) \xrightarrow[]{} 0,
    \end{aligned}
    \end{equation*}
    which guarantees the existence of some $ \Gamma = \Omega \Bigg(\frac{\sqrt{n}}{h^4 + \frac{\log(nh ^ {d/2-1})}{npLh ^ d}} \Bigg)$ that satisfies the condition \eqref{equ:gamma_condition}. 
    
    Finally, by the inequality of  the $\ell_\infty$-norm and the $\ell_2$-norm, we have
    \begin{equation*}
         \sup_{\xb\in \Omega} \big\|\btheta ^ \Gamma(\xb) - \widehat{\btheta}(\xb) \big\|_\infty \leq \sup_{\xb \in \Omega} \big\|\btheta ^ \Gamma(\xb) - \widehat{\btheta}(\xb) \big\|_2 \lesssim h^2 + \sqrt{\frac{\log(nh ^ {d/2-1})}{npLh^d}},
    \end{equation*}
    which completes the proof.
\end{proof}

\subsection{Auxillary Lemmas for Leave-One-Out Technique} \label{sec:four_ineqs}

We provide the proofs for the inductive behavior of inequalities \eqref{equ:two_whole}-\eqref{equ:two_loo} for some index \(\gamma \in [\Gamma]\), assuming that the inequalities \eqref{equ:infty_whole}-\eqref{equ:two_loo} hold for \(\gamma-1 \in [\Gamma]\). We show the inductive behavior of \eqref{equ:infty_whole} in Lemma~\ref{lem:infty_whole}. 

Before presenting the proofs, we first review the brief idea of the leave-one-out technique. This method involves updating the estimate $\boldsymbol\theta^{\gamma,(m)}$ using the gradient of the leave-one-out objective function $\mathcal{L}_{\lambda}^{(m)}(\btheta; \xb)$ for $m \in [n]$, where the $m$-th entry in the objective function is replaced by its expected value. Specifically, the leave-one-out gradient $\nabla \mathcal{L}_{\lambda}^{(m)} (\btheta; \xb)$ is
{\small 
\begin{align*}
   & \nabla \mathcal{L}_{\lambda}^{(m)} (\btheta; \xb)  \\
   &=  \frac{1}{n^2pL} \sum_{(i,j) \in \mathcal{E}; i < j; i,j \neq m} \sum_{\ell \in [L_{ij}]} \mathcal{K}_h \Big( \bX_{ij}^\ell - \xb \Big) \bigg\{ \Big(-y_{ij}(\bX_{ij}^\ell) + \frac{\exp (\theta_j(\bX_{ij}^\ell))}{\exp{(\theta_i(\bX_{ij}^\ell))} + \exp(\theta_j(\bX_{ij}^\ell))} \Big) \big(\eb_j - \eb_i\big) \bigg\} \\
    &\quad  + \frac{1}{n^2 pL} \sum_{i \neq m} \sum_{\ell \in [L_{im}]} p \mathcal{K}_h \Big( \bX_{im}^\ell - \xb \Big) \bigg\{ \Big(- \frac{\exp (\theta_m^*(\bX_{im}^\ell))}{\exp{(\theta_m^*(\bX_{im}^\ell))} + \exp(\theta_i^*(\bX_{im}^\ell))} \\
    & \quad + \frac{\exp (\theta_m(\bX_{im}^\ell))}{\exp{(\theta_m(\bX_{im}^\ell))} + \exp(\theta_i(\bX_{im}^\ell))} \Big) \big(\eb_m - \eb_i\big) \bigg\} + \lambda \btheta.
\end{align*} }
In particular, the $m$-th entry of the gradient, $\big[\nabla \mathcal{L}_{\lambda}^{(m)} (\btheta; \xb) \big]_m$ is 
\begin{align*}
     \big[\nabla \mathcal{L}_{\lambda}^{(m)} (\btheta; \xb) \big]_m = & \frac{1}{n^2pL} \sum_{i \neq m} \sum_{\ell \in [L_{im}]} p \mathcal{K}_h \Big( \bX_{im}^\ell - \xb \Big)  \Big(- \frac{\exp (\theta_m^*(\bX_{im}^\ell))}{\exp{(\theta_m^*(\bX_{im}^\ell))} + \exp(\theta_i^*(\bX_{im}^\ell))} \\
     & \quad + \frac{\exp (\theta_m(\bX_{im}^\ell))}{\exp{(\theta_m(\bX_{im}^\ell))} + \exp(\theta_i(\bX_{im}^\ell))} \Big) + \lambda \theta_m.
\end{align*}

\begin{lemma}\label{lem:two_whole}
    Given that the inequalities \eqref{equ:infty_whole}-\eqref{equ:two_loo} hold for some positive $\gamma$, there exists some positive constant $C$ such that
    \begin{equation} \label{equ:two_iter}
    \sup_{\xb \in \Omega} \big\| \btheta^{\gamma + 1}(\xb) - \btheta ^ * (\xb) \big \|_2 \leq C \sqrt{n} \bigg(h ^ 2 + \sqrt{\frac{\log (nh^{d/2-1})}{npLh^d}} \bigg)
\end{equation}
with a probability of $1 - O(n ^ {-10})$.
\end{lemma}

\begin{proof}
    We start the proof for any   $\xb \in \Omega = [0,1]^d$ by modifying the gradient update rule such that
    \begin{align*}
      &  \btheta^{\gamma + 1} (\xb) - \btheta^*(\xb) = \btheta^\gamma (\xb) - \eta \nabla \mathcal{L}_{\lambda}(\btheta^\gamma; \xb) - \btheta^*(\xb) \\
        &\quad= \big\{ \btheta^\gamma (\xb) - \eta \nabla \mathcal{L}_{\lambda}(\btheta^\gamma; \xb) - [\btheta^*(\xb) - \eta \nabla \mathcal{L}_{\lambda}(\btheta^*; \xb)] \big\} - \eta \nabla \mathcal{L}_{\lambda}(\btheta^*; \xb) \\
        &\quad = \Big\{ \Ib_n - \eta \int_0^1 \Hb_\lambda(\widetilde{\btheta}_{\tau}; \xb) d \tau \Big\} \big( \btheta^\gamma (\xb) - \btheta^*(\xb) \big) - \eta \nabla \mathcal{L}_{\lambda}(\btheta^*; \xb),
    \end{align*}
    where $\widetilde{\btheta}_{\tau}= \btheta^* + \tau (\btheta^\gamma - \btheta^*)$. Here, the last equality holds as a consequence of the fundamental theorem of calculus \citep{lang1993newtonleibniz}. Denote $\Ab$ as the integration term by $\Ab = \int_0^1 \Hb_\lambda(\widetilde{\btheta}_\tau; \xb) \dd \tau$. By the triangle inequality, we have that 
    \begin{equation}\label{equ:two_whole_M1M2}
        \big\| \btheta^{\gamma + 1} (\xb) - \btheta^* (\xb) \big \|_2 \leq \underbrace{ \big\|(\Ib_n - \eta \Ab) (\btheta^\gamma (\xb) - \btheta^*(\xb)) \big \|_2}_{\mathcal{M}_1} + \underbrace{\eta \big\|\nabla \mathcal{L}_{\lambda} (\btheta^*; \xb) \big\|_2}_{\mathcal{M}_2}.
    \end{equation}
    
    We bound $\mathcal{M}_1$ and $\mathcal{M}_2$ separately. First, we bound  $\mathcal{M}_1$.  We apply Lemma \ref{lem:hessian_bound} and check the condition that $\widetilde{\btheta}_{\tau} (\xb)$ belongs to a bound set centered at $\btheta^* (\xb)$. Let $\theta_{\max} = \max_i{\theta_i}$ and $\theta_{\min} = \min_i{\theta_i}$. For $\btheta = \{\theta_1, \theta_2, \ldots, \theta_n \}$, we have
    \begin{equation*}
        \widetilde{\theta}_{\tau, \max}(\xb) - \widetilde{\theta}_{\tau, \min}(\xb) \leq \theta_{\max}^*(\xb) - \theta_{\min}^*(\xb) + 2 \big\|\btheta^\gamma(\xb) - \btheta^*(\xb) \big\|_\infty \leq \log \kappa + 2 \big\|\btheta^\gamma(\xb) - \btheta^*(\xb) \big\|_\infty,
    \end{equation*}
    where $\big\|\btheta^\gamma(\xb) - \btheta^*(\xb) \big\|_\infty = O(1)$ by\eqref{equ:infty_whole} given Assumption \ref{ass:erdos_p} holds. In this case, we apply Lemma \ref{lem:hessian_bound} such that for any $\tau \in [0,1]$, there exist positive constants $c^*$ and $C^*$ such that
    \begin{equation*}
        \lambda + \frac{c^*}{n} \leq \lambda_{\min, \perp} (\Hb_\lambda(\widetilde{\btheta}_\tau; \xb)) \leq \lambda_{\max} (\Hb_\lambda(\widetilde{\btheta}_\tau; \xb)) \leq \lambda + \frac{C^*}{n},
    \end{equation*}
    with a probability of $1 - O(n^{-10})$. 
    For any matrix \(A\), we let
    \[
    \lambda_{\min, \perp}(A) = \min \left\{ \mu \mid \mathbf{z}^\top A \mathbf{z} \geq \mu \|\mathbf{z}\|_2^2 \text{ for all } \mathbf{z} \text{ with } \mathbf{1}^\top \mathbf{z} = 0 \right\},
    \]
    namely, the smallest eigenvalue when restricted to vectors orthogonal to \(\mathbf{1}\).

    Given that $\mathbf{1}^\top \btheta^\gamma (\xb) = \mathbf{1}^\top \btheta^* (\xb) = 0$ according to Fact 1 in Section 6 of \citet{chen2019spectral}, with the difference being that we assume $\mathbf{1}^\top \btheta^*(\xb) = 0$ under the dynamic setup, we have that
    \begin{equation*}
        \big\|(\Ib_n - \eta \Ab) (\btheta^\gamma (\xb) - \btheta^*(\xb)) \big\|_2 \leq \max \big\{|1 - \eta \lambda_{\min, \perp}(\Ab)|, |1 - \eta \lambda_{\max}(\Ab)| \big\} \big\|\btheta^\gamma(\xb) - \btheta^*(\xb) \big\|_2.
    \end{equation*}
    By $\eta \leq \frac{1}{\lambda + \frac{C^*}{n}}$ in the setup, we have $1 - \eta \lambda_{\min, \perp}(\Ab) \geq 0$ and $1 - \eta \lambda_{\max}(\Ab) \geq 0$. Consequently,
    \begin{equation*}
        \mathcal{M}_1 = \big\|(\Ib_n - \eta \Ab) (\btheta^\gamma (\xb) - \btheta^*(\xb)) \big\|_2 \leq (1 - \eta \lambda_{\min, \perp}(\Ab)) \big\|\btheta^\gamma (\xb) - \btheta^*(\xb) \big\|_2.
    \end{equation*}
    Applying \eqref{equ:two_whole}, we obtain
    \begin{equation*}
        \mathcal{M}_1 \leq C \sqrt{n} ( h^2 + \sqrt{\frac{\log (nh^{d/2-1})}{npLh^d}} ) \text{ with probability } 1 - O(n^{-10}).
    \end{equation*}
    
    Then, we control $\mathcal{M}_2$. By the norm equivalence inequality $\|\cdot\|_\infty \leq \| \cdot \|_2 \leq \sqrt{n} \| \cdot \|_\infty$ and Lemma \ref{lem:gradient_bound}, we have,     with probability at least $1 - O(n^{-10})$,
    \begin{equation*}
        \mathcal{M}_2 \leq \eta \frac{C}{\sqrt{n}} ( h^2 + \sqrt{\frac{\log (nh^{d/2-1})}{npLh^d}} ) \asymp \sqrt{n} ( h^2 + \sqrt{\frac{\log (nh^{d/2-1})}{npLh^d}} ).
    \end{equation*}
    
    Plugging the results for $\mathcal{M}_1$ and $\mathcal{M}_2$ in \eqref{equ:two_whole_M1M2}, we complete the proof.
\end{proof}

\begin{lemma}\label{lem:abs_loo}
    Given that  inequalities \eqref{equ:infty_whole}-\eqref{equ:two_loo} hold for some positive real number $\gamma$, then there exists some positive constant $C$, such that
    \begin{equation*}%\label{equ:max_iter}
    \sup_{\xb \in \Omega} \max_{1 \leq m \leq n} \big| \theta^{\gamma + 1, (m)}_m(\xb) - \theta ^ *_m (\xb) \big| \leq C \bigg(h ^ 2 + \sqrt{\frac{\log (nh ^ {d/2-1})}{npLh^d}} \bigg),
    \end{equation*}
    with  probability  $1 - O(n^{-10})$.
\end{lemma}

\begin{proof}
    This proof is mainly inspired by \citet{chen2019spectral}. 
    First, we apply the mean value theorem to the $m$-th component of \(\btheta^{\gamma + 1, (m)}(\xb) - \btheta^*(\xb)\):
    {\small\begin{align*}
        & \quad \theta_m^{\gamma + 1, (m)}(\xb) - \theta_m^*(\xb) = \theta_m^{\gamma, (m)}(\xb) - \eta \Big[ \nabla \mathcal{L}_{\lambda}^{(m)}(\btheta^{\gamma, (m)}; \xb) \Big]_m - \theta_m^*(\xb) \\
        &= \theta_m^{\gamma, (m)}(\xb) - \theta_m^*(\xb) - \eta \lambda \theta_m^{\gamma, (m)}(\xb) - \frac{\eta}{n^2pL} \sum_{i \neq m} \sum_{\ell \in [L_{im}]} p \mathcal{K}_h \big( \bX_{im}^\ell - \xb \big) \Bigg[ -\frac{\exp\big(\theta_m^*(\bX_{im}^\ell)\big)}{\exp\big(\theta_m^*(\bX_{im}^\ell)\big) + \exp\big(\theta_i^*(\bX_{im}^\ell)\big)} \\
        &\quad + \frac{\exp\big(\theta^{\gamma, (m)}_m(\bX_{im}^\ell)\big)}{\exp\big(\theta^{\gamma, (m)}_m(\bX_{im}^\ell)\big) + \exp\big(\theta^{\gamma, (m)}_i(\bX_{im}^\ell)\big)} \Bigg] \\
        &= \underbrace{\big(1 - \eta \lambda\big) \big( \theta_m^{\gamma, (m)}(\xb) - \theta_m^*(\xb) \big)}_{\mathcal{I}_1} - \underbrace{\eta \lambda \theta_m^*(\xb)}_{\mathcal{I}_2} \\
        & - \underbrace{\frac{\eta}{n^2pL} \sum_{i \neq m} \sum_{\ell \in [L_{im}]} p \mathcal{K}_h \big( \bX_{im}^\ell - \xb \big) \Bigg[ \frac{\exp(c_{im})}{\big(1 + \exp(c_{im})\big)^2}  \times \Big[\big(\theta^*_i(\bX_{im}^\ell) - \theta^*_m(\bX_{im}^\ell)\big) - \big(\theta^{\gamma, (m)}_i(\bX_{im}^\ell) - \theta^{\gamma, (m)}_m(\bX_{im}^\ell)\big)\Big] \Bigg]}_{\mathcal{I}_3}, 
    \end{align*}}
    where \( c_{im} \) lies within the interval between \(\theta^*_i(\bX_{im}^\ell) - \theta^*_m(\bX_{im}^\ell)\) and \(\theta^{\gamma, (m)}_i(\bX_{im}^\ell) - \theta^{\gamma, (m)}_m(\bX_{im}^\ell)\). Here, the boundedness of \( c_{im} \) is guaranteed by \eqref{equ:infty_loo}. 
    
    Then, we bound $\mathcal{I}_1$, $\mathcal{I}_2$, and $\mathcal{I}_3$ respectively. For $\mathcal{I}_1$ and $\mathcal{I}_2$, we consider the setup in Algorithm~\ref{al:gradient_descent}: $\eta \leq \frac{1}{\lambda + \frac{C^*}{n}}$ for some constant $C^*$ and $\lambda \asymp \frac{1}{n} \Big( h^2 + \sqrt{\frac{\log(nh^{d/2-1})}{npLh^d}} \Big)$. The setup guarantees that
    \begin{equation*}
        \eta \lambda = O \Big(h^2 + \sqrt{\frac{\log(nh^{d/2-1})}{npLh^d}} \Big) = o(1),
    \end{equation*}
    and it further implies that
    \begin{equation}\label{equ:abs_I1}
        \big|\mathcal{I}_1 \big| \leq \big| \theta_m^{\gamma, (m)}(\xb) - \theta_m^*(\xb) \big| \lesssim h^2 + \sqrt{\frac{\log(nh^{d/2-1})}{npLh^d}} \quad \text{with probability } 1 - O(n^{-10}), 
    \end{equation}
    and
    \begin{equation}\label{equ:abs_I2}
        \big|\mathcal{I}_2 \big| \leq \eta \lambda \sup_{\xb \in \Omega} \big\| \btheta^*(\xb) \big\|_{\infty} \leq \eta \lambda \log \kappa \lesssim h^2 + \sqrt{\frac{\log(nh^{d/2-1})}{npLh^d}} \quad \text{with probability } 1 - O(n^{-10}),
    \end{equation}
    where  $\kappa = O(1)$ according to Assumption \ref{ass:constraint_lps}.
    
    For $\mathcal{I}_3$, we consider
  {\small  \begin{equation} \label{equ:I3_numerator}
        \begin{split}
            & \Bigg| \sum_{i \neq m} \sum_{\ell \in [L_{im}]} p \mathcal{K}_h \big( \bX_{im}^\ell - \xb \big) \Bigg[ \frac{\exp(c_{im})}{\big(1 + \exp(c_{im})\big)^2} \Big[\big(\theta^*_i(\bX_{im}^\ell) - \theta^*_m(\bX_{im}^\ell)\big) - \big(\theta^{\gamma, (m)}_i(\bX_{im}^\ell) - \theta^{\gamma, (m)}_m(\bX_{im}^\ell)\big)\Big] \Bigg] \Bigg| \\
            & \overset{(i)}{\leq} \sum_{i \neq m} \sum_{\ell \in [L_{im}]} p \mathcal{K}_h \big( \bX_{im}^\ell - \xb \big) \Bigg[ \frac{2 \exp(c_{im})}{\big(1 + \exp(c_{im})\big)^2} \sup_{\xb \in \Omega} \max_{1 \leq m \leq n} \big|\theta_m^{\gamma, (m)}(\xb) - \theta_m^*(\xb) \big| \Bigg] \\
            & \overset{(ii)}{\lesssim} p \Bigg( \sum_{i \neq m} \sum_{\ell \in [L_{im}]} L_{im} \Bigg) \sup_{\xb \in \Omega} \max_{1 \leq m \leq n} \big|\theta_m^{\gamma, (m)}(\xb) - \theta_m^*(\xb) \big| \\
            & \overset{(iii)}{\lesssim} npL \Big( h^2 + \sqrt{\frac{\log(nh^{d/2-1})}{npLh^d}} \Big),
        \end{split}
    \end{equation}}
    with probability \(1 - O(n^{-10})\). Here, $(i)$ follows from the triangle inequality, $(ii)$ follows from applying Lemma \ref{lem:meanvar_erdos} on $\sum_{i \neq m} \sum_{\ell \in [L_{im}]} L_{im}$ samples of $\mathcal{K}_h\big( \bX_{im}^\ell - \xb \big)$, and $(iii)$ follows from \eqref{equ:abs_loo} and Assumption \ref{ass:similar_comparison}. By \eqref{equ:I3_numerator}, we  have that
    \begin{equation}\label{equ:abs_I3}
        \big| \mathcal{I}_3 \big| \lesssim \eta \frac{npL \Big( h^2 + \sqrt{\frac{\log(nh^{d/2-1})}{npLh^d}} \Big)}{n^2 pL} = O \Big( h^2 + \sqrt{\frac{\log(nh^{d/2-1})}{npLh^d}} \Big) \quad \text{with probability } 1 - O(n^{-10}).
    \end{equation} 

    Finally, combining \eqref{equ:abs_I1},\eqref{equ:abs_I2}, and \eqref{equ:abs_I3}  completes the proof. The uniform boundedness is guaranteed by the arbitrariness of $\xb \in \Omega$ and $m \in [n]$ in the results for $\mathcal{I}_1$ through $\mathcal{I}_3$.
\end{proof}

\begin{lemma}\label{lem:two_loo}
    Given that the inequalities \eqref{equ:infty_whole}-\eqref{equ:two_loo} hold for some positive $\gamma$, there exists some constant $C$, such that
    \begin{equation*}
    \sup_{\xb \in \Omega} \max_{1 \leq m \leq n} \big\| \btheta^{\gamma + 1, (m)}(\xb) - \btheta ^ {\gamma + 1} (\xb) \big\|_2 \leq C \bigg(h ^ 2 + \sqrt{\frac{\log (nh ^ {d/2-1})}{npLh ^ d}} \bigg),
\end{equation*}
with a probability of $1 - O(n^{-10})$.
\end{lemma}

\begin{proof}
    We first consider the following decomposition:
    \begin{align*}
        \btheta^{\gamma + 1, (m)}(\xb) - \btheta^{\gamma + 1} (\xb) &= \btheta^{\gamma}(\xb) - \eta \nabla \mathcal{L}_\lambda(\btheta^\gamma; \xb) - \Big[\btheta^{\gamma, (m)}(\xb) - \eta \nabla \mathcal{L}_\lambda^{(m)}(\btheta^{\gamma, (m)}; \xb) \Big] \\
        &= \btheta^{\gamma}(\xb) - \eta \nabla \mathcal{L}_\lambda(\btheta^\gamma; \xb) - \Big[\btheta^{\gamma, (m)}(\xb) - \eta \nabla \mathcal{L}_\lambda(\btheta^{\gamma, (m)}; \xb) \Big] \\
        &\quad - \eta \Big[ \nabla \mathcal{L}_\lambda(\btheta^{\gamma, (m)}; \xb) - \nabla \mathcal{L}_\lambda^{(m)}(\btheta^{\gamma, (m)}; \xb) \Big] \\
        &= \underbrace{\Big(\Ib_n - \eta \int_0^1 \Hb_\lambda(\widetilde{\btheta}_{\tau}; \xb) d \tau \Big) \big(\btheta^\gamma(\xb) - \btheta^{\gamma, (m)}(\xb)\big)}_{\mathcal{U}_1} \\
        &\quad + \underbrace{\Big\{ - \eta \Big[ \nabla \mathcal{L}_\lambda(\btheta^{\gamma, (m)}; \xb) - \nabla \mathcal{L}_\lambda^{(m)}(\btheta^{\gamma, (m)}; \xb) \Big] \Big \}}_{\mathcal{U}_2} = \mathcal{U}_1 + \mathcal{U}_2.
    \end{align*}
    
    Then, we bound $\mathcal{U}_1$ and $\mathcal{U}_2$. For $\mathcal{U}_1$, since \eqref{equ:two_loo} holds for some positive $\gamma$, we have
    \begin{equation}\label{equ:two_loo_U1}
        \big\|\mathcal{U}_1 \big\|_2 = \Big \|(\Ib_n - \eta \Ab)\big(\btheta^\gamma(\xb) - \btheta^{\gamma, (m)}(\xb)\big) \Big \|_2 \leq \Big \|\btheta^\gamma(\xb) - \btheta^{\gamma, (m)}(\xb) \Big \|_2 \lesssim h^2 + \sqrt{\frac{\log(nh^{d/2-1})}{npLh^d}},
    \end{equation}
    with probability at least $1 - O(n^{-10})$. This approach is similar of bounding  $\mathcal{M}_1$ in Lemma \ref{lem:two_whole}, with the difference being that we apply \eqref{equ:two_loo} instead of \eqref{equ:two_whole}, and we omit the details here.
    
    For $\mathcal{U}_2$, we calculate the closed form of the gradient $\nabla \mathcal{L}_\lambda(\btheta^{\gamma, (m)}; \xb)$ and the leave-one-out gradient $\nabla \mathcal{L}_\lambda^{(m)}(\btheta^{\gamma, (m)}; \xb)$. We have the gradient $\nabla \mathcal{L}_\lambda(\btheta^{\gamma, (m)}; \xb)$:
    {\small
    \begin{align*}
        & \nabla \mathcal{L}_\lambda(\btheta^{\gamma, (m)}; \xb) \\
        &= \frac{1}{n^2pL} \sum_{(i,j) \in \mathcal{E}} \sum_{\ell \in [L_{ij}]} \mathcal{K}_h \Big(\bX_{ij}^\ell - \xb \Big) \Big\{ \Big(-y_{ij}(\bX_{ij}^\ell) + \frac{\exp(\theta^{\gamma, (m)}_j(\bX_{ij}^\ell))}{\exp(\theta^{\gamma, (m)}_i(\bX_{ij}^\ell)) + \exp(\theta^{\gamma, (m)}_j(\bX_{ij}^\ell))} \Big) (\eb_j - \eb_i) \Big\} \\
        &= \frac{1}{n^2pL} \underbrace{\sum_{(i,j) \in \mathcal{E}; i<j; i,j \neq m} \sum_{\ell \in [L_{ij}]} \mathcal{K}_h \Big(\bX_{ij}^\ell - \xb \Big) \Big\{ \Big(-y_{ij}(\bX_{ij}^\ell) + \frac{\exp(\theta^{\gamma, (m)}_j(\bX_{ij}^\ell))}{\exp(\theta^{\gamma, (m)}_i(\bX_{ij}^\ell)) + \exp(\theta^{\gamma, (m)}_j(\bX_{ij}^\ell))} \Big) (\eb_j - \eb_i) \Big\}}_{\mathcal{N}_1} \\
        &\quad + \frac{1}{n^2pL} \underbrace{\sum_{i \neq m, (i,m) \in \mathcal{E}} \sum_{\ell \in [L_{im}]} \mathcal{K}_h \Big(\bX_{im}^\ell - \xb \Big) \Big\{ \Big(-y_{im}(\bX_{im}^\ell) + \frac{\exp(\theta^{\gamma, (m)}_m(\bX_{im}^\ell))}{\exp(\theta^{\gamma, (m)}_i(\bX_{im}^\ell)) + \exp(\theta^{\gamma, (m)}_m(\bX_{im}^\ell))} \Big) (\eb_m - \eb_i) \Big\}}_{\mathcal{N}_2} \\
        &= \frac{1}{n^2 pL} \Big( \mathcal{N}_1 + \mathcal{N}_2 \Big).
    \end{align*}}
    Meanwhile, the leave-one-out gradient $\nabla \mathcal{L}_\lambda^{(m)} (\btheta^{\gamma, (m)}; \xb)$ is 
{\small    \begin{align*}
        & \quad \nabla \mathcal{L}_\lambda^{(m)} (\btheta^{\gamma, (m)}; \xb) \\
        &= \frac{1}{n^2pL} \underbrace{\sum_{(i,j) \in \mathcal{E}; i < j; i,j \neq m} \sum_{\ell \in [L_{ij}]} \mathcal{K}_h \Big(\bX_{ij}^\ell - \xb \Big) \Big\{ \Big(-y_{ij}(\bX_{ij}^\ell) + \frac{\exp(\theta^{\gamma, (m)}_j(\bX_{ij}^\ell))}{\exp(\theta^{\gamma, (m)}_i(\bX_{ij}^\ell)) + \exp(\theta^{\gamma, (m)}_j(\bX_{ij}^\ell))} \Big) (\eb_j - \eb_i) \Big\}}_{\widetilde{\mathcal{N}}_1} \\
        &\quad + \frac{1}{n^2pL} 
        \underbrace{
        \begin{aligned}
            & \sum_{i \neq m} \sum_{\ell \in [L_{im}]} p \mathcal{K}_h \Big(\bX_{im}^\ell - \xb \Big) \Big\{ \Big(-\frac{\exp(\theta_m^*(\bX_{im}^\ell))}{\exp(\theta_m^*(\bX_{im}^\ell)) + \exp(\theta_i^*(\bX_{im}^\ell))} \\
            &\quad + \frac{\exp(\theta^{\gamma, (m)}_m(\bX_{im}^\ell))}{\exp(\theta_m^{\gamma, (m)}(\bX_{im}^\ell)) + \exp(\theta_i^{\gamma, (m)}(\bX_{im}^\ell))} \Big) (\eb_m - \eb_i) \Big\} 
        \end{aligned}
        }_{\widetilde{\mathcal{N}}_2}\\
        & = \frac{1}{n^2 pL} \Big( \widetilde{\mathcal{N}}_1 + \widetilde{\mathcal{N}}_2 \Big).
    \end{align*}}

    Combining the above two equalities together, we have     
    \begin{equation*}
        \big\|\mathcal{U}_2 \big\|_2 = \frac{\eta}{n^2 pL} \big\|\mathcal{N}_2 - \widetilde{\mathcal{N}}_2\big\|_2.
    \end{equation*}
    For \(\big\|\mathcal{N}_2 - \widetilde{\mathcal{N}}_2\big\|_2\), by Lemma \ref{lem:term_difference}, we have that
        \begin{equation*}
        \big\|\mathcal{N}_2 - \widetilde{\mathcal{N}}_2 \big\|_2 \lesssim \sqrt{np \log n} L \Big( h^2 + \sqrt{\frac{\log(nh^{d/2-1})}{Lh^d}} \Big),
    \end{equation*}
    with probability at least $1 - O(n^{-10})$. Given the setup of Algorithm \ref{al:gradient_descent} such that \(\eta \leq \min \Big\{ \frac{1}{\lambda + \frac{C^*}{n}}, 1 \Big\}\), we conclude that 
    \begin{equation}\label{equ:two_loo_U2}
        \|\mathcal{U}_2 \|_2 \lesssim h^2 + \sqrt{\frac{\log(nh^{d/2-1})}{npLh^d}}. 
    \end{equation}

    Combining \eqref{equ:two_loo_U1} and \eqref{equ:two_loo_U2} yields Lemma \ref{lem:two_loo}. The uniformity is guaranteed since \(\xb\) is arbitrarily selected in the region \(\Omega = [0,1]^d\).
    \end{proof}

Finally, we bound the distance between $\btheta^{\gamma, (m)} (\cdot)$ and $\btheta^*(\cdot)$ that we need in the proof for Lemma~\ref{lem:abs_loo}. Specifically, we show that for some constant $C$, 
\begin{equation} \label{equ:infty_loo}
    \sup_{\xb \in \Omega} \max_{1 \leq m \leq n}  \big\| \btheta^{\gamma, (m)}(\xb) - \btheta ^ * (\xb) \big\|_\infty \leq C \bigg(h ^ 2 + \sqrt{\frac{\log (nh ^ {d/2-1})}{npLh^d}} \bigg).
\end{equation}

\begin{lemma}[Bound on the Distance between $\btheta^{\gamma, (m)} (\cdot)$ and $\btheta^*(\cdot)$] \label{lem:infty_loo}
Given that the inequalities \eqref{equ:infty_whole}-\eqref{equ:two_loo} hold for some positive integer $\gamma$, then there exists some constant $C$ such that
\begin{equation*}
        \sup_{\xb \in \Omega} \max_{1 \leq m \leq n} \big \| \btheta^{\gamma, (m)}(\xb) - \btheta ^ * (\xb) \big \|_\infty \leq C_1 \bigg(h ^ 2 + \sqrt{\frac{\log (nh ^ {d/2-1})}{npLh^d}} \bigg)
\end{equation*}
    with probability at least $1 - O(n^{-10})$.
\end{lemma}

\begin{proof}
   The claim follows directly from inequalities \eqref{equ:infty_whole} and \eqref{equ:two_loo}. For any \(\xb \in \Omega\) and \(m \in [n]\), the triangle inequality yields that
    \begin{equation} \label{equ:bound_twoinfty}
        \big \|\btheta^{\gamma, (m)}(\xb) - \btheta^*(\xb) \big\|_{\infty} \leq \max_{1 \leq m^* \leq n} \big \|\btheta^{\gamma, (m^*)}(\xb) - \btheta^{\gamma}(\xb) \big\|_{2} + \big \|\btheta^{\gamma}(\xb) - \btheta^*(\xb) \big\|_{\infty}.
    \end{equation}
    
    We next apply the upper bounds from \eqref{equ:infty_whole} and \eqref{equ:two_loo} to the right-hand side of \eqref{equ:bound_twoinfty}. We then complete the proof by taking the maximum over all \(m \in [n]\) and \(\xb \in \Omega\) on the left-hand side of \eqref{equ:bound_twoinfty}.
\end{proof}

\begin{lemma}[Difference between $\mathcal{N}_2$ and $\widetilde{\mathcal{N}}_2$]\label{lem:term_difference} Consider the same assumptions as in Lemma \ref{lem:two_loo}. Define
    \begin{equation*}
        \mathcal{N}_2 = \sum_{i \neq m, (i,m) \in \mathcal{E}} \sum_{\ell \in [L_{im}]} \mathcal{K}_h \Big(\bX_{im}^\ell - \xb \Big) \bigg\{ \Big(-y_{im}(\bX_{im}^\ell) + \frac{\exp (\theta^{\gamma, (m)}_m(\bX_{im}^\ell))}{\exp{(\theta^{\gamma, (m)}_i(\bX_{im}^\ell))} + \exp(\theta^{\gamma, (m)}_m(\bX_{im}^\ell))} \Big) \big(\eb_m - \eb_i\big) \bigg\},
    \end{equation*}
        and
        \begin{align*}
            \widetilde{\mathcal{N}}_2 = \sum_{i \neq m} \sum_{\ell \in [L_{im}]} p \mathcal{K}_h \Big( \bX_{im} ^ \ell - \xb \Big) \bigg\{ \Big(& - \frac{\exp (\theta_m^*(\bX_{im}^\ell))}{\exp{(\theta_m^*(\bX_{im} ^ \ell))} + \exp(\theta_i^*(\bX_{im} ^ \ell))} \\ 
            & + \frac{\exp (\theta^{\gamma, (m)}_m(\bX_{im}^\ell))}{\exp{(\theta_m^{\gamma, (m)}(\bX_{im}^\ell))} + \exp(\theta_i^{\gamma, (m)}(\bX_{im}^\ell))} \Big) \big(\eb_m - \eb_i\big) \bigg\}.
        \end{align*}
        Then, we have
        \begin{equation*}
            \big\| \mathcal{N}_2 - \widetilde{\mathcal{N}}_2 \big\|_2 \lesssim \sqrt{np \log n} L \bigg(h^2 + \sqrt{\frac{\log(nh^{d/2-1})}{Lh^d}}\bigg)
        \end{equation*}
        for any   $m \in [n]$ with probability at least $1 - O(n^{-10})$.
    \end{lemma}
    
    \begin{proof}
        The proof is inspired by the proof of Lemma 16 of \citet{chen2019spectral}. First, we have the following decomposition:
        \begin{align*}
            &\mathcal{N}_2 - \widetilde{\mathcal{N}}_2  = \underbrace{\sum_{i \neq m} \mathcal{E}_{im} \sum_{\ell \in [L_{im}]} \mathcal{K}_h(\bX_{im}^\ell -\xb) \bigg\{ \Big( \frac{e^{\theta^*_m(\bX_{im}^\ell)}}{e^{\theta^*_m(\bX_{im}^\ell)} + e^{\theta^*_i(\bX_{im}^\ell)}} - y_{im}(\bX_{im}^\ell)\Big) \big(\eb_m - \eb_i\big) \bigg\}}_{\mathcal{J}_1} \\
            & + \underbrace{ \sum_{i \neq m} (p - \mathcal{E}_{im}) \sum_{\ell \in [L_{im}]} \mathcal{K}_h(\bX_{im}^\ell - \xb) \bigg\{ \Big( \frac{e^{\theta^*_m(\bX_{im}^\ell)}}{e^{\theta^*_m(\bX_{im}^\ell)} + e^{\theta^*_i(\bX_{im}^\ell)}} - \frac{e ^ {\theta^{\gamma, (m)}_m(\bX_{im}^\ell)}}{e ^ {\theta^{\gamma, (m)}_m(\bX_{im}^\ell)} + e ^ {\theta^{\gamma, (m)}_i(\bX_{im}^\ell)}}  \Big) \big(\eb_m - \eb_i\big) \bigg\}}_{\mathcal{J}_2},
        \end{align*}
        where $\mathcal{E}_{ij}$ indicates the existence of edge $(i,j)$ for the Erdős–Rényi random graph $\mathcal{G}(\mathcal{V}, \mathcal{E})$. We denote the vectors $\mathcal{J}_i$, where $i \in \{1,2\}$, as $\mathcal{J}_i = \{j_1^{(i)}, j_2^{(i)}, \ldots, j_n^{(i)}\}$. We consider the $\ell_2$-norm of $\mathcal{J}_i$:
        \begin{equation*}
            \big\|\mathcal{J}_i \big\|_2 = \big|j_m^{(i)} \big| + \sqrt{\sum_{k:k \neq m; (k,m) \in \mathcal{E}} \big(j_k^{(i)} \big) ^ 2},
        \end{equation*}
        where we evaluate $j_m^{(i)}$ seperately as compared to $j_k^{(i)}$ where $k \neq m$. The intuition comes from the leave-one-out technique, where the  $m$-th coefficient contributes the most to the $\ell_2$-norm of $\mathcal{J}_i$. The idea is formalized in the following steps of deduction. 
        
        First, we bound $\mathcal{J}_1$. We discuss the entry of $\mathcal{J}_1 = \big\{j_1^{(1)}, j_2^{(1)}, \ldots , j_n^{(1)} \big\}$ by three different categories:
        \begin{equation*}
            j^{(1)}_k = \begin{cases}
                & \sum_{i \neq m} \mathcal{E}_{im} \sum_{\ell \in [L_{im}]} \mathcal{K}_h(\bX_{im}^\ell -\xb) \Big( \frac{e^{\theta^*_m(\bX_{im}^\ell)}}{e^{\theta^*_m(\bX_{im}^\ell)} + e^{\theta^*_i(\bX_{im}^\ell)}} - y_{im}(\bX_{im}^\ell)\Big),\ k = m.\\
                & - \sum_{k \in [L_{km}]} \mathcal{K}_h (\bX_{km}^{\ell} - \xb) \Big( \frac{e^{\theta^*_m(\bX_{km}^\ell)}}{e^{\theta^*_m(\bX_{km}^\ell)} + e^{\theta^*_k(\bX_{km}^\ell)}} - y_{km}(\bX_{km}^\ell) \Big), \ k \neq m \text{ and } (k,m) \in \mathcal{E}.\\
                & 0, \ \text{otherwise}.
             \end{cases}
        \end{equation*}
            
        Then, we discuss the three categories separately. For $k \neq m$ and $(k,m) \in \mathcal{E}$, we take the contraction over $O(L)$ i.i.d. samples of $\bX_{km} ^ \ell$. Specifically, we have that
        \begin{equation*}
            \big|j_k^{(1)} \big| = L_{km} \bigg|\frac{1}{L_{km}} \sum_{k \in [L_{km}]} \mathcal{K}_h (\bX_{km}^{\ell} - \xb) \Big( \frac{e^{\theta^*_m(\bX_{km}^\ell)}}{e^{\theta^*_m(\bX_{km}^\ell)} + e^{\theta^*_k(\bX_{km}^\ell)}} - y_{km}(\bX_{km}^\ell) \Big)  \bigg|.
        \end{equation*}
        By Assumption \ref{ass:similar_comparison}, Lemma~\ref{lem:vrmean_estim}, and Lemma~\ref{lem:vrvar_estim}, we have
        \begin{equation*}
            \sup_{k \neq m}\bigg|\frac{1}{L_{km}} \sum_{k \in [L_{km}]} \mathcal{K}_h (\bX_{km}^{\ell} - \xb) \Big( \frac{e^{\theta^*_m(\bX_{km}^\ell)}}{e^{\theta^*_m(\bX_{km}^\ell)} + e^{\theta^*_k(\bX_{km}^\ell)}} - y_{km}(\bX_{km}^\ell) \Big)  \bigg| \lesssim \sqrt{\frac{\log(nh^{d/2-1})}{Lh^d}},
        \end{equation*}
        with probability at least $1 - O(n ^ {-10})$.  
        
        For \(k \neq m\) and \((k,m) \notin \mathcal{E}\), we have \(j_k^{(1)} = 0\).
        
        For $k = m$, Lemma \ref{lem:meanvar_erdos} implies that
        \begin{equation*}
            \Bigg| \frac{\sum_{i \neq m} \mathcal{E}_{im} \sum_{\ell \in [L_{im}]} \mathcal{K}_h(\bX_{im}^\ell -\xb) \Big( \frac{e^{\theta^*_m(\bX_{im}^\ell)}}{e^{\theta^*_m(\bX_{im}^\ell)} + e^{\theta^*_i(\bX_{im}^\ell)}} - y_{im}(\bX_{im}^\ell)\Big) }{\sum_{i \neq m} \mathcal{E}_{im} L_{im}} \Bigg| \lesssim \sqrt{\frac{\log (nh^{d/2-1})}{npLh^d}},
        \end{equation*}
        with probability at least $1 - O(n ^ {-10})$. 
        
        In conclusion, we have
        \begin{equation} \label{equ:j_1}
            \big\|\mathcal{J}_1 \big\|_2 = \big|j_m^{(1)} \big| + \sqrt{\sum_{\{k: k \neq m; (k,m) \in \mathcal{E}\}} \big(j_k^{(1)}\big) ^ 2 } \lesssim npL \sqrt{\frac{\log(nh^{d/2-1})}{npLh^d}} + \sqrt{np L ^ 2 \cdot \frac{\log (nh^{d/2-1})}{Lh^d}},
        \end{equation}
        with probability at least $1 - O(n^{-10})$. The right-hand side of \eqref{equ:j_1} is dominated by
        \begin{equation} \label{equ:j1_final}
            npL \sqrt{\frac{\log(nh^{d/2-1})}{npLh^d}} + \sqrt{np} L \sqrt{\frac{\log (nh^{d/2-1})}{Lh^d}} \lesssim L \sqrt{np \log {n}}  \sqrt{\frac{\log(nh^{d/2-1})}{Lh^d}}.
        \end{equation}
        
        For $\mathcal{J}_2 = \big\{j_1^{(2)}, j_2^{(2)}, \ldots, j_n^{(2)} \big\}$, we have
        {\small
        \begin{equation*}
            j^{(2)}_k = \begin{cases}
                & \sum_{i \neq m} (p - \mathcal{E}_{im}) \sum_{\ell \in [L_{im}]} \mathcal{K}_h(\bX_{im}^\ell -\xb) \bigg\{ \Big( \frac{e^{\theta^*_m(\bX_{im}^\ell)}}{e^{\theta^*_m(\bX_{im}^\ell)} + e^{\theta^*_i(\bX_{im}^\ell)}} - \frac{e ^ {\theta^{\gamma, (m)}_m(\bX_{im}^\ell)}}{e ^ {\theta^{\gamma, (m)}_m(\bX_{im}^\ell)} + e ^ {\theta^{\gamma, (m)}_i(\bX_{im}^\ell)}}  \Big) \bigg\},\ k = m.\\
                & (1 - p) \sum_{\ell \in [L_{km}]} \mathcal{K}_h(\bX_{km}^\ell - \xb)\bigg\{ \Big( \frac{e^{\theta^*_m(\bX_{km}^\ell)}}{e^{\theta^*_m(\bX_{km}^\ell)} + e^{\theta^*_k(\bX_{km}^\ell)}} - \frac{e ^ {\theta^{\gamma, (m)}_m(\bX_{km}^\ell)}}{e ^ {\theta^{\gamma, (m)}_m(\bX_{km}^\ell)} + e ^ {\theta^{\gamma, (m)}_k(\bX_{km}^\ell)}}  \Big) \bigg\}, \ k \neq m,(k,m) \in \mathcal{E}.\\
                & -p \sum_{\ell \in [L_{km}]} \mathcal{K}_h(\bX_{km}^\ell - \xb)\bigg\{ \Big( \frac{e^{\theta^*_m(\bX_{km}^\ell)}}{e^{\theta^*_m(\bX_{km}^\ell)} + e^{\theta^*_k(\bX_{km}^\ell)}} - \frac{e ^ {\theta^{\gamma, (m)}_m(\bX_{km}^\ell)}}{e ^ {\theta^{\gamma, (m)}_m(\bX_{km}^\ell)} + e ^ {\theta^{\gamma, (m)}_k(\bX_{km}^\ell)}}  \Big) \bigg\}, \ \text{otherwise}.
             \end{cases}
        \end{equation*}}
        By the mean value theorem between $\theta^*_i(\xb) - \theta^*_m(\xb)$ and $\theta^{\gamma, (m)}_i(\xb) - \theta^{\gamma, (m)}_m(\xb)$, we have
        \begin{align*}
            \bigg| \frac{1}{L_{km}} \sum_{\ell \in [L_{km}]} \mathcal{K}_h(\bX_{km}^\ell &- \xb)  \bigg\{ \Big( \frac{e^{\theta^*_m(\bX_{km}^\ell)}}{e^{\theta^*_m(\bX_{km}^\ell)} + e^{\theta^*_k(\bX_{km}^\ell)}} - \frac{e ^ {\theta^{\gamma, (m)}_m(\bX_{km}^\ell)}}{e ^ {\theta^{\gamma, (m)}_m(\bX_{km}^\ell)} + e ^ {\theta^{\gamma, (m)}_k(\bX_{km}^\ell)}}  \Big) \bigg\} \bigg| \\
            & \lesssim \bigg| \frac{1}{L_{km}} \sum_{\ell \in [L_{km}]} \mathcal{K}_h(\bX_{km}^\ell - \xb)\bigg| \sup_{\xb \in \Omega} \max_{1 \leq m \leq n} \big\|\btheta^*(\xb) - \btheta^{\gamma, (m)}(\xb) \big\|_\infty.
        \end{align*}
        Meanwhile, by Lemma~\ref{lem:mean_estim} and Lemma~\ref{lem:var_estim}, we have that with a probability $1 - O(n ^ {-11})$,
        \begin{equation*}
            \bigg| \frac{1}{L_{km}} \sum_{\ell \in [L_{km}]} \mathcal{K}_h(\bX_{km}^\ell - \xb) - f_{\bX}(\xb) \bigg| \lesssim h^2 + \sqrt{\frac{\log (nh^{d/2-1})}{Lh^d}},
        \end{equation*}
        where the right-hand side term $h^2 + \sqrt{\frac{\log (nh^{d/2-1})}{Lh^d}} = o(1)$ by Assumption \ref{ass:erdos_p}.
        
        Furthermore, Assumption \ref{ass:time_distribution} guarantees that $\frac{1}{L_{km}} \sum_{\ell \in [L_{km}]} \mathcal{K}_h(\bX_{km}^\ell - \xb) = O(1)$. Consequently, for any $k \neq m$,
        \begin{equation*}
        \begin{split}
            \bigg| \sum_{\ell \in [L_{km}]} \mathcal{K}_h(\bX_{km}^\ell - \xb) &\bigg\{ \Big( \frac{e^{\theta^*_m(\bX_{km}^\ell)}}{e^{\theta^*_m(\bX_{km}^\ell)} + e^{\theta^*_k(\bX_{km}^\ell)}} - \frac{e ^ {\theta^{\gamma, (m)}_m(\bX_{km}^\ell)}}{e ^ {\theta^{\gamma, (m)}_m(\bX_{km}^\ell)} + e ^ {\theta^{\gamma, (m)}_k(\bX_{km}^\ell)}}  \Big) \bigg| \\
            &\lesssim L \sup_{\xb \in \Omega} \max_{1 \leq m \leq n} \big \|\btheta^*(\xb) - \btheta^{\gamma, (m)}(\xb) \big\|_\infty \text{ with probability } 1 - O(n^{-10}).
        \end{split}
        \end{equation*}

        Therefore, for $k \neq m$ and $(k, m) \in \mathcal{E}$, we have
        \begin{equation} \label{equ:jm3_E}
            \big|j_k^{(2)} \big| \lesssim (1 - p) L \sup_{\xb \in \Omega} \max_{1 \leq m \leq n}  \big\|\btheta^*(\xb) - \btheta ^ {\gamma, (m)}(\xb) \big\|_\infty,
        \end{equation}
        and for $(k, m) \notin \mathcal{E}$, we have
        \begin{equation} \label{equ:jm3_noE}
            \big|j_k^{(2)} \big| \lesssim p L \sup_{\xb \in \Omega} \max_{1 \leq m \leq n}  \big\|\btheta^*(\xb) - \btheta ^ {\gamma, (m)}(\xb) \big\|_\infty,
        \end{equation}
        with probability at least $1 - O(n^{-10})$.
    
        For \(k = m\), by Lemma \ref{lem:bernstein} conditioning on  \(\bX_{im}^\ell\) and setting \(a = 10\), we obtain
        \begin{equation}\label{equ:jm3_EE}
            \big|j_m^{(2)}\big| \lesssim \sqrt{np\log n}L \sup_{\xb \in \Omega} \max_{1 \leq m \leq n}  \big\|\btheta ^ * (\xb) - \btheta ^ {\gamma, (m)}(\xb) \big\|_\infty \text{ with probability } 1 - O(n^{-10}).
        \end{equation}

        Finally, combining \eqref{equ:jm3_E}, \eqref{equ:jm3_noE}, and \eqref{equ:jm3_EE}, we have
        \begin{align*}
            & \big\|\mathcal{J}_2 \big\|_2 \leq \big|j_m^{(2)}\big| + \sqrt{\sum_{\{k: k \neq m; (k,m) \in \mathcal{E}\}} (j_k^{(2)}) ^ 2 + \sum_{\{k: k \neq m; (k,m) \notin \mathcal{E}\}} (j_k^{(2)}) ^ 2} \\
            & \lesssim \sqrt{np\log n}L \sup_{\xb \in \Omega} \max_{1 \leq m \leq n}  \big\|\btheta ^ * (\xb) - \btheta ^ {\gamma, (m)}(\xb) \big\|_\infty + \sqrt{n p(1-p) L ^ 2 \big(\sup_{\xb \in \Omega} \max_{1 \leq m \leq n}  \big\|\btheta ^ * (\xb) - \btheta ^ {\gamma, (m)}(\xb) \big\|_\infty \big)^ 2} \\
            & \lesssim \sqrt{np\log n}L  \sup_{\xb \in \Omega} \max_{1 \leq m \leq n}\big\|\btheta^*(\xb) - \btheta ^ {\gamma, (m)}(\xb) \big\|_\infty \text{ with probability } 1 - O(n^{-10}).
        \end{align*}
        Given the upper bound of $\sup_{\xb \in \Omega} \max_{1 \leq m \leq n}\big\|\btheta^*(\xb) - \btheta ^ {\gamma, (m)}(\xb) \big\|_\infty$ as indicated by \eqref{equ:infty_loo}, we conclude that with probability at least $1 - O(n^{-10})$, we have
        \begin{equation*}
            \big\|\mathcal{N}_2 - \widetilde{\mathcal{N}}_2 \big\|_2 \leq \big\|\mathcal{J}_1 \big\|_2 + \big\|\mathcal{J}_2 \big\|_2 \lesssim \sqrt{n p \log{n}} L \Big( h ^ 2 + \sqrt{\frac{\log (nh^{d/2-1})}{Lh^d}} \Big),
        \end{equation*}
        which completes the proof.
\end{proof}

\subsection{Auxillary Lemmas for proving Lemma \ref{lem:distance_hat}} \label{sec:auxi_main}

We present two auxiliary lemmas for proving Lemma \ref{lem:distance_hat}. First, we establish the Restrictive Strong Convexity (RSC) condition and the Restrictive Strong Smoothness (RSS) condition for \(\mathcal{L}_\lambda(\btheta; \xb)\) for any \(\xb \in \Omega\). We denote the Hessian matrix of \(\mathcal{L}_\lambda(\btheta; \xb)\) by \(\Hb_\lambda(\btheta; \xb)\).

\begin{lemma}[RSC/RSS Condition for $\mathcal{L}_{\lambda}$] \label{lem:RSC_RSS}
    Under Assumptions \ref{ass:erdos_p} to \ref{ass:kernel_smooth}, the regularized likelihood function \(\mathcal{L}_{\lambda}(\btheta,\xb)\) for \(\xb \in \Omega\) and \(\btheta(\xb) \in \Theta(C_0) = \{ \btheta: \inf_{\xb' \in \Omega} \|\btheta - \btheta^* (\xb') \| \leq C_0 \}\) with some positive \(C_0\), satisfies the RSC/RSS condition that
    \begin{equation} \label{equ:ineq_rscrss}
        \lambda \leq \inf_{\xb \in \Omega} \inf_{\btheta \in \Theta} \lambda_{\min, \perp} (\Hb_\lambda(\btheta; \xb)) \leq \sup_{\xb \in \Omega} \sup_{\btheta \in \Theta} \lambda_{\max} (\Hb_\lambda(\btheta; \xb)) \leq \frac{C}{n} + \lambda,
    \end{equation}
    with probability at least $1 - O(n ^ {-10})$. Here, \(C\) is a positive constant determined by the maximal/minimal values of all \(L_{ij}\) where \((i,j) \in \mathcal{E}\), and the upper/lower bounds for \(f_{\bX}(\xb)\) as \(\xb \in \Omega\).
\end{lemma}
\begin{proof}
    The first inequality in \eqref{equ:ineq_rscrss} holds by the strong convexity of $\mathcal{L}_\lambda$ guaranteed by the regularization term as $\Hb_\lambda(\btheta; \xb) \succeq \lambda \Ib_n$ for any $\btheta$. We then apply Lemma \ref{lem:hessian_bound} and complete the proof.
\end{proof}

Furthermore, we bound the distance between the true value $\btheta ^*(\xb)$ and the regularized MLE~$\widehat{\btheta} (\xb)$ in Lemma \ref{lem:distance_star}.

\begin{lemma}[Distance between $\widehat{\btheta}(\xb)$ and $\btheta ^ * (\xb)$]\label{lem:distance_star} 
    There exists some positive constant $C$ such that $\sup_{\xb \in \Omega} \big\|\widehat{\btheta}(\xb) - \btheta ^*(\xb) \big\|_2 \leq C \sqrt{n}$ with probability at least  $1 - O(n^{-10})$.
\end{lemma}

\begin{proof}
    First, we consider the Taylor expansion of \(\mathcal{L}_{\lambda}(\widehat{\btheta}, \xb)\) that
    \begin{equation*}
        \mathcal{L}_{\lambda}(\widehat{\btheta}, \xb) = \mathcal{L}_{\lambda}(\btheta^*, \xb) + \big\langle \nabla \mathcal{L}_{\lambda}(\btheta^*; \xb), \widehat{\btheta} (\xb) - \btheta ^ * (\xb) \big\rangle + \frac{1}{2} \big(\widehat{\btheta} (\xb)- \btheta ^ * (\xb) \big) \trans \Hb_\lambda(\widetilde{\btheta}; \xb) \big(\widehat{\btheta} (\xb) - \btheta ^ * (\xb) \big),
    \end{equation*}
    where \(\widetilde{\btheta}\) lies between \(\widehat{\btheta}\) and \(\btheta ^ *\). Since \(\mathcal{L}_{\lambda}(\widehat{\btheta}, \xb) \leq \mathcal{L}_{\lambda}(\btheta^*, \xb)\), we have
    \begin{equation*}
        \big\langle \nabla \mathcal{L}_{\lambda}(\btheta^*; \xb),\widehat{\btheta} (\xb)- \btheta ^ * (\xb)  \big\rangle + \frac{1}{2} \big(\widehat{\btheta} (\xb)- \btheta ^ * (\xb) \big) \trans \Hb_\lambda(\widetilde{\btheta}; \xb) \big(\widehat{\btheta} (\xb) - \btheta ^ * (\xb) \big) \big\rangle \leq 0.
    \end{equation*}
    By the Cauchy-Schwarz inequality, we have
    \begin{equation} \label{equ:dist_CS}
        \frac{1}{2} \big(\widehat{\btheta} (\xb)- \btheta ^ * (\xb) \big) \trans \Hb_\lambda(\widetilde{\btheta}; \xb) \big(\widehat{\btheta} (\xb) - \btheta ^ * (\xb) \big) \leq \big\| \nabla \mathcal{L}_\lambda (\btheta^*; \xb) \big\|_2 \big\|\widehat{\btheta}(\xb) - \btheta ^*(\xb) \big\|_2.
    \end{equation}
    Therefore, we derive
    \begin{equation} \label{equ:hat_star}
        \begin{split}
            \big\| \widehat{\btheta}(\xb) - \btheta ^*(\xb) \big\|_2 & \leq \frac{2 \big\|\nabla \mathcal{L}_\lambda (\btheta^*; \xb) \big\|_2}{ \min_{\xb \in \Omega} \lambda_{\min, \perp} \big(\Hb_\lambda(\widetilde{\btheta}; \xb) \big)} \\
            &\overset{(i)}{\leq} \frac{2 \big\|\nabla \mathcal{L}_\lambda (\btheta^*; \xb) \big\|_2}{\lambda} \overset{(ii)}{\leq} \frac{2 \sqrt{n} \big\|\nabla \mathcal{L}_\lambda (\btheta^*; \xb) \big\|_\infty}{\lambda} \overset{(iii)}{\lesssim} \frac{h^2 + \sqrt{\frac{\log (n h^{d/2-1})}{npLh^d}}}{\sqrt{n}\lambda},
        \end{split}
    \end{equation}
    where \((i)\) holds as \(\Hb_\lambda(\widetilde{\btheta}; \xb) \succeq \lambda \Ib_n\), \((ii)\) holds by the norm equivalence inequality \(\| \cdot \|_2 \leq \sqrt{n} \| \cdot \|_\infty\), and \((iii)\) holds as we apply Lemma \ref{lem:gradient_bound}. 
    
Finally, we set \(\lambda \asymp \frac{1}{n} \Big( h^2 + \sqrt{\frac{\log (nh^{d/2-1})}{npLh^d}} \Big)\) to ensure the uniformly of \eqref{equ:hat_star} for any \(\xb \in \Omega\), which completes the proof.
\end{proof}

\subsection{Proof for the Bound on the Gradient}

We discuss the behavior of the gradient $\nabla \mathcal{L}(\btheta^*; \xb)$ that 
{\small
\begin{equation} \label{equ:gradient}
\begin{aligned}
    &\nabla \mathcal{L}(\btheta^*; \xb) \\
    & \ \ = \frac{1}{n^2pL} \sum_{(i,j) \in \mathcal{E}} \sum_{\ell \in [L_{ij}]} \mathcal{K}_h \Big(\bX_{ij}^\ell - \xb \Big) \bigg\{ \Big(-y_{ij}(\bX_{ij}^\ell) + \frac{\exp (\theta^*_j(\bX_{ij}^\ell))}{\exp{(\theta^*_i(\bX_{ij}^\ell))} + \exp(\theta^*_j(\bX_{ij}^\ell))} \Big) 
   \big(\eb_j - \eb_i\big) \bigg\}.
\end{aligned}
\end{equation}
}

\begin{lemma}[Bound on the Gradient on $\btheta^*$] \label{lem:gradient_bound}
Given the conditions in Section \ref{sec:assumptions},  there exists a positive constant $C$ such that
\begin{equation*} %\label{equ:gradient_bound}
    \sup_{\xb \in \Omega} \big\| \nabla \mathcal{L}(\btheta^*; \xb) \big\|_{\infty} \leq \frac{C}{n} \sqrt{\frac{\log (nh^{d/2-1})}{npLh^d}},
\end{equation*}
with probability at least $1 - O(n ^ {-10})$. Furthermore, if $\lambda \asymp \frac{1}{n} \Big( h ^ 2 + \sqrt{\frac{\log (nh ^ {d/2-1})}{npLh^d}} \Big)$, there exists a positive constant $C ^ *$ such that
\begin{equation} \label{equ:gradient_lambda}
    \sup_{\xb \in \Omega} \big \| \nabla \mathcal{L}_{\lambda}(\btheta^*; \xb) \big \|_{\infty} \leq \frac{C ^ *}{n} \bigg(h ^ 2 + \sqrt{\frac{\log (nh^{d/2-1})}{npLh ^ d}} \bigg)
\end{equation}
with probability at least $1 - O(n ^ {-10})$.
\end{lemma}

\begin{proof}
    To facilitate our discussion, we let
    \begin{equation*}
        \bA = \sum_{(i,j) \in \mathcal{E}} \sum_{\ell \in [L_{ij}]} \mathcal{K}_h \Big(\bX_{ij} ^ \ell - \xb \Big) \bigg\{ \Big( -y_{ij}(\bX_{ij} ^ \ell) + \frac{\exp{(\theta^*_j(\bX_{ij} ^ \ell))}}{\exp{(\theta^*_i(\bX_{ij} ^ \ell))} + \exp{(\theta^*_j(\bX_{ij} ^ \ell))}} \Big) \big(\eb_j - \eb_i \big) \bigg\},
    \end{equation*}
    and denote the $k ^ {\text{th}}$ entry of the gradient $\nabla \mathcal{L}(\btheta^*; \xb)$ as $A_k$:
    \begin{equation*}
            A_k = \sum_{(i,k) \in \mathcal{E}} \sum_{\ell \in [L_{ik}]} \mathcal{K}_h(\bX_{ik} ^ \ell - \xb) \bigg\{-y_{ik} (\bX_{ik} ^ \ell) + \frac{e^{\theta_k^*(\bX_{ik}^\ell)}}{e^{\theta_i^*(\bX_{ik}^\ell)}+e^{\theta_k^*(\bX_{ik}^\ell)}} \bigg\}.
    \end{equation*}
    Consider the residuals $\epsilon_{ik} ^ \ell$ that
    \begin{equation*}
        \epsilon^\ell_{ik} = -y_{ik}(\bX_{ik} ^ \ell) + \frac{e ^ {\theta_k^*} (\bX_{ik} ^ \ell)}{e ^ {\theta_i^*} (\bX_{ik} ^ \ell) + e ^ {\theta_k^*} (\bX_{ik} ^ \ell)}.
    \end{equation*}
    We observe that the residuals \(\epsilon_{ik}^\ell\) are independent across all pairs \((i,k) \in \mathcal{E}\) due to the independence of \(\bX_{ik}^\ell\)'s, as ensured by Assumption \ref{ass:time_independence}. Furthermore, these residuals are uniformly bounded that \(|\epsilon_{ik}^\ell| \leq 2\), and their mean is zero, as specified by the model setup. Consequently, the conditions listed in Section \ref{sec:kernel_regression} are satisfied, and the mean-variance decomposition  in Lemma \ref{lem:meanvar_erdos} is applicable for the proof. Specifically, we have
    \begin{equation*} %\label{equ:bound_A2}
        \sup_{k \in [n]} \sup_{\xb \in \Omega} \big| A_k \big| = O\bigg(npL  \sqrt{\frac{\log(nh ^ {d/2-1})}{npLh ^ d}}\bigg), \text{ with probability at least} 1 - O(n^{-10}).
    \end{equation*} It further implies that for any $k \in [n]$,
    \begin{equation*} 
    \begin{split}
        \sup_{\xb \in \Omega} \big\| \nabla \mathcal{L}(\btheta ^ *; \xb)\big\|_{\infty} = \frac{1}{n^2p L} \sup_{k \in [n]} & \sup_{\xb \in \Omega} \big| A_k \big| \lesssim \frac{npL \Big(\sqrt{\frac{\log(nh ^ {d/2-1})}{npLh ^ d}}\Big)}{n^2 pL} = \frac{1}{n} \bigg(\sqrt{\frac{\log(nh ^ {d/2-1})}{npLh ^ d}}\bigg),
    \end{split}
    \end{equation*}
    with probability at least $1 - O(n^{-10})$. For the regularized gradient $\nabla \mathcal{L}_\lambda (\btheta ^ *; \xb) $, \eqref{equ:gradient_lambda} holds when $\lambda \asymp \frac{1}{n} \Big(h ^ 2 + \sqrt{\frac{\log (nh ^ {d/2-1})}{npLh ^ d}}\Big)$.

    Finally, applying the triangle inequality completes the proof.
\end{proof}
    
\subsection{Proof for the Bound on the Hessian Matrix}

   We derive bounds for the eigenvalues of Hessian matrix $\Hb(\btheta; \xb)$ given by
    {\small
    \begin{equation} \label{equ:hessian}
        \Hb(\btheta; \xb) = \frac{1}{n^2pL} \sum_{(i,j) \in \mathcal{E}} \sum_{\ell \in [L_{ij}]} \mathcal{K}_h \Big(\bX_{ij}^\ell - \xb \Big) \bigg\{ \frac{\exp (\theta_i(\bX_{ij}^\ell)) \exp (\theta_j(\bX_{ij}^\ell))}{(\exp (\theta_i(\bX_{ij}^\ell)) + \exp (\theta_j(\bX_{ij}^\ell))) ^ 2} \big(\eb_j - \eb_i\big)  \big(\eb_j - \eb_i\big) \trans  \bigg\}.
    \end{equation}}
    
    We introduce the matrix \(\Qb_{\mathcal{G}} = \Qb_{\mathcal{G(V, E)}}\) associated with the Erdős–Rényi random graph \(\mathcal{G} = \mathcal{G(V, E)}\), where \(|\mathcal{V}| = n\), and each edge is included with probability \(p\). We define \(\Qb_{\mathcal{G}}\) as
    \begin{equation*}
        \Qb_{\mathcal{G}} = \sum_{(i,j) \in \mathcal{E}} (\eb_j - \eb_i)(\eb_j - \eb_i) \trans.
    \end{equation*}
    
    \begin{lemma}[Bound on the Eigenvalues of the Hessian Matrix] \label{lem:hessian_bound}
    For any $\xb \in \Omega$ and positive constant $C_0$ such that $\btheta \in \Theta(C_0) = \big\{ \btheta: \inf _{\xb' \in \Omega} \| \btheta - \btheta^* (\xb') \|_{\infty} \leq C_0 \big\}$, there exist positive constants $0 < C_{\min{}} \leq C_{\max{}} < \infty$ such that the following holds uniformly with probability at least $1 - O(n ^ {-10})$:
    \begin{equation*} 
       \inf_{\xb \in \Omega} \inf_{\btheta \in \Theta} \lambda_{\min, \perp} \big(\Hb(\btheta; \xb) \big) \geq \frac{C_{\min{}}}{n} \text{\ \ \  and \ \ } \sup_{\xb \in \Omega} \sup_{\btheta \in \Theta} \lambda_{\max} \big(\Hb(\btheta; \xb) \big) \leq \frac{C_{\max{}}}{n}.
    \end{equation*}
     Here $\lambda_{\min{}, \perp}$ denotes the smallest eigenvalue that is restricted to vectors orthogonal to $\mathbf{1}$.
    \end{lemma}
    
    \begin{proof}
        Given the existence of a positive constant \(C_0\) such that \(\btheta \in \Theta(C_0)\), Lemma \ref{lem:constrain_function} shows that for any   \(\btheta \in \Theta(C_0)\) and \(\xb \in \Omega\),
        \begin{equation*}
            \frac{\exp (\theta_i(\xb)) \exp (\theta_j(\xb))}{\big(\exp (\theta_i(\xb)) + \exp (\theta_j(\xb)) \big) ^ 2} \in \bigg[ \frac{1}{4 \kappa e ^ {2 C_0}}, \frac{1}{4} \bigg].
        \end{equation*}
        We define matrices \(\Mb (\btheta, \xb)\) and $\Pb(\xb)$ as
        \begin{equation*}
            \Mb (\btheta, \xb) = \sum_{(i,j) \in \mathcal{E}} \sum_{\ell \in [L_{ij}]} \mathcal{K}_h \Big(\bX_{ij}^\ell - \xb \Big) \bigg\{ \frac{\exp (\theta_i(\bX_{ij}^\ell)) \exp (\theta_j(\bX_{ij}^\ell))}{(\exp (\theta_i(\bX_{ij}^\ell)) + \exp (\theta_j(\bX_{ij}^\ell))) ^ 2} \big(\eb_j - \eb_i\big)  \big(\eb_j - \eb_i\big)  \trans \bigg\},
        \end{equation*}
        and 
        \begin{equation*}
            \Pb(\xb) = \sum_{(i,j) \in \mathcal{E}} \sum_{\ell \in [L_{ij}]} \mathcal{K}_h \Big(\bX_{ij}^\ell - \xb \Big)  (\eb_j - \eb_i)(\eb_j - \eb_i) \trans.
        \end{equation*}
        Since Weyl's inequality implies
        \begin{equation*}
            \lambda_{\max} (\Mb (\btheta, \xb)) \leq \frac{1}{4} \lambda_{\max} (\Pb(\xb));\ \  \lambda_{\min, \perp}(\Mb (\btheta, \xb)) \geq \frac{1}{4 \kappa e^{2C_0}} \lambda_{\min, \perp}(\Pb(\xb)),
        \end{equation*}
        bounding the eigenvalues of \(\Mb (\btheta, \xb)\) defined in \eqref{equ:hessian} suffices to bound the eigenvalues of \(\Pb(\xb)\). 

        First, we bound \(\lambda_{\min, \perp}(\Mb (\btheta, \xb))\). We have the following decomposition of 
         \(\Pb(\xb)\):
         \begin{align*}
         \begin{split}
              \Pb(\xb) &= \underbrace{\sum_{(i,j) \in \mathcal{E}} L_{ij} f_{\bX}(\xb) (\eb_j - \eb_i)(\eb_j - \eb_i) \trans}_{\Nb(\xb)} - \underbrace{ \sum_{(i,j) \in \mathcal{E}} \Big(L_{ij}f_{\bX}(\xb) - \sum_{\ell \in [L_{ij}]}\mathcal{K}_h(\bX_{ij}^\ell - \xb) \Big) (\eb_j - \eb_i)(\eb_j - \eb_i) \trans }_{\Rb(\xb)}.
         \end{split}
        \end{align*}
        By Weyl's inequality, we have \(\lambda_{\min, \perp}(\Pb(\xb)) \geq \lambda_{\min, \perp}(\Nb(\xb)) - \lambda_{\max}(\Rb(\xb))\). 
        
        For $\lambda_{\min, \perp}(\Nb(\xb))$, Lemma \ref{lem:lower_lambdamin} implies that with probability at least \(1 - O(n^{-10})\),
        \begin{equation} \label{equ:min_Nt}
            \begin{split}
                \lambda_{\min{}, \perp}(\Nb(\xb)) \geq L_{m} f_{\bX} (\xb)   \lambda_{\min{}, \perp}\Big( \sum_{(i,j) \in \mathcal{E}} (\eb_j - \eb_i)(\eb_j - \eb_i) \trans \Big) \geq \frac{np L_{m} f_{\bX} (\xb) }{2}.
            \end{split}
        \end{equation}

        For \(\lambda_{\max}(\Rb(\xb))\), we consider \(\mathcal{D}_{ij} = \Big| \frac{1}{L_{ij}} \sum_{\ell \in [L_{ij}]} \mathcal{K}_h(\bX_{ij}^\ell - \xb) - f_{\bX}(\xb) \Big|\). Applying the union bound among \((i,j) \in \mathcal{E}\), we have
        \begin{equation} \label{equ:unionbound_hessian}
        \begin{split}
            & \mathbb{P} \bigg(\sup_{(i,j) \in \mathcal{E}} \mathcal{D}_{ij} \lesssim h ^ 2 + \sqrt{\frac{\log (nh ^ {d/2-1})}{Lh ^ d}} \bigg) = \mathbb{P} \bigg( \cap_{(i,j) \in \mathcal{E}} \Big\{ \mathcal{D}_{ij} \lesssim h ^ 2 + \sqrt{\frac{\log (nh ^ {d/2-1})}{Lh ^ d}} \Big\} \bigg) \\
            & = 1 - \mathbb{P} \bigg( \cup_{(i,j) \in \mathcal{E}} \Big\{ \mathcal{D}_{ij} \lesssim h ^ 2 + \sqrt{\frac{\log (nh ^ {d/2-1})}{Lh ^ d}} \Big\} ^ c \bigg) \\
            & \geq 1 - \sum_{(i,j) \in \mathcal{E}} \mathbb{P} \Big( \Big\{ \mathcal{D}_{ij} \lesssim h ^ 2 + \sqrt{\frac{\log (nh ^ {d/2-1})}{Lh ^ d}} \Big\} ^ c \Big) \\
            & = 1 - \sum_{(i,j) \in \mathcal{E}} \bigg\{ 1 - \mathbb{P} \Big( \Big\{ \mathcal{D}_{ij} \lesssim h ^ 2 + \sqrt{\frac{\log (nh ^ {d/2-1})}{Lh ^ d}} \Big\}  \Big) \bigg\}.
        \end{split}
        \end{equation}
        By applying Lemma \ref{lem:mean_estim} and Lemma \ref{lem:var_estim} with \(\tau = 12\), we have for any \((i,j) \in \mathcal{E}\),
        \begin{equation*}
            \mathbb{P} \Big( \Big\{ \mathcal{D}_{ij} \lesssim h ^ 2 + \sqrt{\frac{\log (nh ^ {d/2-1})}{Lh ^ d}} \Big\}  \Big) = 1 - O(n^{-12}).
        \end{equation*}
        Combining the results with \eqref{equ:unionbound_hessian} and considering the graph structure \(|\mathcal{E}| \lesssim n^2\), we obtain
        \begin{equation*}
            \sup_{(i,j) \in \mathcal{E}}\mathcal{D}_{ij} \lesssim h ^ 2 + \sqrt{\frac{\log (nh ^ {d/2-1})}{Lh^d}} \text{ with probability at least } 1 - O(n^ {-10}).
        \end{equation*}
        Given Assumption \ref{ass:erdos_p}, we conclude that \( \sup_{(i,j) \in \mathcal{E}} \mathcal{D}_{ij} = o(1)\) with probability \(1 - O(n^{-10})\). Lemma~\ref{lem:lower_lambdamin} implies that, with probability at least $1 - O(n ^ {-10})$, the difference of each pair is bounded as
        \begin{equation} \label{equ:max_Rt}
            \lambda_{\max}(\Rb(\xb)) = \lambda_{\max} \bigg( \sum_{(i,j) \in \mathcal{E}} L_{ij} \mathcal{D}_{ij} (\eb_j - \eb_i)(\eb_j - \eb_i) \trans \bigg) = o(npL).
        \end{equation}
         
        Combining \eqref{equ:min_Nt} and \eqref{equ:max_Rt} guarantees the existence of some  $C_{\min{}}>0$, such that
        \begin{equation*}
            \lambda_{\min{}, \perp}(\Mb(\btheta, \xb)) \geq \frac{1}{4 \kappa e^{2C_0}} \big(\lambda_{\min{}, \perp}(\Nb(\xb)) - \lambda_{\max}(\Rb(\xb))\big) \geq C_{\min{}} npL
        \end{equation*}
        with probability at least $1 - O(n^{-10})$. Therefore, by \eqref{equ:hessian}, we have
        \begin{equation*}
            \lambda_{\min{}, \perp} \big(\Hb(\btheta; \xb) \big) \geq \frac{C_{\min{}} npL}{n^2pL} = \frac{C_{\min{}}}{n}, \text{ with probability at least } 1 - O(n^{-10}).
        \end{equation*}
        
        Then, we bound \(\lambda_{\max}(\Mb(\btheta, \xb))\). Consider \(\lambda_{\max}(\Pb(\xb))\), and Weyl's inequality implies that
         \begin{equation*}
             \lambda_{\max{}}(\Pb(\xb)) \leq \lambda_{\max{}}(\Nb(\xb)) + \lambda_{\max{}}(\Rb(\xb)).
         \end{equation*}
         For $\lambda_{\max{}}(\Nb(\xb))$, Lemma \ref{lem:lower_lambdamin} implies that
         \begin{equation*}
             \lambda_{\max{}}(\Nb(\xb)) \leq L_M f_{\bX}(\xb) \lambda_{\max{}} \bigg(\sum_{(i,j) \in \mathcal{E}} (\eb_j - \eb_i)(\eb_j - \eb_i) \trans \bigg) \leq 2np L_M f_{\bX} (\xb),
         \end{equation*}
         with probability at least\(1 - O(n^{-10})\). Combining the upper bound of \(\lambda_{\max}(\Nb(\xb))\) and \(\lambda_{\max}(\Rb(\xb))\) from \eqref{equ:max_Rt}, we have
         \begin{equation*} %\label{equ:max_Mt}
             \lambda_{\max{}}(\Mb(\btheta, \xb)) \leq \frac{1}{4} \lambda_{\max{}}(\Pb(\xb)) \leq \frac{1}{4} \big (\lambda_{\max{}}(\Nb(\xb)) + \lambda_{\max{}}(\Rb(\xb)) \big) \leq C_{\max{}}npL,
         \end{equation*}
         with probability at least \(1 - O(n^{-10})\). Consequently, we have
         \begin{equation*}
             \lambda_{\max{}}(\Hb(\btheta; \xb)) \leq \frac{C_{\max{}} npL}{n^2 pL} = \frac{C_{\max{}}}{n}, \text{ with probability at least} 1 - O(n^{-10}),
         \end{equation*}
            which completes the proof.
    \end{proof}

    The auxiliary Lemmas \ref{lem:constrain_function} and \ref{lem:lower_lambdamin} provide key technical insights essential for the proof of Lemma \ref{lem:hessian_bound}.
    
    \begin{lemma}[Constraint on Function Value] \label{lem:constrain_function}
        Suppose that Assumption \ref{ass:constraint_lps} holds. For the function \(g(\cdot, \cdot): \mathbb{R} \times \mathbb{R} \to \mathbb{R}\),
        \begin{equation} \label{equ:func}
            g(x_1, x_2) = \frac{e^{x_1} e^{x_2}}{(e^{x_1} + e^{x_2})^2},
        \end{equation}
        we have 
        \begin{equation*}
            \frac{1}{4 \kappa e^{2C}} \leq \min_{1 \leq i, j \leq n} \{g(\theta_i, \theta_j)\} \leq \max_{1 \leq i, j \leq n} \{g(\theta_i, \theta_j)\} \leq \frac{1}{4}
        \end{equation*}
        for all \(\btheta = \{\theta_1, \theta_2, \ldots, \theta_n\} \trans \in \mathbb{R}^n\) such that \(\|\btheta - \btheta^*\|_\infty \leq C\). Here $\kappa$ is the the ratio of the upper bound to the lower bound of the preference scores.
    \end{lemma}
    
    \begin{proof}
        First, we have that
        \begin{equation*}
             \max_{1 \leq i, j \leq n} \{g(\theta_i, \theta_j)\} \leq \frac{1}{4}.
        \end{equation*}
        Therefore, it suffices to show that
        \begin{equation*}
             \min_{1 \leq i, j \leq n} \{g(\theta_i, \theta_j)\} \geq \frac{1}{4 \kappa e ^ {2C}}.
        \end{equation*}
        Without loss of generality, we assume that $\theta_i \leq \theta_j$. Therefore, we have that
        \begin{equation*}
            \frac{e^{\theta_i} e ^ {\theta_j}}{(e^{\theta_i} + e^{\theta_j}) ^ 2} = \frac{e ^ {\theta_i-\theta_j}}{(1 + e^{\theta_i - \theta_j}) ^ 2} = \frac{e ^ {-|\theta_i-\theta_j|}}{(1 + e^{-|\theta_i - \theta_j|}) ^ 2} \geq \frac{e ^ {-|\theta_i-\theta_j|}}{4},
        \end{equation*}
        as $e ^ {-|\theta_i-\theta_j|} > 0$ and $(1 + e^{-|\theta_i - \theta_j|}) ^ 2 \leq 4$. Furthermore, we have that
        \begin{equation*}
            |\theta_i - \theta_j| \leq |\theta_i - \theta^*_i| + |\theta^*_i - \theta^*_j| + |\theta^*_j - \theta_j| \leq 2C + \log \kappa,
        \end{equation*}
        which yields that for any $(\theta_i, \theta_j)$ pairs,
        \begin{equation} \label{equ:g_thetaij}
            g(\theta_i, \theta_j) \geq \frac{e ^ {-|\theta_i-\theta_j|}}{4} \geq \frac{e^{-(2C + \log \kappa)}}{4} = \frac{1}{4\kappa e^ {2C}}.
        \end{equation}
        Here \eqref{equ:g_thetaij} completes the proof by the arbitrarity of pair $(\theta_i, \theta_j)$. 
        
    \end{proof}
    
    Recall that $\lambda_{\min{}, \perp}(\Qb_\mathcal{G})$ is the spectral gap of the Laplacian matrix of the Erdős–Rényi random graph $\mathcal{G}(\mathcal{V}, \mathcal{E})$. The following lemma is Lemma 10 of \citet{chen2019spectral}, as is a direct corollary of (5.3.3) in \citet{tropp2015hoeffding}. We provide it here fore self-completeness.

    \begin{lemma}[Bound on $\lambda_{\max}$ and $\lambda_{\min{}, \perp}$ for $\Qb_{\mathcal{G}}$] \label{lem:lower_lambdamin}
        For any $\tau > 0$, there exists some positive constant $C = C(\tau)$ such that for probability $p$ satisfying Assumption \ref{ass:erdos_p},
        \begin{equation*}
            \mathbb{P} \big(np/2 \leq \lambda_{\min{}, \perp}(\Qb_\mathcal{G}) \leq  \lambda_{\max{}}(\Qb_\mathcal{G}) \leq 2np \big) \geq 1 - O(n ^ {-\tau}).
        \end{equation*}
    \end{lemma}

\section{Proof of Lemmas in Section~\ref{pf:thm:band}}\label{pf:lem_band}
    \begin{lemma} [Selection of $\zeta_1$ and $\zeta_2$ for $|T - T_0|$]\label{lem:zeta_T}
        Consider $T$ defined in \eqref{equ:Def_TT0} and $T_0$ defined in \eqref{equ:Gao_T0}. Given Assumptions \ref{ass:time_independence}--\ref{ass:constraint_lps} and \ref{ass:kernel_general}--\ref{ass:kernel_smooth},  there exist  $\zeta_1 = o \Big(\sqrt{\log (npLh ^ {d/2 - 1} )} \Big)$ and $\zeta_2 = O(n ^ {-10})$ such that
        \begin{equation*}
             \mathbb{P}\big(|T - T_0| > \zeta_1 \big) < \zeta_2,
        \end{equation*}
        where $h$ satisfies $npLh^{d+4} \leq \log(nh^{d/2-1})$.
    \end{lemma}
    
    \begin{proof}
        Let \(T^m(\xb) = \sqrt{h^d \Xi} \cdot \big| \widehat{\theta}_m(\xb) - \theta^*_m(\xb) \big|. \)
        We have the following decomposition:
        \begin{align*}
            \frac{1}{\sqrt{h^d \Xi}} T^m(\xb) & =  (1 + \epsilon_{1,m}(\xb)) \bigg|-\frac{\mathcal{U}_{m}(\xb)}{\mathcal{V}_{m}(\xb)} \bigg| + \epsilon_{2,m}(\xb) \\
            & \leq (1 + \epsilon_{1,m}(\xb))\frac{1}{\sqrt{h^d \Xi}} T_0^m(\xb) + \epsilon_{2,m}(\xb)  + (1 + \epsilon_{1,m}(\xb)) \bigg|\mathcal{U}_{m}(\xb) \frac{\mathcal{V}^*_{m}(\xb) - \mathcal{V}_{m}(\xb)}{\mathcal{V}_{m}(\xb) \mathcal{V}^*_{m}(\xb)} \bigg|.
        \end{align*}
        Denote
        \begin{equation*}
            \epsilon_{3,m}(\xb) = (1 + \epsilon_{1,m}(\xb)) \bigg| \mathcal{U}_{m}(\xb) \frac{\mathcal{V}^*_{m}(\xb) - \mathcal{V}_{m}(\xb)}{\mathcal{V}_{m}(\xb) \mathcal{V}^*_{m}(\xb)} \bigg|,
        \end{equation*}
        and we have
        \begin{equation*}
             T^m(\xb) - T_0^m(\xb)  = \epsilon_{1,m}(\xb) T_0^m(\xb) + \sqrt{h^d \Xi} (\epsilon_{2,m}(\xb) + \epsilon_{3,m}(\xb) ).
        \end{equation*}

        Then, we bound each part separately. For $\epsilon_{3,m}(\xb)$, Lemma \ref{lem:bound_VVstar} indicates that $\big|\mathcal{V}^*_{m} (\xb) - \mathcal{V}_{m} (\xb) \big|$ is bounded above by $O\Big(   \sqrt{\frac{\log (npLh ^ {d/2 - 1})}{npLh^d}} \Big)$. Lemma \ref{lem:bound_VVstar} and Lemma \ref{lem:lowerbound_Vstar} imply there exist \( C_1 , C_2\) such that $\mathcal{V} ^ *_{m} (\xb) \geq C_1 $, and $\mathcal{V}_{m} (\xb) \geq C_2 $ with probability at least $1 - O(n^{-10})$. Furthermore, Lemma~\ref{lem:vrvar_estim} on $\mathcal{U}_{m} (\xb)$ yields that
        \(    \sup_{m \in [n], \xb \in \bX} \big| \mathcal{U}_{m} (\xb)\big| \lesssim \sqrt{\frac{\log (npLh ^ {d/2 - 1})}{npLh^d}}.
        \)
        Therefore, we have
        \begin{equation*}
            \sup_{m \in [n], \xb \in \bX} \bigg| \mathcal{U}_{m}(\xb) \frac{\mathcal{V}^*_{m}(\xb) - \mathcal{V}_{m}(\xb)}{\mathcal{V}_{m}(\xb) \mathcal{V}^*_{m}(\xb)}\bigg| \lesssim \frac{\log (npLh ^ {d/2 - 1})}{npLh^d}
        \end{equation*}
        with probability at least $1 - O(n^{-10})$. By Lemma~\ref{thm:LTO_asymptotic}, \( \|\epsilon_1 \|_\infty = o(1)\) and we have 
        \begin{equation}\label{equ:epsilon3}
            \sup_{m \in [n], \xb \in \bX} \epsilon_{3,m}(\xb) \lesssim  \frac{\log (npLh ^ {d/2 - 1})}{npLh^d}.
        \end{equation}

        For $\epsilon_{2,m}(\xb)$, by Lemma~\ref{thm:LTO_asymptotic}, we have
        \begin{equation}\label{equ:epsilon2}
            \|\epsilon_2 \|_\infty = o\bigg(dh^2 + \sqrt{\frac{\log (npLh ^ {d/2 - 1})}{npLh^d}}\bigg)
        \end{equation}
        
        Combining \eqref{equ:epsilon3} and \eqref{equ:epsilon2}, we have with probability \(1- O(n^{-10})\),
        \begin{align*}
            \sup_{m \in [n], \xb \in \bX} \Big( T^m(\xb) - T_0^m(\xb) \Big) & \lesssim \sqrt{npLh^d} \bigg[ o \Big(dh^2 + \sqrt{\frac{\log (npLh ^ {d/2 - 1} )}{npLh^d}} \Big)+  \frac{\log (npLh ^ {d/2 - 1})}{npLh^d}  \bigg] \\
            &  \lesssim o \Big(\sqrt{\log (npLh ^ {d/2 - 1} )} \Big),
        \end{align*}
        where $h$ satisfies $npLh^{d+4} \leq \log(nh^{d/2-1})$.
        
        Therefore, there exists \( \zeta_1 =o \Big(\sqrt{\log (npLh ^ {d/2 - 1} )} \Big)\) and \( \zeta_2  = O(n^{-10})\) such that 
        \[\mathbb{P}(|\sup_{m \in [n], \xb \in \bX} T^m(\xb) - \sup_{m \in [n], \xb \in \bX} T_0^m(\xb)| = \mathbb{P}(|T - T_0| > \zeta_1) < \zeta_2, \]
        which completes the proof.
    \end{proof}

    \begin{lemma} [Selection of $\zeta_1$ and $\zeta_2$ for $|W - W_0|$]\label{lem:zeta_W}
        Consider $W$ and $W_0$ defined in \eqref{equ:Def_W} and \eqref{equ:Def_W0}. Given Assumptions \ref{ass:time_independence}--\ref{ass:constraint_lps} and \ref{ass:kernel_general}--\ref{ass:kernel_smooth}, there exist $\zeta_1 = O(n^{-\epsilon/4} \sqrt{\log n})$ and $\zeta_2 = O(n ^ {-10})$ such that
        \begin{equation*}
             \mathbb{P}\Big(\mathbb{P}\big(|W - W_0| > \zeta_1 \big| \{y_{im}, \mathcal{E}_{im}, \bX_{im} ^ \ell\}_{i \neq m, \ell \in [L_{im}]}\big) > \zeta_2\Big) < \zeta_2,
        \end{equation*}
        where $h \leq ({npLd^2})^{-\frac{1}{d+4}}$.
        
    \end{lemma}
    
    \begin{proof}
        Let \begin{equation*} %\label{equ:A_1}
        \mathcal{A}_1 = \bigg\{\frac{np}{2} \leq \min_{i \in [n]} d_i \leq \max_{i \in [n]} d_i \leq \frac{3np}{2} \bigg\},
        \end{equation*}  
        where $d_i = \sum_{j \neq i; j \in [n]} \mathcal{E}_{ij}$. We have
        \begin{align*}
            \mathbb{P}\Big(|W - W_0| > \zeta_1 \big| \mathcal{D}\Big)  = \mathbb{P}\Big(|W - W_0| > \zeta_1 \big| \mathcal{A}_1, \mathcal{D} \Big) \mathbb{P} \big(\mathcal{A}_1 \big| \mathcal{D} \big) +\mathbb{P}\Big(|W - W_0| > \zeta_1 \big| \mathcal{A}_1 ^ c, \mathcal{D}\Big) \mathbb{P} \big(\mathcal{A}_1 ^ c \big| \mathcal{D} \big).
        \end{align*}
        By Lemma \ref{lem:bound_degree}, $\mathbb{P} \big(\mathcal{A}_1 ^ c \big) \gtrsim O(n^{-10})$. Therefore, it suffices to show that
        \begin{equation*}
            \mathbb{P}\Big(\mathbb{P}\big(|W - W_0| > \zeta_1 \big| \widetilde{\mathcal{D}} \big) > \zeta_2\Big) < \zeta_2 \ \ \text{for some} \ \ \zeta_1 \text{ and }\zeta_2,
        \end{equation*}
        where $\widetilde{\mathcal{D}} = \{\mathcal{A}_1, \mathcal{D}\}$. 
        
        First, we have the following decomposition:
        \begin{align*}
            \big| W_0 - W \big| & \leq \bigg| \sup_{m \in [n], \xb \in \bX} \sqrt{h^d \Xi} \cdot  \bigg| \frac{\mathcal{G}_m(\xb)}{\mathcal{V}^*_m(\xb)}\bigg|  - \sup_{m \in [n], \xb \in \bX} \sqrt{h^d \Xi} \cdot  \bigg|\frac{ \mathcal{G}_{m} (\xb)}{\overline{\mathcal{V}}_{m} (\xb)}\bigg| \bigg| \\
            &\quad + \bigg| \sup_{m \in [n], \xb \in \bX} \sqrt{h^d \Xi} \cdot  \bigg|\frac{ \mathcal{G}_m(\xb)}{ \overline{\mathcal{V}}_{m} (\xb)} \bigg| - \sup_{m \in [n], \xb \in \bX} \sqrt{h^d \Xi} \cdot  \bigg|\frac{ \overline{\mathcal{G}}_{m} (\xb)}{\overline{\mathcal{V}}_{m} (\xb)}\bigg| \bigg| \\
            & \leq  \sup_{m \in [n], \xb \in \bX} \bigg|\sqrt{h^d \Xi} \cdot \mathcal{G}_m(\xb) \bigg| \sup_{m \in [n], \xb \in \bX} \bigg| \frac{\overline{\mathcal{V}}_{m} (\xb) - \mathcal{V}^*_m(\xb)}{\mathcal{V}^*_m(\xb) \overline{\mathcal{V}}_{m} (\xb)} \bigg|  \\
            &\quad +   \sup_{m \in [n], \xb \in \bX} \bigg| \frac{\sqrt{h^d \Xi }}{ \overline{\mathcal{V}}_{m} (\xb)} \Big( \mathcal{G}_m(\xb) - \overline{\mathcal{G}}_{m} (\xb) \Big) \bigg| .
        \end{align*} 
        
        Then, we bound the three parts separately.

        \textbf{Part I: }Consider \[\sup_{m \in [n], \xb \in \bX} \bigg|\sqrt{h^d \Xi} \cdot \mathcal{G}_m(\xb) \bigg|.\] 
        The main idea of the proof is to apply the Borell-TIS inequality \citep{adler2009ineq}. For the validity of the inequality, we want to obtain the upper bound of $\Var \big(\mathcal{G}_m(\xb) \big| \widetilde{\mathcal{D}} \big)$ and the conditional expectation of $\sup_{m \in [n], \xb \in \bX} \bigg|\sqrt{h^d \Xi} \cdot \mathcal{G}_m(\xb) \bigg|$. 

        First, we consider the variance. Note that $\mathbb{E}[\mathcal{G}_m(\xb) \big| \widetilde{\mathcal{D}}] = 0$, and
        \begin{align*} %\label{equ:var_gstar}
            \Var \big(\mathcal{G}_m(\xb) \big| \widetilde{\mathcal{D}} \big) & = \frac{1}{n^2 p^2 L^2} \sum_{i \neq m} \sum_{\ell \in [L_{im}]} \bigg(\mathcal{E}_{im} \mathcal{K}_h(\bX_{im} ^ \ell - \xb) \Big(-y_{im} (\bX_{im} ^ \ell) + \psi \big(\theta_m^*(\bX_{im} ^ \ell) - \theta_i^*(\bX_{im} ^ \ell) \big) \Big) \bigg) ^ 2, \\
            & \leq \frac{1}{n^2 p^2 L^2} \sum_{i \neq m} \sum_{\ell \in [L_{im}]} \mathcal{E}_{im} \mathcal{K}_h^2(\bX_{im} ^ \ell - \xb) = \frac{1}{n^2 p^2 L^2 h^{2d}} \sum_{i \neq m} \sum_{\ell \in [L_{im}]} \mathcal{E}_{im} \mathcal{K}^2\Big(\frac{\bX_{im} ^ \ell - \xb}{h}\Big).
        \end{align*}
        By Assumptions~\ref{ass:kernel_general}--\ref{ass:kernel_smooth}, we consider $\mathcal{K}^2(\cdot)$ as a new kernel and denote $C_{\text{square}} = \int_{\mathbb{R}} \mathcal{K}^2(u) du$. Therefore, we have
        \begin{align*}
            &\sup_{m \in [n], \xb \in \bX} \Var \bigg(\sqrt{h^d \Xi} \cdot \mathcal{G}_m(\xb) \big| \widetilde{\mathcal{D}} \bigg) \\
            &\quad \leq C_{\text{square}} \sup_{m \in [n], \xb \in \bX} \frac{\Xi}{n^2 p^2 L^2 h^d C_{\text{square}}} \sum_{i \neq m} \sum_{\ell \in [L_{im}]} \mathcal{E}_{im} \mathcal{K}^2\Big(\frac{\bX_{im} ^ \ell - \xb}{h}\Big) \\
            &\quad \lesssim \frac{1}{n^2 p^2 L^2} \sup_{m \in [n], \xb \in \bX} \big(\Xi \big) ^ 2 \sup_{m \in [n], \xb \in \bX}  \bigg| \frac{1}{\Xi} \sum_{i \neq m} \sum_{\ell \in [L_{im}]} \mathcal{E}_{im} \Big\{ \frac{1} {h^d C_{\text{square}}} \mathcal{K}^2\Big(\frac{\bX_{im} ^ \ell - \xb}{h}\Big) \Big\}  \bigg|.
        \end{align*}
        First, by Lemma \ref{lem:bound_degree}, we have with probability at least $1 - O(n^{-10})$
        \begin{align*}
            \sup_{m \in [n], \xb \in \bX} \big(\Xi \big) ^ 2 \lesssim (npL) ^ 2.
        \end{align*}
        We apply Lemma \ref{lem:mean_estim} and Lemma \ref{lem:var_estim} on the new kernel $\mathcal{K}^2(\cdot)/C_{\text{square}}$. By Lemma \ref{lem:bound_degree}, we have with probability at least $1 - O(n^{-10})$,
        \begin{equation*}
            \sup_{m \in [n], \xb \in \bX}  \bigg| \frac{1}{\Xi} \sum_{i \neq m} \sum_{\ell \in [L_{im}]} \mathcal{E}_{im} \Big\{ \frac{1} {h^d C_{\text{square}}} \mathcal{K}^2\Big(\frac{\bX_{im} ^ \ell - \xb}{h}\Big) \Big\} - f_{\bX}(\xb)  \bigg| \lesssim h^2 + \sqrt{\frac{\log(nh^{d/2-1})}{npLh^d}}.
        \end{equation*}
        By Assumption \ref{ass:time_distribution} on the bound of $f_{\bX}(\xb)$, we have that for sufficiently large $n$
        \begin{equation} \label{equ:bound_vargstar}
            \sup_{m \in [n], \xb \in \bX} \Var \bigg(\sqrt{h^d \Xi} \cdot \mathcal{G}_m(\xb) \Big| \widetilde{\mathcal{D}} \bigg) \lesssim h^2 + \sqrt{\frac{\log(nh^{d/2-1})}{npLh^d}} \ \text{ with probability} \ 1 - O(n^{-10}),
        \end{equation}
        which completes the proof for obtaining the upper bound of the variance.

        Then, we consider the conditional expectation: 
        \begin{equation} \label{equ:gstar_maximal}
        \begin{split}
            \mathbb{E} \bigg[ \sup_{m \in [n], t \in [0,1]} \Big| \sqrt{h^d \Xi} \cdot \mathcal{G}_m(\xb)  \Big| \bigg| \widetilde{\mathcal{D}} \bigg] \\
            = \mathbb{E} \bigg[ \sup_{m \in [n], t \in [0,1]} \Big| \sqrt{h \sum_{i \neq m} \mathcal{E}_{im} L_{im}} & \sum_{i \neq m} \sum_{\ell \in [L_{im}]} \xi_{im}^\ell \mathcal{E}_{im} K_h(T_{im} ^ \ell - t) \\
            & \times \Big(-y_{im}(T_{im} ^ \ell) + \psi(\widehat{\theta}_m(T_{im} ^ \ell) - \widehat{\theta}_i(T_{im} ^ \ell)) \Big) \Big| \bigg|  \widetilde{\mathcal{D}} \bigg].
        \end{split}
        \end{equation}
        To derive the upper bound of the conditional expectation, we apply Dudley's entropy integral for~$\xi_{im} ^ \ell$ in $\mathcal{G} ^ *(\xb)$ conditioning on $\widetilde{\mathcal{D}}$. 
        Let $\widehat{\mathbb{P}}_{i,m}$ be a probability measure considering the indicators~$\delta_{im} ^ \ell$ as Dirac measure such that
        \begin{equation*}
            \widehat{\mathbb{P}}_{i,m} = \frac{1}{\Xi} \sum_{i \neq m} \sum_{\ell \in [L_{im}]} \mathcal{E}_{im} \delta_{im} ^ \ell,
        \end{equation*}
        where $\delta_{im} ^ \ell$ is the indicator for the respective data $\mathcal{D}_{im} ^ \ell = \big\{y_{im}, \mathcal{E}_{im}, \bX_{im} ^ \ell \big\}$. 
        We denote the function class $\mathcal{H}_h$ by
        \begin{equation*}
            \mathcal{H}_h = \Big\{\mathcal{K}_h(\bX_{im} ^ \ell - \xb) \big(-y_{im}(\bX_{im} ^ \ell) + \psi(\widehat{\theta}_m(\bX_{im} ^ \ell) - \widehat{\theta}_i(\bX_{im} ^ \ell)) \big) \Big| i,m \in [n], \ell \in [L_{im}], \xb \in \bX \Big\}.
        \end{equation*}
        Given $\xi_{im} ^ \ell$ being the only randomness in \eqref{equ:gstar_maximal} with guaranteed sub-Gaussianity, we have that
        \begin{equation} \label{equ:gstar_dudley}
        \begin{split}
                \mathbb{E} \bigg[ \sup_{m \in [n], \xb \in \bX} \Big|  \sqrt{h^d \Xi} \cdot \mathcal{G}_m(\xb)  \Big| \Big| \widetilde{\mathcal{D}}   \bigg] \lesssim  \sqrt{npLh^d} \ \mathbb{E} \bigg[\int_0 ^ {\sigma_{i,m}} \sqrt{\log N \big(\mathcal{H}_h, L^2 ({\widehat{\mathbb{P}}_{i,m}}), \epsilon \big)} \dd \epsilon \Big| \widetilde{\mathcal{D}} \bigg].
        \end{split}
        \end{equation}
        We then bound the covering number, $N \big(\mathcal{H}_h, L^2 ({\widehat{\mathbb{P}}_{i,m}}), \epsilon \big)$. Denote
        \begin{equation*}
            \sigma_{i,m} ^ 2 = \sup_{m \in [n], \xb \in \bX} \widehat{\mathbb{P}}_{i,m} \bigg[\Big(\mathcal{K}_h(\bX_{im} ^ \ell - \xb) \big(-y_{im}(\bX_{im} ^ \ell) + \psi(\widehat{\theta}_m(\bX_{im} ^ \ell) - \widehat{\theta}_i(\bX_{im} ^ \ell)) \big) \Big) ^ 2 \bigg].
        \end{equation*}
        Similar to \eqref{equ:bound_vargstar}, we have  $\sigma_{i,m} ^ 2 = O(1/h^d)$. 
        By $-y_{im}(\bX_{im} ^ \ell) + \psi(\widehat{\theta}_m(\bX_{im} ^ \ell) - \widehat{\theta}_i(\bX_{im} ^ \ell)) \in [-1,1]$, the covering number of $\mathcal{H}_h$ is bounded above such that
        \begin{equation}\label{equ:u_bound_covering_number}
             N\{\mathcal{H}_h, L ^ 2(\widehat{\mathbb{P}}_{i,m}), \epsilon \} \leq \sup_{\mathcal{Q}} \big\{ N\{\mathcal{H}_h, L ^ 2(\mathcal{Q}), \epsilon \} \big\} \leq \bigg( \frac{2 C_\mathcal{K} \|\mathcal{K}\|_{\text{TV}}}{h \epsilon} \bigg) ^ 4,
        \end{equation}
        where $\mathcal{Q}$ is any probability measure. 
        
        Therefore, by \eqref{equ:u_bound_covering_number}, we conclude that
        \begin{equation*}
            \int_0 ^ {\sigma_{i,m}} \sqrt{\log N \big(\mathcal{H}_h, L^2 (\widehat{\mathbb{P}}_{i,m} ), \epsilon \big)} \dd \epsilon \leq \int_0 ^ {\sigma_{i,m}} 2 \sqrt{\log \big(2C_\mathcal{K} \|\mathcal{K}\|_{\text{TV}} \big) - \log \big(h \epsilon \big)} \dd \epsilon.
        \end{equation*}
        By Jensen's inequality and the concavity of $\varphi(\cdot) = \sqrt{\cdot}$, we further have
        \begin{equation} \label{equ:covering_jensen}
            \int_0 ^ {\sigma_{i,m}} \sqrt{\log N \big(\mathcal{H}_h, L^2 (\widehat{\mathbb{P}}_{i,m}), \epsilon \big)} \dd \epsilon \leq 2 \sqrt{\int_0 ^ {\sigma_{i,m}}  \Big( \log \big(2C_\mathcal{K} \|\mathcal{K}\|_{\text{TV}} \big) - \log \big(h \epsilon \big) \Big) \dd \epsilon.}
        \end{equation}
        Integrating out the RHS of \eqref{equ:covering_jensen} gives
        \begin{equation*}
            \int_0 ^ {\sigma_P} \sqrt{\log N \big(\mathcal{H}_h, L^2 (\mathcal{Q}), \epsilon \big)} \dd \epsilon \leq \sqrt{ \sigma_P \big( \log(2C_\mathcal{K} \|\mathcal{K}\|_{\text{TV}}) + 1 - \log (h \sigma_P) \big) },
        \end{equation*}
        where we apply the finite maximal inequality given the sub-Gaussianity. Combining with \eqref{equ:gstar_dudley}, we directly have the desired upper bound of the conditional expectation. 

        Finally, we apply Borell-TIS inequality \citep{adler2009ineq} on $\sup_{m \in [n], \xb \in \bX} \Big| \sqrt{h^d \Xi} \mathcal{G}_m(\xb) \Big|$ conditioning on all $\mathcal{E}_{im}$, $\bX_{im} ^ \ell$, and $y_{im} (\bX_{im} ^ \ell)$. In particular, we have that for any $u > 0$,
        \begin{equation} \label{equ:Borell_gstar}
        \begin{split}
            \mathbb{P} \bigg( &\sup_{m \in [n], \xb \in \bX}  \Big| \sqrt{h^d \Xi} \mathcal{G}_m(\xb) \Big| \geq \mathbb{E}\bigg[ \sup_{m \in [n], \xb \in \bX} \Big| \sqrt{h^d \Xi} \mathcal{G}_m(\xb) \Big| \bigg| \Big\{ y_{im}, \mathcal{E}_{im}, \bX_{im} ^ \ell \Big\} \bigg]  + u \bigg| \Big\{ y_{im}, \mathcal{E}_{im}, \bX_{im} ^ \ell \Big\} \bigg) \\
            & \leq  \exp \bigg(- \frac{u^2}{2 \sup_{m \in [n], \xb \in \bX} \Var \Big( \sqrt{h^d \Xi} \mathcal{G}_m(\xb) \big| \big\{ y_{im},\mathcal{E}_{im}, \bX_{im} ^ \ell \big\} \Big)} \bigg)\\
            & \leq  \exp \bigg(- \frac{u^2}{2 \big( h^2 + \sqrt{\frac{\log(nh^{d/2-1})}{npLh^d}}\big) }  \bigg) .
        \end{split}
        \end{equation}
        
        Applying \eqref{equ:bound_vargstar} and \eqref{equ:gstar_maximal} to \eqref{equ:Borell_gstar}, we have that there exists some $u = O(\sqrt{\log(npL h^{d/2-1})})$, such that 
        \begin{equation} \label{equ:bound_W1A}
            \mathbb{P} \bigg(\sup_{m \in [n], \xb \in \bX} \Big| \sqrt{h^d \Xi} \mathcal{G}_m(\xb) \Big| \gtrsim \sqrt{\log(npL h^{d/2-1})} \bigg| y_{im}, \mathcal{E}_{im}, \bX_{im} ^ \ell  \bigg) \leq \zeta_2 
        \end{equation}
        with probability at least $1 - O(n^{-10})$, where $\zeta_2 = O(n^{-10})$. 
    
        \textbf{Part II: } Consider \[\sup_{m \in [n], \xb \in \bX} \Big| \frac{\overline{\mathcal{V}}_{m} (\xb) - \mathcal{V}^*_m(\xb)}{\mathcal{V}^*_m(\xb) \overline{\mathcal{V}}_{m} (\xb)} \Big|. \]
        
        We have the following decomposition:
        \begin{equation*}
            \sup_{m \in [n], \xb \in \bX} \bigg| \frac{\overline{\mathcal{V}}_{m} (\xb) - \mathcal{V}^*_m(\xb)}{\mathcal{V}^*_m(\xb) \overline{\mathcal{V}}_{m} (\xb)} \bigg| \leq  \sup_{m \in [n], \xb \in \bX} \bigg| \frac{\overline{\mathcal{V}}_{m} (\xb) - \mathcal{V}_m(\xb)}{\mathcal{V}^*_m(\xb) \overline{\mathcal{V}}_{m} (\xb)} \bigg| + \sup_{m \in [n], \xb \in \bX} \bigg| \frac{\mathcal{V}_m(\xb) - \mathcal{V}^*_m(\xb)}{\mathcal{V}^*_m(\xb) \overline{\mathcal{V}}_{m} (\xb)} \bigg|,
        \end{equation*}
        where
        \begin{equation*}
            \overline{\mathcal{V}}_{m} (\xb) - \mathcal{V}_m(\xb) = \frac{1}{npL} \sum_{i \neq m} \sum_{\ell \in [L_{im}]} \mathcal{E}_{im} \mathcal{K}_h(\bX_{im} ^ \ell - \xb) \big(\psi'(\widehat{\theta}_m(\bX_{im} ^ \ell) - \widehat{\theta}_i(\bX_{im} ^ \ell)) - \psi'(\theta^*_m(\bX_{im} ^ \ell) - \theta^*_i(\bX_{im} ^ \ell)) \big).
        \end{equation*}
        
        We then bound these parts separately. For $\big| \overline{\mathcal{V}}_{m} (\xb) - \mathcal{V}_m(\xb) \big|$, by Theorem \ref{thm:estimate}  and Lipschitz property of $\psi'(\cdot)$, we have that
        \begin{equation}\label{equ:psi_prime}
            \sup_{m \in [n],\xb \in \bX}\big| \psi'(\widehat{\theta}_m(\bX_{im} ^ \ell) - \widehat{\theta}_i(\bX_{im} ^ \ell)) - \psi'(\theta^*_m(\bX_{im} ^ \ell) - \theta^*_i(\bX_{im} ^ \ell)) \big|  \lesssim h^2 + \sqrt{\frac{\log(nh^{d/2-1})}{npLh^d}}.
        \end{equation} 
        By Lemmas \ref{lem:mean_estim} and \ref{lem:var_estim} and \eqref{equ:psi_prime}, we have
        \begin{equation} \label{equ:Vbar_V}
        \begin{split}
            &\sup_{m \in [n],\xb \in \bX}\big| \overline{\mathcal{V}}_{m} (\xb) - \mathcal{V}_m(\xb)\big| \\
            &\quad = \frac{1}{npL} \sum_{i \neq m} \sum_{\ell \in [L_{im}]}   \mathcal{E}_{im} \mathcal{K}_h(T_{im}^\ell - t) \big( \psi'(\widehat{\theta}_m(\bX_{im} ^ \ell) - \widehat{\theta}_i(\bX_{im} ^ \ell)) - \psi'(\theta^*_m(\bX_{im} ^ \ell) - \theta^*_i(\bX_{im} ^ \ell)) \big) \\
            &\quad \lesssim \frac{1}{npL} \bigg( h^2 + \sqrt{\frac{\log(nh^{d/2-1})}{npLh^d}} \bigg) \sum_{i \neq m} \sum_{\ell \in [L_{im}]} \mathcal{E}_{im} \mathcal{K}_h(T_{im}^\ell - t)\\
            &\quad \lesssim \bigg( h^2 + \sqrt{\frac{\log(nh^{d/2-1})}{npLh^d}} \bigg)^2.
        \end{split}
        \end{equation}
    
        For \(\mathcal{V}^*_m(\xb) \overline{\mathcal{V}}_{m} (\xb)\), we first have
        \begin{align*}
            \mathcal{V}_{m} ^ * (\xb) \overline{\mathcal{V}}_{m} (\xb) &= (\mathcal{V}_{m} ^ * (\xb)) ^ 2 - \mathcal{V}_{m} ^ * (\xb) (\mathcal{V}_{m} ^ * (\xb) - \overline{\mathcal{V}}_{m} (\xb)). 
        \end{align*}
        By Lemma \ref{lem:lowerbound_Vstar}, $\mathcal{V}_{m} ^ * (\xb) \gtrsim 1.$ By Lemma \ref{lem:bound_VVstar}, we have
        \begin{equation}\label{equ:Vstar_V}
            \sup_{m \in n, \xb \in \Omega} \big|\mathcal{V}^*_{m} (\xb) - \mathcal{V}_{m} (\xb)\big| \lesssim \sqrt{\frac{\log (npLh ^ {d/2 - 1})}{npLh^d}}.
        \end{equation}
        Combining \eqref{equ:Vbar_V}, \eqref{equ:Vstar_V}, we obtain 
        \begin{equation}\label{equ:Vstar_Vbar}
            \sup_{m \in n, \xb \in \Omega} \big|\mathcal{V}^*_{m} (\xb) - \overline{\mathcal{V}}_{m} (\xb) \big| \lesssim \sqrt{\frac{\log (npLh ^ {d/2 - 1})}{npLh^d}}.
        \end{equation}
        Then, by Lemma \ref{lem:mean_estim} and Lemma \ref{lem:var_estim}, we further have
        \begin{equation}\label{equ:Vstar_Vstar_V}
            \mathcal{V}_{m} ^ * (\xb) (\mathcal{V}_{m} ^ * (\xb) - \overline{\mathcal{V}}_{m} (\xb)) \lesssim h^2 + \sqrt{\frac{\log(nh^{d/2-1})}{npLh^d}}.
        \end{equation}
        
        Therefore, by \eqref{equ:Vbar_V} and \eqref{equ:Vstar_Vstar_V}, we conclude that
        \begin{equation*}
            \sup_{m \in [n], \xb \in \bX} \bigg| \frac{\overline{\mathcal{V}}_{m} (\xb) - \mathcal{V}_m(\xb)}{\mathcal{V}^*_m(\xb) \overline{\mathcal{V}}_{m} (\xb)} \bigg| \lesssim h^2 + \sqrt{\frac{\log(nh^{d/2-1})}{npLh^d}}.
        \end{equation*}
        Similarly, we have
        \begin{equation*}
            \sup_{m \in [n], \xb \in \bX} \bigg| \frac{\mathcal{V}_m(\xb) - \mathcal{V}^*_m(\xb)}{\mathcal{V}^*_m(\xb) \overline{\mathcal{V}}_{m} (\xb)} \bigg| \lesssim \sqrt{\frac{\log(npL h^{d/2-1})}{npLh^d}}.
        \end{equation*}
        Combining the  two inequalities above, we conclude that
        \begin{equation} \label{equ:bound_W1B}
            \sup_{m \in [n], \xb \in \bX} \bigg| \frac{\overline{\mathcal{V}}_{m} (\xb) - \mathcal{V} ^ *_m(\xb)}{\mathcal{V}^*_m(\xb) \overline{\mathcal{V}}_{m} (\xb)} \bigg| \lesssim h^2 + \sqrt{\frac{\log(nh^{d/2-1})}{npLh^d}}
        \end{equation}
        with probability at least $1 - O(n^{-10})$. 
    
        Hence, we have with probability at least $1 - O(n^{-10})$,
        \begin{equation*}
        \begin{split}
            \mathbb{P} \bigg( &\sup_{m \in [n], \xb \in \bX} \bigg|\sqrt{h^d \Xi} \cdot \mathcal{G}_m(\xb) \bigg|  \sup_{m \in [n], \xb \in \bX} \bigg| \frac{\overline{\mathcal{V}}_{m} (\xb) - \mathcal{V}^*_m(\xb)}{\mathcal{V}^*_m(\xb) \overline{\mathcal{V}}_{m} (\xb)} \bigg| \\
            & \quad \gtrsim \sqrt{\log(npL h^{d/2-1})} \big( h^2 + \sqrt{\frac{\log(nh^{d/2-1})}{npLh^d}} \big) \bigg| y_{im}, \mathcal{E}_{im}, \bX_{im} ^ \ell  \bigg) \leq \zeta_2,
        \end{split}
        \end{equation*}
        where $\zeta_2 = O(n^{-10})$. 
        
        \textbf{Part III: } Consider
        \begin{equation*}
            \sup_{m \in [n], \xb \in \bX} \bigg|\frac{\sqrt{h^d \Xi}}{\overline{\mathcal{V}}_{m} (\xb)} \Big(\mathcal{G}_m(\xb)  - \overline{\mathcal{G}}_m(\xb)\Big) \bigg|,
        \end{equation*}
        where \begin{equation*}
            \mathcal{G}_m(\xb)  - \overline{\mathcal{G}}_m(\xb) = \frac{1}{npL} \sum_{i \neq m} \sum_{\ell \in [L_{im}]} \xi_{im} ^ \ell \mathcal{E}_{im} \mathcal{K}_h (\bX_{im} ^ \ell - \xb) \Big[  \psi(\theta^*_m(\bX_{im} ^ \ell) - \theta^*_i(\bX_{im} ^ \ell)) - \psi(\widehat{\theta}_m(\bX_{im} ^ \ell) - \widehat{\theta}_i(\bX_{im} ^ \ell))\Big].
        \end{equation*}
        Similarly, we have $\mathbb{E}\big[\mathcal{G}_m(\xb)  - \overline{\mathcal{G}}_m(\xb) \big| \big\{y_{im}, \mathcal{E}_{im}, \bX_{im} ^ \ell \big\} \big] = 0$, and
         \begin{align*} %\label{equ:var_gdiff}
            & \Var \big(\mathcal{G}_m(\xb) - \overline{\mathcal{G}}_m(\xb) \big| \mathcal{E}_{im}, \bX_{im} ^ \ell \big) \\
            &\ \ = \frac{1}{n^2 p^2 L^2} \sum_{i \neq m} \sum_{\ell \in [L_{im}]} \bigg(\mathcal{E}_{im} \mathcal{K}_h(\bX_{im} ^ \ell - \xb) \Big(\psi(\theta^*_m(\bX_{im} ^ \ell) - \theta^*_i(\bX_{im} ^ \ell)) - \psi(\widehat{\theta}_m(\bX_{im} ^ \ell) - \widehat{\theta}_i(\bX_{im} ^ \ell)) \Big) \bigg) ^ 2\\
            &\ \  \lesssim \frac{1}{n^2 p^2 L^2} \bigg(h^2 + \sqrt{\frac{\log(nh^{d/2-1})}{npLh^d}} \bigg) ^ 2 \sum_{i \neq m} \sum_{\ell \in [L_{im}]} \mathcal{E}_{im} \mathcal{K}_h ^ 2(\bX_{im} ^ \ell - \xb).
        \end{align*}
       
        Therefore, we have that
        \begin{equation} \label{equ:bound_vargdiff}
        \begin{split}
            & \sup_{m \in [n], \xb \in \bX}  \Var \bigg(\sqrt{h^d \Xi} \big( \mathcal{G}_m(\xb) - \overline{\mathcal{G}}_m(\xb) \big) \Big| \mathcal{E}_{im}, \bX_{im} ^ \ell \bigg) \\ 
            &\ \  \leq \frac{C_{\text{square}} }{n^2 p^2 L^2} \bigg(h^2 + \sqrt{\frac{\log(nh^{d/2-1})}{npLh^d}} \bigg) ^ 2 \sup_{m \in [n], \xb \in \bX} \frac{\Xi}{h^d C_{\text{square}}} \sum_{i \neq m} \sum_{\ell \in [L_{im}]} \mathcal{E}_{im} \mathcal{K}^2\Big(\frac{\bX_{im} ^ \ell - \xb}{h}\Big) \\
            &\ \  \leq \frac{C_{\text{square}} }{n^2 p^2 L^2} \bigg(h^2 + \sqrt{\frac{\log(nh^{d/2-1})}{npLh^d}} \bigg) ^ 2 \sup_{m \in [n], \xb \in \bX} \big( \Xi \big)^2 \cdot\\
            &\quad  \sup_{m \in [n], \xb \in \bX}  \bigg| \frac{1}{\Xi} \sum_{i \neq m} \sum_{\ell \in [L_{im}]} \mathcal{E}_{im} \Big\{ \frac{1} {h^d C_{\text{square}}} \mathcal{K}^2\Big(\frac{\bX_{im} ^ \ell - \xb}{h}\Big) \Big\}  \bigg| \\
            &\ \  {\lesssim} \bigg(h^2 + \sqrt{\frac{\log(nh^{d/2-1})}{npLh^d}} \bigg) ^ 3, \ \text{ with probability } 1 - O(n^{-10}).
        \end{split}
        \end{equation}
    
        By the finite maximal inequality, we have 
        \begin{equation} \label{equ:gdiff_maximal}
        \begin{split}
            \mathbb{E}\bigg[ \sup_{m \in [n], \xb \in \bX} & \Big| \sqrt{h^d \Xi} \big( \mathcal{G}_m(\xb) - \overline{\mathcal{G}}_m(\xb) \big) \Big|  \bigg| \mathcal{E}_{im}, \bX_{im} ^ \ell \bigg] \\ 
            & \lesssim \sqrt{\sup_{m \in [n], \xb \in \bX} \Var \Big(\sqrt{h^d \Xi}\big( \mathcal{G}_m(\xb) - \overline{\mathcal{G}}_m(\xb) \big)\big| \mathcal{E}_{im}, \bX_{im} ^ \ell \Big) } \sqrt{\log(npL h^{d/2-1})}.
        \end{split}
        \end{equation}
        
        We then apply Borell-TIS inequality \citep{adler2009ineq} for $\sqrt{h^d \Xi} \big(\mathcal{G}_m(\xb)  - \overline{\mathcal{G}}_m(\xb)\big)$ conditional on $\mathcal{E}_{im}$ and $\bX_{im} ^ \ell$. Consequently, we have that for any $u > 0$,
        \begin{equation*} %\label{equ:Borell_gdiff}
        \begin{split}
             \mathbb{P} \bigg(& \sup_{m \in [n], \xb \in \bX}  \Big| \sqrt{h^d \Xi} \big( \mathcal{G}_m(\xb) - \overline{\mathcal{G}}_m(\xb) \bigg) \Big| \\
            &\quad \geq \mathbb{E} \bigg[ \sup_{m \in [n], \xb \in \bX} \Big| \sqrt{h^d \Xi}\big( \mathcal{G}_m(\xb) - \overline{\mathcal{G}}_m(\xb) \big) \Big| \bigg| \mathcal{E}_{im}, \bX_{im} ^ \ell \bigg]  + u \bigg| \mathcal{E}_{im}, \bX_{im} ^ \ell \bigg) \\
            & \leq  \exp \bigg(- \frac{u^2}{2 \sup_{m \in [n], \xb \in \bX} \Var \Big( \sqrt{h^d \Xi} \big( \mathcal{G}_m(\xb) - \overline{\mathcal{G}}_m(\xb) \big) \big| \mathcal{E}_{im}, \bX_{im} ^ \ell \Big)} \bigg).
        \end{split}
        \end{equation*}
        
        By \eqref{equ:bound_vargdiff} and \eqref{equ:gdiff_maximal}, we have that there exists some 
        \begin{equation*}
            u = O\bigg(\sqrt{\log(npL h^{d/2-1})} \Big(h^2 + \sqrt{\frac{\log(nh^{d/2-1})}{npLh^d}} \Big) \bigg),
        \end{equation*}
        such that{\small
        \begin{equation} \label{equ:bound_W2A}
            \mathbb{P} \bigg(\sup_{m \in [n], \xb \in \bX} \Big| \sqrt{h^d \Xi} \big( \mathcal{G}_m(\xb) - \overline{\mathcal{G}}_m(\xb) \big)\Big| \gtrsim \sqrt{\log(npL h^{d/2-1})} \Big(h^2 + \sqrt{\frac{\log(nh^{d/2-1})}{npLh^d}} \Big) \bigg| \mathcal{E}_{im}, \bX_{im} ^ \ell  \bigg) \leq \zeta'_2 
        \end{equation}}
        with probability at least $1 - O(n^{-10})$ and $\zeta'_2 = O(n^{-10})$. By $\widetilde{V}_m(\xb) = O(1)$, we further have with probability at least $1 - O(n^{-10})$,
        \begin{equation} \label{equ:bound_W2}
            \mathbb{P} \bigg(\sup_{m \in [n], \xb \in \bX}  \Big| \frac{\sqrt{h^d \Xi}}{\widetilde{V}_m(\xb)} \big( \mathcal{G}_m(\xb) - \overline{\mathcal{G}}_m(\xb) \big) \Big| \gtrsim \sqrt{\log(npL h^{d/2-1})} \Big( h^2 + \sqrt{\frac{\log(nh^{d/2-1})}{npLh^d}} \Big) \bigg| \mathcal{E}_{im}, \bX_{im} ^ \ell \bigg) \leq \zeta'_2,
        \end{equation}
        which completes the proof for part III. 
        
        Combining \eqref{equ:bound_W1A}, \eqref{equ:bound_W1B}, and \eqref{equ:bound_W2} shows the existence of some $\zeta_1 = o(\sqrt{\log(npL h^{d/2-1})})$ and $\zeta_2 = O(n^{-10})$ such that
        \begin{equation*}
            \mathbb{P}\Big(\mathbb{P}\big(|W - W_0| > \zeta_1 \big|y_{im}, \mathcal{E}_{im}, \bX_{im} ^ \ell\big) > \zeta_2\Big) < \zeta_2,
        \end{equation*}
        where $npLh^{d+4} \leq \log(nh^{d/2-1})$, which completes the proof.
    \end{proof}
    
    \begin{lemma} [Bound for $\big|\mathcal{V}^*_{m}(\xb) - \mathcal{V}_{m} (\xb)\big|$] \label{lem:bound_VVstar}
        Given the conditions in Section \ref{sec:assumptions}, we have that
        \begin{equation*}
            \sup_{m \in n, \xb \in \Omega} \big|\mathcal{V}^*_{m} (\xb) - \mathcal{V}_{m} (\xb)\big| \lesssim \sqrt{\frac{\log (npLh ^ {d/2 - 1})}{npLh^d}},
        \end{equation*}
        with probability at least $1 - O(n ^ {-10})$.
    \end{lemma}
    
    \begin{proof}
        Consider the following decomposition:
        \begin{align*}
            npL \big( \mathcal{V}^*_{m} (\xb) - \mathcal{V}_{m} (\xb) \big) &= \underbrace{\sum_{i \neq m} (\mathcal{E}_{im} - p) \Bigg\{ \sum_{\ell \in [L_{im}]} \Big\{ \mathcal{K}_h(\bX_{im} ^ \ell - \xb) \psi'(\theta_m^*(\bX_{im} ^ \ell) - \theta_i^*(\bX_{im} ^ \ell))  \Big\} \Bigg\}}_{V_{1} ^ {m}(\xb)} \\
            & \quad + \underbrace{p \sum_{i \neq m} \sum_{\ell \in [L_{im}]} \Big\{\psi'(\theta_m^*(\bX_{im} ^ \ell) - \theta_i^*(\bX_{im} ^ \ell)) \big(\mathcal{K}_h(\bX_{im} ^ \ell - \xb) - \mathbb{E}[\mathcal{K}_h(\bX_{im} ^ \ell - \xb) ] \big) \Big\}}_{V_{2} ^ {m}(\xb)}. 
        \end{align*}
        By Assumption \ref{ass:constraint_lps}, there exist constants $c, C$ such that 
        $$0 < c \leq \psi'(\theta_m^*(\bX_{im} ^ \ell) - \theta_i^*(\bX_{im} ^ \ell)) \leq C.$$

        Then, we bound $V_{1} ^ {m}(\xb)$ and $V_{2} ^ {m}(\xb)$ separately. For $V_{1} ^ {m} (\xb)$, for any $m \in [n]$, by  Lemma~\ref{lem:var_estim}, we have
        \begin{equation} \label{equ:V1_bound} 
            \sup_{\xb \in \bX} \big|V_1 ^ {m} (\xb)\big| \lesssim npL  \sqrt{\frac{\log (npLh ^ {d/2 - 1})}{npLh^d}}.
        \end{equation}
         For $V_2 ^ {m} (\xb)$, we apply the proof scheme of Lemma \ref{lem:vrvar_estim}, which yield that for any $m \in [n]$,
        \begin{equation} \label{equ:V2_bound}
            \sup_{\xb \in \bX} \big|V_2 ^ {m} (\xb)\big| \lesssim np^2 L  \sqrt{\frac{\log (npLh ^ {d/2 - 1})}{npLh^d}}.
        \end{equation}
        Combining \eqref{equ:V1_bound} and \eqref{equ:V2_bound}, we have
        \begin{equation*}
            \sup_{m \in [n], \xb \in \bX} \big|\mathcal{V}^*_{m} (\xb) - \mathcal{V}_{m} (\xb)\big| \lesssim \sqrt{\frac{\log (npLh ^ {d/2 - 1})}{npLh^d}}.
        \end{equation*}
    \end{proof}

    \begin{lemma} [Lower Bound of $\mathcal{V}^*(\xb)$] \label{lem:lowerbound_Vstar} 
        Given the assumptions shown in Section \ref{sec:assumptions}, there exists \(C \geq 0 \) such that  $\mathcal{V}_m ^* (\xb) \geq C$ for $m \in [n]$.
    \end{lemma}
    
    \begin{proof}
        Consider
        \begin{equation*}
             \mathcal{V}^*_{m} (\xb) = \mathbb{E}\Big[ \frac{1}{npL} \sum_{i \neq m} \sum_{\ell \in [L_{im}]} \mathcal{E}_{im} \mathcal{K}_h(\bX_{im} ^ \ell - \xb) \psi'(\theta_m^*(\bX_{im} ^ \ell) - \theta_i^*(\bX_{im} ^ \ell)) \Big].
        \end{equation*}
        By Assumption \ref{ass:constraint_lps}, there exists some  constant $K > 0$ such that $|\theta_m^*(\bX_{im} ^ \ell) - \theta_i^*(\bX_{im} ^ \ell)| \leq K$. Therefore, there exists $\psi '(K) > 0$ such that $\psi'(\theta_m^*(\bX_{im} ^ \ell) - \theta_i^*(\bX_{im} ^ \ell)) \geq \psi '(K) > 0$ .
        We then have 
        \begin{equation} \label{equ:Vstar_lowerbound}
        \begin{aligned}
            \mathcal{V}^*_{m} (\xb) &\geq \psi '(K) \mathbb{E}\Big[ \frac{1}{npL} \sum_{i \neq m} \sum_{\ell \in [L_{im}]} \mathcal{E}_{im} \mathcal{K}_h(\bX_{im} ^ \ell - \xb) \Big]  \\
            &\geq  \frac{1}{npL} \psi '(K) \mathbb{E} \bigg[\mathbb{E}\Big[\sum_{i \neq m} \sum_{\ell \in [L_{im}]} \mathcal{E}_{im} \mathcal{K}_h(\bX_{im} ^ \ell - \xb) \Big| \big\{\mathcal{E}_{im} \big\} \Big] \bigg].
        \end{aligned}
        \end{equation}
        By \eqref{equ:Vstar_lowerbound}, we have that
        \begin{align*}
            \mathbb{E} \big[\mathcal{K}_h(\bX_{im} ^ \ell - \xb) \big] & = \frac{1}{h^d} \int_{0} ^ 1  \mathcal{K} \Big( \frac{\tau - \xb}{h} \Big) f_{\bX}(\tau ) \dd \tau \\ 
            & = \int_{-\xb/h} ^ {(1-\xb)/ h} \mathcal{K} (\sigma) f_{\bX}(\sigma h + \xb) \dd \sigma \\
            & \gtrsim \int_{-\xb/h} ^ {(1-\xb)/ h} \mathcal{K} (\sigma) \dd \sigma \\
            & \geq \int_0 ^ {1/h} \mathcal{K} (\sigma) \dd \sigma ,
        \end{align*}
        where the third holds as Assumption \ref{ass:time_distribution} guarantees a lower bound of $f_{\bX}(\cdot)$, and the fourth holds by the symmetricity of $\mathcal{K}(\cdot)$ guaranteed by Assumption \ref{ass:kernel_general}. Therefore, set $h = O(n ^ {- \epsilon / 4})$ for some $\epsilon > 0$, and we have
        \begin{align*}
            \mathcal{V}^*_{m} (\xb)  \geq \frac{1}{npL} \psi '(K) \sum_{i \neq m} \sum_{\ell \in [L_{im}]} \mathbb{E} \bigg[\mathcal{E}_{im} \mathbb{E}\Big[ \mathcal{K}_h(\bX_{im} ^ \ell - \xb) \Big| \big\{\mathcal{E}_{im} \big\} \Big] \bigg] \geq \psi '(K) \int_0 ^ {1/h} \mathcal{K} (\sigma) d \sigma . 
        \end{align*}
        Therefore we have
        \begin{equation*}
            \mathcal{V}_m^*(\xb) \geq C \text{ for any   } m \in [n],
        \end{equation*}
    which completes the proof
    \end{proof}

\section{Proof of Technical Lemmas} \label{sec:proofs}
\subsection{Auxiliary Lemmas for Kernel Regression} \label{sec:kernel_regression}

Note that \citet{tan2021estim} derive properties of the kernel density estimator for \(m\) i.i.d. samples \(\{X_1, X_2, \ldots, X_m\} \sim f_X(\cdot)\) with support on \([0,1]\). A natural generalization for multivariate \(\bX \sim f_{\bX}(\cdot)\) on a multivariate kernel \(\mathcal{K}_h\) is presented in \citet{hardle1997kernel} and \citet{ruppert1994kernel}. Specifically, we define
\begin{equation*}
    w_\xb(\bX_i) = \mathcal{K}_h \big(\bX_i - \xb\big) = K_h \big(\| \bX_i - \xb\|_2 \big); \quad \mathbb{P}_m w_{\xb} = \frac{1}{m} \sum_{i = 1} ^ m \mathcal{K}_h(\bX_i - \xb).
\end{equation*}

To facilitate our discussion, we impose the following assumptions: (1) the distribution of \(\bX_i\) satisfies Assumptions \ref{ass:time_independence}--\ref{ass:time_distribution}; (2) the multivariate kernel function \(\mathcal{K}(\cdot)\) satisfies Assumptions \ref{ass:kernel_general}--\ref{ass:kernel_smooth}. These conditions hold throughout the derivation of the following lemmas. 
First, we bound the mean and variance of \(\mathbb{P}_m w_{\xb}\)  respectively. Lemma \ref{lem:mean_estim} is a restated version Lemma 2 of \citet{tan2021estim}, and Lemma \ref{lem:var_estim} is a restated version of Lemma 3 of \citet{tan2021estim}. We provide them here for the self-completeness. 

\begin{lemma}[Rates on the Mean of the Kernel Density Estimator] \label{lem:mean_estim}
    Follow the notations  in Section \ref{sec:kernel_regression}, and suppose taht  Assumptions \ref{ass:time_independence}--\ref{ass:time_distribution} and \ref{ass:kernel_general}--\ref{ass:kernel_smooth} hold, we have that 
    \begin{equation*}
        \sup_{\xb \in \Omega} \big|\mathbb{E}[\mathbb{P}_m w_\xb] - f_{\bX}(\xb) \big| = O(h^2).
    \end{equation*}
\end{lemma}

\begin{lemma}[Rate on the Variance of the Kernel Density Estimator] \label{lem:var_estim}
    Folow the notations  in Section \ref{sec:kernel_regression}, and suppose that Assumptions \ref{ass:time_independence}--\ref{ass:time_distribution} and \ref{ass:kernel_general}--\ref{ass:kernel_smooth} hold, then for a sufficiently large sample size $m$ such that $m \gg \log(n/h) / h^d$ and any real number $\tau > 0$, we have that 
    \begin{align*}
        \sup_{\xb \in \Omega} \big|\mathbb{P}_m w_{\xb} - \mathbb{E}[\mathbb{P}_m w_{\xb} ] \big| & \lesssim \sqrt{\frac{\log (nh ^ {d/2 - 1})}{mh^d}} \ \ \text{with probability at least} \ \ \  1 - O(n ^ {- \tau}).
    \end{align*}
\end{lemma}

Based on Lemma \ref{lem:mean_estim} and Lemma \ref{lem:var_estim}, we consider a generic framework for a function \( g(\cdot): \mathcal{X} = [0,1]^d \mapsto \mathbb{R} \). Specifically, suppose that \( g(\cdot) \) is twice continuously differentiable with uniformly bounded partial derivatives, denoted as \( \| g \|_{W^{2,\infty}} \leq C \). Suppose there are \( m \) i.i.d. samples \( \{\bX_1, \bX_2, \ldots, \bX_m\} \sim f_{\bX}(\cdot) \) with support \( \Omega \subset \mathcal{X} \). In particular, we consider
\begin{equation}\label{equ_v_Pv}
    v_\xb(\bX_i) = v(\bX_i, \xb) = \mathcal{K}_h(\bX_i - \xb)g(\bX_i), \quad \mathbb{P}_m v_\xb = \frac{1}{m} \sum_{i=1}^m \mathcal{K}_h(\bX_i - \xb)g(\bX_i),
\end{equation}
and
\begin{equation}\label{equ_r_Pr}
    r_\xb(\bX_i) = r(\bX_i, \xb) = \mathcal{K}_h(\bX_i - \xb)\epsilon_i, \quad \mathbb{P}_m r_\xb = \frac{1}{m} \sum_{i=1}^m \mathcal{K}_h(\bX_i - \xb)\epsilon_i.
\end{equation}
Here, \(\epsilon_i\) is independent of \(\bX_i\) and uniformly bounded such that \(\sup_{i \in [m]} |\epsilon_i| \leq C_{\epsilon}\). We also assume that
 \(\mathbb{E}[\epsilon_i | \bX_{1:m}] = 0\). Lemma \ref{lem:vrmean_estim} provides the rate on the bias of the estimator $\mathbb{P}_m v_{\xb}$ by evaluating the difference between its expectation $\mathbb{E} \big[ \mathbb{P}_m v_{\xb} \big]$ and the estimand $f_{\bX}(\xb) g(\xb)$. Lemma \ref{lem:vrvar_estim} evaluates the variance term of $\mathbb{P}_m v_{\xb}$ and $\mathbb{P}_m r_{\xb}$.

\begin{lemma}[Rates on the Bias of $\mathbb{P}_m v_{\xb}$ and $\mathbb{P}_m r_{\xb}$] \label{lem:vrmean_estim}
    Given Assumptions~ \ref{ass:time_independence}--\ref{ass:time_distribution} and \ref{ass:kernel_general}--\ref{ass:kernel_smooth}, we have that 
    \begin{equation*}
        \sup_{\xb \in \Omega} \big|\mathbb{E}[\mathbb{P}_m v_{\xb}] - f_{\bX}(\xb) g(\xb) \big| = O(h ^ 2), \ \ \ \text{ and } \ \ \  \mathbb{E}\big[\mathbb{P}_m r_{\xb} \big] = 0.
    \end{equation*}
\end{lemma}

\begin{proof}
    For the bias of $\mathbb{P}_m v_{\xb}$, we apply Theorem 2.1 of \citet{ruppert1994kernel}, noting that the regularity conditions for the multivariate kernel \(\mathcal{K}\) are satisfied under Assumptions \ref{ass:kernel_general}--\ref{ass:kernel_smooth}. Specifically, using the Taylor expansion from Theorem 1 in \citet{hardle1997kernel}, we obtain
    \begin{equation*}
        \mathbb{E}[\mathbb{P}_m v_{\xb}] - f_{\bX}(\xb) g(\xb) = \mu_2(\mathcal{K}) \big( \nabla g(\xb) \trans \Hb \Hb {\trans} \nabla f_{\bX}(\xb) \big) + \frac{1}{2} \mu_2(\mathcal{K}) f_{\bX} (\xb) \mathrm{tr} \Big\{\Hb {\trans} \mathcal{H}_g(\widetilde{\xb}) \Hb \Big\},
    \end{equation*}
    where \(\widetilde{\xb}\) lies within a Euclidean ball centered at \(\xb\) with radius \(O(h)\). Given that \(\Hb = \Ib_d\) for some \(d = O(1)\), we conclude that
    \begin{equation*}
        \sup_{\xb \in \Omega} \big|\mathbb{E}[\mathbb{P}_m v_{\xb}] - f_{\bX} (\xb)g(\xb) \big| = O(h ^ 2).
    \end{equation*}
    
    For the residual term $\mathbb{P}_m r_{\xb}$, we apply the law of iterative expectations such that for any   $\xb \in \Omega$,
    \begin{align*}
        \mathbb{E}\big[\mathbb{P}_m r_{\xb} \big] &= \mathbb{E} \Big[ \mathcal{K}_h(\bX_1 - \xb) \epsilon_1\Big] \\
        &= \mathbb{E} \Big[ \mathbb{E} \Big[ \mathcal{K}_h(\bX_1 - \xb) \epsilon_1 \Big| \bX_1\Big] \Big] = \mathbb{E} \Big[ \mathcal{K}_h(\bX_1 - \xb) \mathbb{E} \Big[ \epsilon_1 \Big| \bX_1\Big] \Big] = 0,
    \end{align*}
    which completes the proof.
\end{proof}
        
\begin{lemma}[Rate on the Variance of $\mathbb{P}_m v_{\xb}$ and $\mathbb{P}_m r_{\xb}$] \label{lem:vrvar_estim}
    Given Assumptions~ \ref{ass:time_independence}--\ref{ass:time_distribution} and \ref{ass:kernel_general}--\ref{ass:kernel_smooth}, we have that for sample size $m \gg \log (n/h) /h^d$,
    \begin{align*}
        \sup_{\xb \in \Omega} \big|\mathbb{P}_m v_{\xb} - \mathbb{E}[\mathbb{P}_m v_{\xb} ] \big| & \lesssim \sqrt{\frac{\log(n h^{d/2-1})}{mh^d}} \ \ \text{with probability at least} \  1 - O( n^{-\tau} ); \\
        \sup_{\xb \in \Omega} \big |\mathbb{P}_m  r_{\xb} - \mathbb{E}[\mathbb{P}_m r_{\xb} ] \big| = \sup_{\xb \in \Omega} \big|\mathbb{P}_m r_{\xb} \big| & \lesssim \sqrt{\frac{\log(n h^{d/2-1})}{mh^d}} \ \ \text{with probability at least} \ 1 - O( n^{-\tau}).
    \end{align*}
\end{lemma}

\begin{proof}
    The proof to Lemma \ref{lem:vrvar_estim} is inspired by the proofs of Lemma 4 and Lemma 5 of \citet{tan2021estim}. Specifically, we  apply Lemma \ref{lem:bound_supreme} to control
\begin{equation*}
    \mathcal{\mathbb{E}} \bigg\{ \sup_{\xb \in \Omega} \frac{1}{m} \bigg| \sum_{i = 1} ^ m \Big| h(\bX_i; \xb) - \mathbb{E} \big[h(\bX_i; \xb)\big] \Big| \bigg| \bigg\},
\end{equation*}
    with $h(\bX_i; \xb)$ being either $v_{\xb}$ or $r_{\xb}$. We then apply the concentration inequality in Lemma \ref{lem:concentrate_supreme}. For the validity of the concentration inequality, weobtain some positive constants \(\eta\) and \(\tau^2\) such that
    \begin{equation*}
        \sup_{\xb \in \Omega} \big \|h(\cdot; \xb) - \mathbb{E}[h(\cdot; \xb)] \big\|_\infty \leq \eta \ \ \ \text{and} \ \ \sup_{\xb \in \Omega} \Big\{ \frac{1}{m}  \sum_{i = 1} ^ m \text{Var}({h(\bX_i; \xb)}) \Big\} \leq \tau ^ 2.
    \end{equation*}
    Since $v_{\xb} (\cdot)$ and $r_{\xb} (\cdot)$ are both possible candidates for $h(\cdot; \xb)$, we consider the two cases separately.
    
    For $v_{\xb}(\cdot)$, we consider the function class
    \begin{equation*}
        \mathcal{V} = \Big\{v_{\xb}(\bX) - \mathbb{E}[v_{\xb}(\bX)] \big| \xb \in \Omega, \bX \sim f_{\bX}(\cdot) \Big\}.
    \end{equation*}
    For $\sup_{\xb \in \Omega} \big \|v_{\xb}(\cdot) - \mathbb{E}[v_{\xb}(\cdot)] \big\|_\infty$, we have that for any $\xb \in \Omega$,
    \begin{align*}
        \big \|v_{\xb}(\cdot) - \mathbb{E}[v_{\xb}(\cdot)] \big\|_\infty &\leq \big \|v_{\xb}(\cdot) \big\|_{\infty} + \big \|\mathbb{E}[v_{\xb}(\cdot)] \big \|_{\infty} \\
        & \leq \big\|\mathcal{K}_h(\cdot - \xb) g(\cdot) \big\|_{\infty} + \big\|\mathbb{E}[v_{\xb}(\cdot)] \big\|_{\infty} \leq 2 C_g \Big(\frac{1}{h} \|K\|_{\infty} \Big) ^ d + \big\|\mathbb{E}[v_{\xb}(\cdot)] \big\|_{\infty}  \\
        & \overset{(i)}{=} \frac{2C_g \|K\|_{\infty} ^ d}{h^d} + \big\| f_{\Xb} \cdot g \big\|_\infty + O(h^2)  \overset{(ii)}{=} C_1 \frac{\|K\|_{\infty} ^ d}{h^d},
    \end{align*}
    for some positive constant $C_1$. Here $(i)$ holds by Lemma \ref{lem:vrmean_estim} for $\mathbb{E}[v_{\xb}(\cdot)]$ and $(ii)$ holds by Assumption~\ref{ass:time_distribution}. 

    For the variance term, we consider the following decomposition:
    \begin{equation}\label{equ_var_v}
        \begin{split}
            & \sup_{\xb \in \Omega}\text{Var} \big(v_{\xb}(\bX_i) - \mathbb{E}[v_{\xb}(\bX_1)] \big) = \sup_{\xb \in \Omega} \mathbb{E} \big[(v_{\xb}(\bX_i) - \mathbb{E}[v_{\xb}(\bX_1)] ) ^ 2 \big] \\
            & \quad \leq 2 \mathbb{E} \big[v_{\xb}^2(\bX_i) \big] + 2 \big(\mathbb{E}\big[v_{\xb}(\bX_i) \big] \big) ^ 2 \leq \underbrace{\sup_{\xb \in \Omega} 2 \mathbb{E} \big[v_{\xb}^2(\bX_i) \big]}_{\mathcal{Q}_1} + \underbrace{\sup_{\xb \in \Omega} 2\big(\mathbb{E}\big[v_{\xb}(\bX_i) \big] \big) ^ 2}_{\mathcal{Q}_2}.
        \end{split}
    \end{equation}
    We then apply the multivariate integration by substitution resembling \citep{ruppert1994kernel}. Specifically, for any $\xb \in \Omega$,
    \begin{equation}\label{equ_exp_v2}
        \mathbb{E} \big[v_{\xb}^2(\bX_i) \big] \leq \frac{C}{\det (h\mathbf{I}_d)} \|K\|_2^2 = \frac{C}{h^d} \|K\|_2^2 \ \text{ for some positive constant } C.
    \end{equation}
    Furthermore, Lemma \ref{lem:vrmean_estim} guarantees that $\mathcal{Q}_2$ is dominated by $\mathcal{Q}_1$. Therefore, combining \eqref{equ_var_v} and \eqref{equ_exp_v2} implies that
    \begin{equation*}
        \sup_{\xb \in \Omega}\text{Var} \big(v_{\xb}(\bX_i) - \mathbb{E}[v_{\xb}(\bX_1)] \big) \leq \frac{C_2}{h^d} \|K\|_2^2 \ \text{ for some positive constant } C_2.
    \end{equation*}
    Given Assumptions \ref{ass:kernel_general}-\ref{ass:kernel_smooth} for the multivariate kernel function $\mathcal{K}(\cdot)$, Lemma \ref{lem:covering_vx} implies a covering number bound of the function class $\mathcal{H} = \big\{ v_{\xb}(\cdot) | \xb \in \Omega, \bX \sim f_{\bX}(\cdot) \big\}$. In particular, for some positive constant $C$,
    \begin{equation*}
        \sup_{\mathcal{Q}} N(\mathcal{H}, L^2(\mathcal{Q}), \rho) \leq C \Big(\frac{4 C_K \|K\|_{\infty} \|K\|_{\text{TV}}}{h \rho} \Big) ^ {4d},
    \end{equation*}
    for any   probability measure $\mathcal{Q}$ and $\rho \in (0,1)$. We then apply Lemma \ref{lem:bound_supreme} with
    \begin{equation*}
        A = 2 C^{1/4d} C_K \|K\|_{\text{TV}}, \ \ \|F\|_{L^2(\mathbb{P}_m)} = \Big( \frac{2 \|K\|_{\infty}} {h} \Big)^d, \ \ V = 4d, \ \ \sigma_P^2 = \frac{C_2 \|K\|_2^2}{h^d}.
    \end{equation*}
    By $h = o(1)$ and $d = O(1)$, we conclude that
    \begin{equation*}
        \mathbb{E}\bigg\{ \sup_{\xb \in \Omega} \frac{1}{m} \bigg|\sum_{i = 1} ^ m \Big\{v_{\xb}(\bX_i) - \mathbb{E}\big[v_{\xb}(\bX_i) \big] \Big\} \bigg| \bigg\} \lesssim \bigg(\sqrt{\frac{\log (h^{d/2-1})}{mh^d}} + \frac{\log (h^{d/2-1})}{m h^d}\bigg).
    \end{equation*}
    Given the condition on sample size such that $m \gg \log(n/h) /h^d$, we have $\frac{\log (h^{d/2-1})}{m h^d} = o(1)$ and
    \begin{equation*}
        \mathbb{E}\bigg\{ \sup_{\xb \in \Omega} \frac{1}{m} \bigg|\sum_{i = 1} ^ m \Big\{v_{\xb}(\bX_i) - \mathbb{E}\big[v_{\xb}(\bX_i) \big] \Big\} \bigg| \bigg\} \lesssim \sqrt{\frac{\log (h^{d/2-1})}{mh^d}}.
    \end{equation*}
    We then apply Lemma \ref{lem:concentrate_supreme} with
    \begin{equation*}
        \eta = \frac{C_1 \|K\|_\infty^d}{h^d},\ \  \tau ^ 2 = \frac{C_2}{h ^ d} \|K\| ^ 2_2, \ \ \mathbb{E}[Y] \leq C \sqrt{\frac{\log (h^{d/2-1})}{mh^d}},
    \end{equation*}
    which implies that
    \begin{equation*} 
        % \begin{split}
            \mathbb{P} \bigg(  \sup_{\xb \in \Omega} \frac{1}{m} \bigg|\sum_{i = 1} ^ m \Big\{v_{\xb}(\bX_i) - \mathbb{E}\big[v_{\xb}(\bX_i) \big] \Big\} \bigg|  \geq C \bigg\{\sqrt{\frac{\log (h^{d/2-1})}{mh^d}}  + \frac{t}{\sqrt{h ^ d}} \sqrt{1 +  \sqrt{\frac{\log (h^{d/2-1})}{mh^d}}} + \frac{t^2}{h^d}\bigg\} \bigg) \leq e^{-mt^2}.
        % \end{split}
    \end{equation*}
    By taking $t = \sqrt{\frac{\tau \log n}{m}}$, we have
    \begin{align*}
        & \sup_{\xb \in \Omega} \big| \mathbb{P}_m v_{\xb} - \mathbb{E}[\mathbb{P}_m v_{\xb}] \big| = \sup_{\xb \in \Omega} \frac{1}{m} \bigg|\sum_{i = 1} ^ m \Big\{v_{\xb}(\bX_i) - \mathbb{E}\big[v_{\xb}(\bX_i) \big] \Big\} \bigg| \\
        & \lesssim \bigg\{\sqrt{\frac{\log (h^{d/2-1})}{mh^d}}  + \frac{t}{\sqrt{h^d}} \sqrt{1 + \sqrt{\frac{\log (h^{d/2-1})}{mh^d}}} + \frac{t^2}{h ^ d}\bigg\} \\
        & \lesssim \sqrt{\frac{\log(n h^{d/2-1})}{mh^d}}
    \end{align*}
    with probability at least $1- O(n^{-\tau})$ and completes the proof for $v_{\xb}(\cdot)$.

    Then, we prove for the residual term $r_{\xb}(\cdot)$.
    First, the boundedness of $\epsilon_i$ implies that for some positive constant $C_1^*$,
    \begin{equation*}
        \sup_{\xb \in \Omega} \big\|r_{\xb}(\cdot) \big\|_\infty \leq \sup_{\xb \in \Omega} \big\|\mathcal{K}_h(\cdot - \xb) \epsilon_i 
    \big\|_{\infty} \leq \frac{1}{h^d} C_1 ^ * \|K\|^d_\infty.
    \end{equation*}
    For the variance, we apply Theorem 1 of \citet{hardle1997kernel}, and we have
        \begin{equation*}
        \sup_{\xb \in \Omega}\text{Var} \big(r_{\xb}(\bX_i) \big) \leq \frac{C^*_2}{h^d} \|K\|_2^2 \ \text{ for some positive constant } C^*_2.
    \end{equation*}
    Finally, we apply Lemma \ref{lem:covering_rx} for the covering number of the function class $\mathcal{R} = \big \{r_{\xb}(\bX) | \xb \in \Omega, \bX \sim f_{\bX}(\cdot) \big\}$. Specifically, for any probability measure \(\mathcal{Q}\) and \(\rho \in (0,1)\), we have
    \begin{equation*}
        N\{\mathcal{R}, L^2(\mathcal{Q}), \rho \} \leq C^* \Big(\frac{4 C_K \|K\|_{\infty} \|K\|_{\text{TV}}}{h \rho} \Big) ^ {4d},
    \end{equation*}
    where \(C^*\) is some positive constant. Following a similar argument as for \(v_{\xb}(\cdot)\), we obtain
    \begin{align*}
            \sup_{\xb \in \Omega} \big|\mathbb{P}_m r_{\xb} - \mathbb{E}[\mathbb{P}_m r_{\xb}] \big| = \sup_{\xb \in \Omega} \frac{1}{m} \bigg|\sum_{i = 1} ^ m \Big\{r_{\xb}(\bX_i) - \mathbb{E}\big[r_{\xb}(\bX_i) \big] \Big\} \bigg| \lesssim \sqrt{\frac{\log(n h^{d/2-1})}{mh^d}},
    \end{align*}
    with probability at least $1 - O( n ^ {-\tau})$, which completes the proof.
\end{proof}

Lemma \ref{lem:meanvar_erdos} demonstrates the application of Lemmas \ref{lem:mean_estim}, \ref{lem:var_estim}, \ref{lem:vrmean_estim}, and \ref{lem:vrvar_estim} to our setting. 

In what follows, we denote
\begin{equation*}
\begin{split}
    \mathbb{P}_{\mathcal{E}} w_{\xb} &= \frac{1}{\sum_{i \neq m} \mathcal{E}_{im} L_{im}} \sum_{i \neq m} \mathcal{E}_{im} \Big\{ \sum_{\ell = 1} ^ {L_{im}} \mathcal{K}_h(\bX_{im} ^ \ell - \xb) \Big\}, \\
    \mathbb{P}_{\mathcal{E}} v_{\xb} &= \frac{1}{\sum_{i \neq m} \mathcal{E}_{im} L_{im}} \sum_{i \neq m} \mathcal{E}_{im} \Big\{ \sum_{\ell = 1} ^ {L_{im}} \mathcal{K}_h(\bX_{im} ^ \ell - \xb) g(\bX_{im} ^ \ell) \Big\}, \\
    \mathbb{P}_{\mathcal{E}} w_{\xb} &= \frac{1}{\sum_{i \neq m} \mathcal{E}_{im} L_{im}} \sum_{i \neq m} \mathcal{E}_{im} \Big\{ \sum_{\ell = 1} ^ {L_{im}} \mathcal{K}_h(\bX_{im} ^ \ell - \xb) \epsilon_{im} ^ \ell \Big\}.
\end{split}
\end{equation*}

\begin{lemma} [Rates for Erdős–Rényi random graph Random Graph] \label{lem:meanvar_erdos}
    Given the notations and assumptions shown in Sections \ref{sec:assumptions}, we have that
    \begin{equation*}
    \begin{split}
        \sup_{m \in [n]} \sup_{\xb \in \Omega} \Big| \mathbb{P}_{\mathcal{E}} w_{\xb} - f_{\bX} (\xb) \Big| &\lesssim h^2 + \sqrt{\frac{\log (nh^{d/2-1})}{npLh^d}} \text{ with probability at least } 1 - O(n^{-10}); \\
        \sup_{m \in [n]} \sup_{\xb \in \Omega} \Big| \mathbb{P}_{\mathcal{E}} v_{\xb} - f_{\bX} (\xb) g(\xb) \Big| &\lesssim h^2 + \sqrt{\frac{\log (nh^{d/2-1})}{npLh^d}} \text{ with probability at least } 1 - O(n^{-10}); \\
        \sup_{m \in [n]} \sup_{\xb \in \Omega} \big| \mathbb{P}_{\mathcal{E}} r_{\xb} \big| &\lesssim \sqrt{\frac{\log (nh^{d/2-1})}{npLh^d}} \text{ with probability at least } 1 - O(n^{-10}).
    \end{split}
    \end{equation*}
\end{lemma}

\begin{proof}
    To prove Lemma \ref{lem:meanvar_erdos}, we first define the event
    \begin{equation*}
        \mathcal{A} = \bigg\{\sup_{m \in [n]} \sup_{\xb \in \Omega} \Big| \mathbb{P}_{\mathcal{E}} w_{\xb} - f_{\bX} (\xb)\Big| \leq C \Big(h^2 + \sqrt{\frac{\log(nh^{d/2-1})}{npLh^d}} \Big) \text{ for some constant } C \bigg\},
    \end{equation*}
    and  show that \(\mathbb{P}(\mathcal{A}) = 1 - O(n^{-10})\). Specifically, we have that \(\mathcal{A} = \cap_{m = 1}^n \mathcal{A}_m\), where
    \begin{equation*}
        \mathcal{A}_m = \bigg\{\sup_{\xb \in \Omega} \Big| \mathbb{P}_{\mathcal{E}} w_{\xb} - f_{\bX} (\xb)\Big| \leq C \Big(h^2 + \sqrt{\frac{\log(nh^{d/2-1})}{npLh^d}} \Big) \text{ for some constant } C \bigg\}.
    \end{equation*}

    By the union bound, we have
    \[
    \mathbb{P}(\mathcal{A}) = 1 - \mathbb{P}(\mathcal{A}^c) = 1 - \mathbb{P}\left(\bigcup_{m = 1}^n \mathcal{A}_m^c\right) \geq 1 - \sum_{m = 1}^n \mathbb{P}(\mathcal{A}_m^c).
    \]
    Thus, it suffices to show that \(\mathbb{P}(\mathcal{A}_m^c) = O(n^{-11})\) for any \(m \in [n]\).

    Applying the law of total probability, we obtain
    \begin{equation*}
        \mathbb{P}(\mathcal{A}_m^c) = \mathbb{P}\big(\mathcal{A}_m^c  \big| \mathcal{B} \big) \mathbb{P} (\mathcal{B}) + \mathbb{P}\big(\mathcal{A}_m^c  \big| \mathcal{B} ^ c \big) \mathbb{P} (\mathcal{B} ^ c) \leq \mathbb{P}\big(\mathcal{A}_m^c  \big| \mathcal{B} \big) + \mathbb{P} (\mathcal{B} ^ c) = 1 - \mathbb{P}\big(\mathcal{A}_m  \big| \mathcal{B} \big) + \mathbb{P} (\mathcal{B} ^ c),
    \end{equation*}
    where the event \(\mathcal{B}\) is defined as
    \begin{equation*}
        \mathcal{B} = \Big\{\frac{np}{2} \leq \min_{i \in [n]}{d_i} \leq \max_{i \in [n]}{d_i} \leq \frac{3np}{2} \Big\}.
    \end{equation*}

    Lemma \ref{lem:bound_degree} ensures that \(\mathbb{P}(\mathcal{B}^c) = O(n^{-11})\) in the context of the Erdős–Rényi random graph \(\mathcal{G}(\mathcal{V}, \mathcal{E})\). To bound \(\mathbb{P}\big(\mathcal{A}_m \mid \mathcal{B}\big)\), we decompose the supremum into two parts:
    \[
    \sup_{\xb \in \Omega} \left|\mathbb{P}_{\mathcal{E}} w_{\xb} - f_{\bX}(\xb)\right| \leq \sup_{\xb \in \Omega} \left|\mathbb{P}_{\mathcal{E}} w_{\xb} - \mathbb{E}[\mathbb{P}_{\mathcal{E}} w_{\xb}]\right| + \sup_{\xb \in \Omega} \left|\mathbb{E}[\mathbb{P}_{\mathcal{E}} w_{\xb}] - f_{\bX}(\xb)\right|.
    \]
    Define the events
    \begin{align*}
    \mathcal{A}_m^{(1)} &= \left\{\sup_{\xb \in \Omega} \left| \mathbb{P}_{\mathcal{E}} w_{\xb} - \mathbb{E}[\mathbb{P}_{\mathcal{E}} w_{\xb}] \right| \leq C \sqrt{\frac{\log(nh^{d/2-1})}{npLh^d}} \text{ for some constant } C \right\}, \\
    \mathcal{A}_m^{(2)} &= \left\{\sup_{\xb \in \Omega} \left| \mathbb{E}[\mathbb{P}_{\mathcal{E}} w_{\xb}] - f_{\bX}(\xb)\right| \leq C h^2 \text{ for some constant } C \right\}.
    \end{align*}

    Since \(\mathcal{A}_m^{(1)} \cap \mathcal{A}_m^{(2)}\) implies \(\mathcal{A}_m\), we have
    \[
    \mathbb{P}\big(\mathcal{A}_m \mid \mathcal{B}\big) \geq \mathbb{P}\big(\mathcal{A}_m^{(1)} \cap \mathcal{A}_m^{(2)} \mid \mathcal{B}\big),
    \]
    leading to
    \[
    1 - \mathbb{P}\big(\mathcal{A}_m \mid \mathcal{B}\big) \leq \mathbb{P}\big(\mathcal{A}_m^{(1)c} \cup \mathcal{A}_m^{(2)c} \mid \mathcal{B}\big) \leq \mathbb{P}\big(\mathcal{A}_m^{(1)c} \mid \mathcal{B}\big) + \mathbb{P}\big(\mathcal{A}_m^{(2)c} \mid \mathcal{B}\big) = O(n^{-11}).
    \]
    Here, the last equality follows by applying Lemmas \ref{lem:mean_estim} and \ref{lem:var_estim} to the samples \(\sum_{i \neq m} \mathcal{E}_{im} L_{im}\). Thus, for any \(m \in [n]\),
    \begin{equation*}
        \mathbb{P} \big( \mathcal{A}_{m} ^ c \big) \leq \Big( 1 - \mathbb{P} \big(\mathcal{A}_m \big| \mathcal{B} \big) \Big) + \mathbb{P}(\mathcal{B} ^ c) = O(n^{-11}),
    \end{equation*}
    completing the proof for \(\mathbb{P}_{\mathcal{E}} w_{\xb}\). The proofs for \(\mathbb{P}_{\mathcal{E}} v_{\xb}\) and \(\mathbb{P}_{\mathcal{E}} r_{\xb}\) follow similarly by substituting Lemmas \ref{lem:mean_estim} and \ref{lem:var_estim} with Lemmas \ref{lem:vrmean_estim} and \ref{lem:vrvar_estim}, respectively.
\end{proof}
    
\subsection{Auxiliary Lemmas for Covering Number} \label{sec:covering_number}
We present auxiliary lemmas concerning the covering number for specific classes of functions. These lemmas are crucial for satisfying the covering number conditions of the function family \(\mathcal{F}\) in Lemma~\ref{lem:bound_supreme} and for controlling the variance term in kernel regression in Lemma~\ref{lem:var_estim}.

Lemma \ref{lem:covering_bv} is a restated version of Lemma 3 in \citet{gine2009nickl}, and we provide it here for completeness.
\begin{lemma}[Bound on Covering Number for Functions of Bounded Variation] \label{lem:covering_bv}
    Let $K (\cdot) : \mathbb{R} \xrightarrow[]{} \mathbb{R}$ be a function of bounded variation. Define the function class $\mathcal{F}_h = \big\{K \big( \frac{x - \cdot}{h} \big) \big| x \in \mathbb{R} \big \}$. Then there exists $C_K < \infty$ independent of $h$ and $K$ such that for any $0 < \rho < 1$ and probability measure $\mathcal{Q}$, the covering number for $\mathcal{F}_h$ with $L^2(\mathcal{Q})$-norm distance and radius $\rho$ satisfies
    \begin{equation*}
        \sup_{\mathcal{Q}} \{ N\{\mathcal{F}_h, L ^ 2(\mathcal{Q}), \rho \} \} \leq \bigg( \frac{2 C_K \|K\|_{\text{TV}}}{\rho} \bigg) ^ 4,
    \end{equation*}
    where $\|K\|_{\text{TV}}$ is the total variation norm of the function $K$.
\end{lemma}

Furthermore, Lemma \ref{lem:covering_lipschitz} is a restated version of Lemma 14 of \citet{tan2021estim}. It provides a bound of the covering number for a set of Lipschitz functions defined on $[0,1]$.

\begin{lemma}[Bound on Covering Number for Lipschitz Function] \label{lem:covering_lipschitz}
    Let $f(\ell)$ be a Lipschitz function defined on $[0,1]$ such that $|f(\ell) - f(\ell ')| \leq L_f |\ell - \ell'|$ for any $\ell, \ell' \in [0,1]$. $L_f$ is the Lipschitz coefficient. The function class is defined as $\mathcal{F} = \{g_\ell = f(\ell) | \ell \in [0,1] \}$. For any   $\rho \in (0,1)$ and probability measure $\mathcal{Q}$, the covering number of the funcion class $\mathcal{F}$ satisties that
    \begin{equation*}
        N\{\mathcal{F}, L^2 (\mathcal{Q}), \rho \} \leq \frac{L_f}{\rho}.
    \end{equation*}
\end{lemma}

Finally, Lemma \ref{lem:covering_combined} is a restated version of Lemma 15 of \citet{tan2021estim}. It provides a bound on the covering number for combined functional classes under addition and multiplication.

\begin{lemma}[Bound on Covering Number for Combined Functional Classes] \label{lem:covering_combined}
    Let $\mathcal{F}_1$ and $\mathcal{F}_2$ be two function classes satisfying
    \begin{equation*}
        N\{\mathcal{F}_1, L^2(\mathcal{Q}), a_1 \rho \} \leq C_1 \rho ^ {-v_1} \ \ \ \text{and}  \ \ \ N\{\mathcal{F}_2, L^2(\mathcal{Q}), a_2 \rho \} \leq C_2 \rho ^ {-v_2}
    \end{equation*}
    for some $C_1, C_2, a_1, a_2, v_1, v_2 > 0$ and any   $\rho \in (0,1)$. Define that $\|\mathcal{F}_\ell\|_\infty = \sup_{f \in \mathcal{F}_\ell} \|f\|_\infty$ for $\ell \in \{1,2\}$, and $U = \max_{\ell \in \{1,2\}} \{\|\mathcal{F}_\ell\|_\infty\}$. Consider the multiplicative class $\mathcal{F}_{\times} = \{f_1 f_2 | f_1 \in \mathcal{F}_1, f_2 \in \mathcal{F}_2 \}$ and the additive class $\mathcal{F}_{+} = \{f_1 + f_2 | f_1 \in \mathcal{F}_1, f_2 \in \mathcal{F}_2 \}$, we have that for any $\rho \in (0,1)$,
    \begin{equation*}
        N\{\mathcal{F}_\times, L^2(\mathcal{Q}), \rho \} \leq C_1 C_2 \Big( \frac{2a_1U}{\rho}\Big) ^ {v_1} \Big( \frac{2a_2U}{\rho }\Big) ^ {v_2}; \ \ N\{\mathcal{F}_+, L^2(\mathcal{Q}), \rho \} \leq C_1 C_2 \Big( \frac{2a_1}{\rho}\Big) ^ {v_1} \Big( \frac{2a_2}{\rho}\Big) ^ {v_2}.
    \end{equation*}
\end{lemma}

Building on the proving technique used in Lemma \ref{lem:covering_combined}, we establish a bound for the covering number of \textit{multiplicative} kernels, as stated in Lemma \ref{lem:covering_kernel}.

\begin{lemma} [Bound on Covering Number for Multiplicative Kernels] \label{lem:covering_kernel}
    Let \(\mathcal{F}_1\) and \(\mathcal{F}_2\) be function classes defined on the domains \(\mathcal{X}_1\) and \(\mathcal{X}_2\), respectively. Assume that 
    \begin{equation*}
        N\{\mathcal{F}_1, L^2(\mathcal{Q}), a_1 \rho \} \leq C_1 \rho ^ {-v_1} \ \ \ \text{and}  \ \ \ N\{\mathcal{F}_2, L^2(\mathcal{Q}), a_2 \rho \} \leq C_2 \rho ^ {-v_2}
    \end{equation*}
    for some constants \(C_1, C_2, a_1, a_2, v_1, v_2 > 0\), and for any probability measure \(\mathcal{Q}\) and \(\rho \in (0,1)\). Define \(\|\mathcal{F}_\ell\|_\infty = \sup_{f \in \mathcal{F}_\ell} \|f\|_\infty\) for \(\ell \in \{1,2\}\), and let \(U = \max_{\ell \in \{1,2\}} \{\|\mathcal{F}_\ell\|_\infty\}\). Consider the \textit{multiplicative} function class \(\mathcal{F}_{\times}\) defined as
    \begin{equation*}
        \mathcal{F}_{\times} = \Big\{f(x_1,x_2) = f_1(x_1) f_2(x_2) \big| f_1 \in \mathcal{F}_1, f_2 \in \mathcal{F}_2 \Big\}.
    \end{equation*}
    Then, for any \(\rho \in (0,1)\) and any probability measure \(\mathcal{Q}\) on \(\mathcal{X}_1 \times \mathcal{X}_2\), the covering number satisfies
    \begin{equation*}
        N \big\{\mathcal{F}_{\times}, L^2(\mathcal{Q}), \rho \big \} \leq C_1 C_2 \Big( \frac{2a_1U}{\rho}\Big) ^ {v_1} \Big( \frac{2a_2U}{\rho }\Big) ^ {v_2}.
    \end{equation*}
\end{lemma}

\begin{proof}
    To prove Lemma \ref{lem:covering_kernel}, we construct a \(\rho\)-net \(\mathcal{N}_{\times}\) for the function class \(\mathcal{F}_{\times}\). For any \(\rho \in (0,1)\), let \(\mathcal{N}_1 = \{ f_{11}, f_{12}, \dots, f_{1N_1} \}\) and \(\mathcal{N}_2 = \{ f_{21}, f_{22}, \dots, f_{2N_2} \}\) be the \(\rho/(2U)\)-nets of \(\mathcal{F}_1\) and \(\mathcal{F}_2\) with respect to the probability measures \(\mathcal{Q}_1\) and \(\mathcal{Q}_2\), respectively. For any functions \(h_1(\cdot)\) and \(h_2(\cdot)\), we have
    \[
    \int_{\mathcal{X}_1 \times \mathcal{X}_2} h_1(x_1) \, d\mathcal{Q}(x_1, x_2) = \int_{\mathcal{X}_1} h_1(x_1) \, d\mathcal{Q}_1(x_1),
    \]
    and
    \[
    \int_{\mathcal{X}_1 \times \mathcal{X}_2} h_2(x_2) \, d\mathcal{Q}(x_1, x_2) = \int_{\mathcal{X}_2} h_2(x_2) \, d\mathcal{Q}_2(x_2).
    \]
    The covering numbers \(N_1\) and \(N_2\) are therefore bounded that
    \begin{equation}\label{equ:cover_num_N1N2}
        N_1 = N \big\{\mathcal{N}_1, L^2(\mathcal{Q}_1), \rho/2U \big\} \leq C_1 \Big( \frac{2a_1U}{\rho} \Big) ^ {v_1}, \  N_2 = N \big\{\mathcal{N}_2, L^2(\mathcal{Q}_2), \rho/2U \big\} \leq C_2 \Big( \frac{2a_2U}{\rho} \Big) ^ {v_2}.
    \end{equation}
    Now, define the \(\rho\)-net \(\mathcal{N}_{\times}\) for \(\mathcal{F}_{\times}\) as:
    \begin{equation*}
        \mathcal{N}_{\times} = \Big\{f(x_1,x_2) = f_{1} ^ {(j)}(x_1) f_{2} ^ {(k)}(x_2) \big| f_{1} ^ {(j)} \in \mathcal{N}_1, f_{2} ^ {(k)} \in \mathcal{N}_2 \Big\}.
    \end{equation*}

    We next show that \(\mathcal{N}_{\times}\) is indeed a \(\rho\)-net for \(\mathcal{F}_{\times}\). Specifically, for any \(f = f_1(x_1) f_2(x_2) \in \mathcal{F}_{\times}\), there exist functions \(f_{1j} \in \mathcal{N}_1\) and \(f_{2k} \in \mathcal{N}_2\) such that \(\widetilde{f} = f_{1j}(x_1) f_{2k}(x_2) \in \mathcal{N}_{\times}\) satisfies \(\|f_1 - f_{1j}\|_{L^2(\mathcal{Q}_1)} \leq \rho/(2U)\) and \(\|f_2 - f_{2k}\|_{L^2(\mathcal{Q}_2)} \leq \rho/(2U)\).
    Consider the \(L^2(\mathcal{Q})\)-norm of the difference \(f - \widetilde{f}\) that
    \begin{equation*}
        \big\| f - \widetilde{f} \big\|_{L^2(\mathcal{Q})} ^ 2 = \int_{\mathcal{X}_1 \times \mathcal{X}_2} \Big|f_1 (x_1) f_2(x_2) - f_{1j}(x_1) f_{2k}(x_2) \Big|^2 d \mathcal{Q}(x_1, x_2).
    \end{equation*}
    By \eqref{equ:cover_num_N1N2}, we apply the triangle inequality that
    \begin{align*}
        \|f - \widetilde{f}\|_{L^2(\mathcal{Q})}^2 
        &\leq 2 \int_{\mathcal{X}_1 \times \mathcal{X}_2} \left[f_1(x_1)^2 \left(f_2(x_2) - f_{2k}(x_2)\right)^2 + f_{2k}(x_2)^2 \left(f_1(x_1) - f_{1j}(x_1)\right)^2 \right] d\mathcal{Q}(x_1, x_2) \\
        &\leq 2U^2 \left(\|f_1 - f_{1j}\|_{L^2(\mathcal{Q}_1)}^2 + \|f_2 - f_{2k}\|_{L^2(\mathcal{Q}_2)}^2\right) \\
        &\leq 2U^2 \left[\left(\frac{\rho}{2U}\right)^2 + \left(\frac{\rho}{2U}\right)^2\right] = \rho^2.
    \end{align*}

    Therefore, \(\mathcal{N}_{\times}\) is a \(\rho\)-net for \(\mathcal{F}_{\times}\), and the covering number is bounded by
    \begin{equation*}
        N \big\{\mathcal{F}_{\times}, L^2(\mathcal{Q}), \rho \big \} \leq N_1 N_2 = C_1 C_2 \Big( \frac{2a_1U}{\rho}\Big) ^ {v_1} \Big( \frac{2a_2U}{\rho }\Big) ^ {v_2}.
    \end{equation*}
    for any probability measure $\mathcal{Q}$ and $\rho \in (0,1)$, which completes the proof.
\end{proof}

To prove Lemma \ref{lem:var_estim}, we apply Lemma \ref{lem:covering_kernel} for the covering number of the multivariate kernel and generalize Lemma 3 from \citet{tan2021estim}. Similarly, we need to verify the covering number condition for the function class \(\mathcal{R} = \left\{r_{\xb}(\bX) \mid \xb \in \Omega, \bX \sim f_{\bX}(\cdot) \right\}\).

\begin{lemma}[Covering Number Condition for $\mathcal{R}$] \label{lem:covering_rx} 
    Suppose the density function \(f_{\bX}(\cdot)\) satisfies Assumptions \ref{ass:time_independence}--\ref{ass:time_distribution}, and the kernel function \(\mathcal{K}(\cdot)\) satisfies Assumptions \ref{ass:kernel_general}--\ref{ass:kernel_smooth}. Then, for any probability measure \(\mathcal{Q}\) and \(\rho \in (0,1)\),
    \begin{equation*}
        N\{\mathcal{R}, L^2(\mathcal{Q}), \rho \} \leq C_\epsilon^4 \Big(\frac{4 C_K \|K\|_{\infty} \|K\|_{\text{TV}}}{h \rho} \Big) ^ {4d}.
    \end{equation*}
\end{lemma}

\begin{proof}
    Since the kernel \(\mathcal{K}_h\) is \textit{multiplicative}, for \(\bX = \{X_1, X_2, \dots, X_d\}\) and \(\xb = \{x_1, x_2, \dots, x_d\}\), we have
    \begin{equation*}
        \big| r_{\xb}(\bX) \big| =  \big| \mathcal{K}_h (\bX - \xb) \epsilon \big| = \bigg| \Big( \prod_{\ell = 1} ^ d K_h(X_{\ell} - x_\ell) \Big) \epsilon \bigg| \leq C_\epsilon \prod_{\ell = 1} ^ d K_h(X_{\ell} - x_\ell).
    \end{equation*}
    
    Applying Lemma \ref{lem:covering_bv} to each \(K_h(X_{\ell} - x_\ell)\) and Lemma \ref{lem:covering_kernel} with \(U = \|K\|_{\infty}\), we obtain that for any probability measure \(\mathcal{Q}\) and \(\rho \in (0,1)\),
    \begin{equation*}
        N\{\mathcal{R}, L^2(\mathcal{Q}), \rho \} \leq C_\epsilon^4 \bigg( \Big(\frac{2 \|K\|_{\infty} (2 C_K \|K\|_{\text{TV}} )}{h \rho} \Big) ^ {4} \bigg) ^ d = C_\epsilon^4 \Big(\frac{4 C_K \|K\|_{\infty} \|K\|_{\text{TV}}}{h \rho} \Big) ^ {4d}.
    \end{equation*}
\end{proof}

Given that \(h = o(1)\), we establish the covering number condition for the function class \(\mathcal{H} = \left\{ v_{\xb}(\cdot) \mid \xb \in \Omega, \bX \sim f_{\bX}(\cdot) \right\}\), as stated in Lemma \ref{lem:covering_vx}. We assume that the regularity conditions on the density function \(f_{\bX}(\cdot)\) and the kernel function \(K(\cdot)\) hold as described in Section \ref{sec:assumptions}.

\begin{lemma}[Covering Number Condition for $\mathcal{H}$] \label{lem:covering_vx}
    Suppose that the density function \(f_{\bX}(\cdot)\) satisfies Assumptions \ref{ass:time_independence}--\ref{ass:time_distribution} and the kernel function \(\mathcal{K}(\cdot)\) satisfies Assumptions \ref{ass:kernel_general}--\ref{ass:kernel_smooth}. Then, for any probability measure \(\mathcal{Q}\) and \(\rho \in (0,1)\), we have
    \begin{equation*}
         \sup_\mathcal{Q} N\{\mathcal{H}, L^2(\mathcal{Q}), \rho\} \leq 16 C_g^4 \Big(\frac{4 C_K \|K\|_{\infty} \|K\|_{\text{TV}}}{h \rho} \Big) ^ {4d}.
    \end{equation*}
\end{lemma}

\begin{proof}
    To evaluate \(\sup_\mathcal{Q} N\left\{\mathcal{H}, L^2(\mathcal{Q}), \rho\right\}\), we first consider the boundedness of \(g(\cdot)\), such that \(\|g\|_{\infty} \leq C_g\). Consequently, we have
    \begin{equation*}
        \big|v_{\xb}(\bX) \big| = \bigg| \Big( \prod_{\ell = 1} ^ d K_h(X_{\ell} - x_\ell) \Big) \big(g(\bX) - g(\xb) \big) \bigg| \leq 2 C_g \prod_{\ell = 1} ^ d K_h(X_{\ell} - x_\ell).
    \end{equation*}
    We then apply Lemma \ref{lem:covering_bv} to each \(K_h(X_{\ell} - x_\ell)\) and Lemma \ref{lem:covering_kernel} with \(U = \|K\|_{\infty}\), which yields that for any probability measure \(\mathcal{Q}\) and \(\rho \in (0,1)\),
    \begin{equation*}
        N\{\mathcal{H}, L^2(\mathcal{Q}), \rho \} \leq (2C_g)^4 \bigg( \Big(\frac{2 \|K\|_{\infty} (2 C_K \|K\|_{\text{TV}} )}{h \rho} \Big) ^ {4} \bigg) ^ d = 16 C_g^4 \Big(\frac{4 C_K \|K\|_{\infty} \|K\|_{\text{TV}}}{h \rho} \Big) ^ {4d},
    \end{equation*}
    completing the proof.
\end{proof}

In the proof of Theorem~\ref{thm:GMB_validity}, we apply Lemma 5.7 from \citet{vanHandel2014notes}, which requires calculating the covering numbers of \((t, t') \in \bX^2_\epsilon\), where \(\bX^2_\epsilon\) is defined in \eqref{equ:def_Tepsilon}. We present Lemma \ref{lem:covering_ttprime}, which provides the necessary bounds under Assumption \ref{ass:GMB_covering}.

\begin{lemma}[Covering Number for $\bX^2_\epsilon$]  \label{lem:covering_ttprime} 
    The covering number \(N(\bX^2_\epsilon, d_{\|\cdot\|}(\cdot, \cdot), \widetilde{\epsilon})\) is bounded by
    \begin{equation} \label{equ:neq_covering}
    N(\bX^2_\epsilon, d_{\|\cdot\|}(\cdot, \cdot), \widetilde{\epsilon}) \leq N(\bX, \|\cdot\|, \widetilde{\epsilon}/2) \cdot N(B_{\|\cdot\|}(t, \epsilon + \widetilde{\epsilon}), \|\cdot\|, \widetilde{\epsilon}/2),
    \end{equation}
    where \(\bX^2_\epsilon\) is defined in \eqref{equ:def_Tepsilon} and the distance metric \(d_{\|\cdot\|}\) is given by
    \[
    d_{\|\cdot\|}((t_1, t_2), (t'_1, t'_2)) = \|t_1 - t'_1\| + \|t_2 - t'_2\|.
    \]
\end{lemma}
\begin{proof}
    We begin by constructing a \(\widetilde{\epsilon}/2\)-net for \(\bX\), denoted as \(\mathcal{N}_{\widetilde{\epsilon}/2}\). According to the definition of an \(\epsilon\)-net, for every \(\xb \in \bX\), there exists a point \(\pi(t) \in \mathcal{N}_{\widetilde{\epsilon}/2}\) such that
    \[
    \|t - \pi(t)\| \leq \widetilde{\epsilon}/2.
    \]

    Next, we consider the \(\|\cdot\|\)-ball \(B_{\|\cdot\|}(t, \epsilon + \widetilde{\epsilon}) = \{\bx' \in \bX \mid \|t' - t\| \leq \epsilon + \widetilde{\epsilon}\}\). We construct a \(\widetilde{\epsilon}/2\)-cover for this ball, denoted as \(\mathcal{N}_{\widetilde{\epsilon}/2, t}\). By the definition of this cover, for every \(s \in B_{\|\cdot\|}(t, \epsilon + \widetilde{\epsilon})\), there exists a point \(q_t(s) \in \mathcal{N}_{\widetilde{\epsilon}/2, t}\) such that
    \[
    \|s - q_t(s)\| \leq \widetilde{\epsilon}/2.
    \]

    Considering the arbitrary choice of \(t\) and \(s\), we construct the \(\widetilde{\epsilon}\)-net for \(\bX_{\epsilon}^2\) as
    \[
    \mathcal{N}^*_{\widetilde{\epsilon}} = \{(t_0, t'_0) \mid t_0 \in \mathcal{N}_{\widetilde{\epsilon}/2}, t'_0 \in \mathcal{N}_{\widetilde{\epsilon}/2, t_0}\}.
    \]

    For any pair \((t, t') \in \bX_\epsilon^2\), the \(\epsilon\)-net definition guarantees the existence of \(\pi(t) \in \mathcal{N}_{\widetilde{\epsilon}/2}\) and \(q_{\pi(t)}(t') \in \mathcal{N}_{\widetilde{\epsilon}/2, \pi(t)}\). By applying the triangle inequality, we have
    \[
    \|t' - \pi(t)\| \leq \|t' - t\| + \|t - \pi(t)\|.
    \]
    It follows that \(t'\) lies within the \(\widetilde{\epsilon}/2\)-cover \(\mathcal{N}_{\widetilde{\epsilon}/2, \pi(t)}\), ensuring that \(q_{\pi(t)}(t') \in \mathcal{N}_{\widetilde{\epsilon}/2, \pi(t)}\). Consequently, we can identify a pair \((\pi(t), q_{\pi(t)}(t')) \in \mathcal{N}^*_{\widetilde{\epsilon}}\) such that:
    \[
    d_{\|\cdot\|}((t, t'), (\pi(t), q_{\pi(t)}(t'))) = \|t - \pi(t)\| + \|t' - q_{\pi(t)}(t')\| \leq \frac{\widetilde{\epsilon}}{2} + \frac{\widetilde{\epsilon}}{2} = \widetilde{\epsilon}.
    \]

    This construction demonstrates that \(\mathcal{N}^*_{\widetilde{\epsilon}}\) is a \(\widetilde{\epsilon}\)-net for \(\bX_\epsilon^2\), consistent with the definition of an \(\epsilon\)-net. The covering number of \(\mathcal{N}_{\widetilde{\epsilon}/2}\) is bounded as shown in \eqref{equ:neq_covering}, which is consistent with the setup of \(\mathcal{N}^*_{\widetilde{\epsilon}}\).
\end{proof}

\subsection{Auxiliary Lemmas for Concentration Property}

We present key results in the form of concentration inequalities that underpin the derived rates. Lemma \ref{lem:bound_degree} and Lemma \ref{lem:total_degree} establish bounds for vertex degrees in an Erdős–Rényi random graph, leveraging fundamental concentration results such as the Chernoff bound. Lemma \ref{lem:bernstein} introduces a variant of the Bernstein inequality. Lemma \ref{lem:concentrate_supreme} shows that the supremum of certain empirical processes, under mild regularization conditions, converges to its mean. The validity of Lemma \ref{lem:concentrate_supreme} is supported by Lemma \ref{lem:bound_supreme}, which relies on covering entropy conditions discussed in Section \ref{sec:covering_number}.

Lemma \ref{lem:bound_degree} is a restated version of Lemma 1 in \citet{chen2019spectral}. It follows directly from the Chernoff bound, with an additional application of the union bound over \(i \in [n]\). 

\begin{lemma}[Bound on the Degree of Erdős–Rényi Random Graph] \label{lem:bound_degree}
Consider $\mathcal{E}_{ij}$ as the indicator of whether the proposed edge $(i,j)$ exists in the Erdős–Rényi random graph $\mathcal{G}(\mathcal{V}, \mathcal{E})$. The degree of vertex $i$ is denoted as $d_i$ where $d_i = \sum_{j \neq i; j \in [n]} \mathcal{E}_{ij}$. If $p$ satisfies Assumption \ref{ass:erdos_p}, then the event
\begin{equation}
    \mathcal{A}_1 = \bigg\{\frac{np}{2} \leq \min_{i \in [n]} d_i \leq \max_{i \in [n]} d_i \leq \frac{3np}{2} \bigg\}
\end{equation}
occurs with probability at least $1 - O(n^{-10})$.
\end{lemma}

Lemma \ref{lem:total_degree} is a direct corollary of Lemma \ref{lem:bound_degree}.

\begin{lemma}[Bound on the Total Number of Edges of Erdős–Rényi Random Graph] \label{lem:total_degree} The event
\begin{equation*}
    \mathcal{A}_2 = \bigg\{\frac{n^2 p}{4} \leq \sum_{i = 1} ^ n d_i \leq 2n^2 p \bigg\}
\end{equation*}
occurs with probability at least $1 - O(n^{-10})$.
\end{lemma}
\begin{proof}
    We  show that \(\mathcal{A}_1 \subset \mathcal{A}_2\). Assuming that \(\mathcal{A}_1\) holds, it follows that \(\mathcal{A}_2\)  also holds because
    \begin{equation*}
    \frac{n^2 p}{4} < \frac{n^2 p}{2} \leq n \Big\{\min_{i \in [n]} d_i \Big\} \leq \sum_{i = 1} ^ n d_i \leq n \Big\{\max_{i \in [n]} d_i \Big\} \leq \frac{3n^2 p}{2} < 2n^2 p.
    \end{equation*}
    Therefore, \(\mathbb{P}(\mathcal{A}_1) \leq \mathbb{P}(\mathcal{A}_2)\), and Lemma \ref{lem:total_degree} follows directly from Lemma \ref{lem:bound_degree}.

\end{proof}

\begin{remark}
    For Lemma \ref{lem:bound_degree} and Lemma \ref{lem:total_degree}, the specific constant in the tail probability \(O(n^{-10})\) is not crucial. Both lemmas remain valid with a tail probability of \(O(n^{-\tau})\) for any positive constant \(\tau\), provided that the constant \(C_0 = C_0(\tau)\) is sufficiently large.
\end{remark}

Lemma \ref{lem:bernstein} is a version of the Bernstein inequality as presented by \citet{boucheron2013ineq}. For brevity, the proof is omitted.

\begin{lemma}[Bernstein Inequality]\label{lem:bernstein}
    Consider $n$ independent random variables $Z_\ell$, $\ell \in [n]$, each satisfying \(|Z_\ell| \leq B\). For any \(a \geq 2\), the following holds with probability at least \(1 - 2n^{-a}\):
    \[
    \left|\sum_{\ell = 1}^n Z_\ell - \mathbb{E}\left[\sum_{\ell = 1}^n Z_\ell\right]\right| \leq \sqrt{2 a \log n \sum_{\ell = 1}^n \mathbb{E}[Z_\ell^2]} + \frac{2a}{3} B \log n.
    \]
\end{lemma}

Lemmas \ref{lem:concentrate_supreme} and \ref{lem:bound_supreme} introduce a series of inequalities that bound the expectation of the supremum of an empirical process. A typical approach involves first applying Lemma \ref{lem:bound_supreme} to bound \(\mathbb{E}[Y]\) as defined in Lemma \ref{lem:concentrate_supreme}. Then, Lemma \ref{lem:concentrate_supreme}  establishes the  concentration for the supremum of the empirical process. This technique is discussed in \citet{bousquet2002ineq}, \citet{vdgeer2008hd}, \citet{koltchinskii2011ineq}, and \citet{lu2015band}.

\begin{lemma}[Concentration on the Supreme of Empirical Process] \label{lem:concentrate_supreme}
    Let $X_1, \ldots, X_n$ be independent random variables and let $\mathcal{F}$ be a function class such that there exists $\eta$ and $\tau ^ 2$ satisfying
    \begin{equation*}
        \sup_{f \in \mathcal{F}} \|f\|_{\infty} \leq \eta \ \ \ \text{and} \ \ \ \sup_{f \in \mathcal{F}} \frac{1}{n} \sum_{i \in n} \mathrm{Var}\{f(X_i)\} \leq \tau^2.
    \end{equation*}
    Define
    \begin{equation*}
        Y = \sup_{f \in \mathcal{F}} \bigg| \frac{1}{n} \sum_{i \in [n]} [f(X_i) - \mathbb{E}\{f(X_i)\}]\bigg|.
    \end{equation*}
    Then for any $t > 0$,
    \begin{equation*}
        \mathbb{P} \Big[Y \geq \mathbb{E}[Y] + t \sqrt{2 \{\tau ^ 2 + 2 \eta \mathbb{E}[Y]\}} + 2 t ^ 2 \eta / 3 \Big] \leq \exp{(-nt^2)}.
    \end{equation*}
\end{lemma}

\begin{lemma}[Upper Bound on the Expectation of the Supreme of the Empirical Process] \label{lem:bound_supreme}
    Assume that the functions in \(\mathcal{F}\), defined on \(\mathcal{X}\), are uniformly bounded by a constant \(U\). Let \(F(\cdot)\) be the envelope function of \(\mathcal{F}\), such that \(|f(x)| \leq F(x)\) for all \(x \in \mathcal{X}\) and \(f \in \mathcal{F}\). Define \(\sigma_P^2 = \sup_{f \in \mathcal{F}} \mathbb{E}(f^2)\). Let \(X_1, \dots, X_n\) be i.i.d. copies of the random variable \(X\), and denote the empirical measure by \(\mathbb{P}_n = \frac{1}{n} \sum_{i \in [n]} \delta_{X_i}\).

    If for some constants \(A, V > 0\), and for all \(\epsilon > 0\) and \(n \geq 1\), the covering entropy satisfies
    \begin{equation*}
        N\{\mathcal{F}, L^2(\mathbb{P}_n), \epsilon \} \leq \bigg(\frac{A\|F\|_{L^2(\mathbb{P}_n)}}{\epsilon} \bigg) ^ V,
    \end{equation*}
    then for any i.i.d. sub-Gaussian mean-zero random variables \(\xi_1, \dots, \xi_n\), there exists a universal constant \(C\) such that
    \begin{equation*}
        \mathbb{E} \bigg\{\sup_{f \in \mathcal{F}} \frac{1}{n} \bigg|\sum_{i \in [n]} \xi_i f(X_i) \bigg| \bigg\} \leq C \bigg\{\sqrt{\frac{V}{n}} \sigma_P \sqrt{\log \bigg(\frac{A\|F\|_{L^2(\mathbb{P})}}{\sigma_P} \bigg)} + \frac{VU}{n}\log \bigg(\frac{A\|F\|_{L^2(\mathbb{P})}}{\sigma_P}\bigg) \bigg\}.
    \end{equation*}
    Furthermore, we have
    \begin{equation*}
        \mathbb{E} \bigg\{\sup_{f \in \mathcal{F}} \frac{1}{n} \bigg|\sum_{i \in [n]}[f(X_i) - \mathbb{E}[f(X_i)] \bigg| \bigg\} \leq C \bigg\{\sqrt{\frac{V}{n}} \sigma_P \sqrt{\log \bigg(\frac{A\|F\|_{L^2(\mathbb{P})}}{\sigma_P} \bigg)} + \frac{VU}{n}\log \bigg(\frac{A\|F\|_{L^2(\mathbb{P})}}{\sigma_P}\bigg) \bigg\}.
    \end{equation*}
\end{lemma}

In Section \ref{sec:theory}, we require a new version of Bousquet's inequality, as presented in Lemma \ref{lem:bousquet_indep}. This version applies to independent functions \( f_i(\xb, X_i) \) that have uniformly bounded infinite norms and variances, though they may not be identical. Inspired by \citet{bousquet2002ineq}, this version captures the concentration for the maximum of the empirical process \( \sum_{i \in [n]} f_i(\xb, X_i) \) to \( \xb \in \bX \).

\begin{lemma}[Concentration on Independent Empirical Process] \label{lem:bousquet_indep}
    Let \( X_1, \ldots, X_n \) be independent random variables. Assume that there exist uniform constants \( \eta \) and \( \tau ^ 2 \) satisfying
   \begin{equation*}
       \sup_{i \in [n], \xb \in \bX} \|f_i(\xb, X_i)\|_{\infty} \leq \eta \ \ \ \text{and} \ \ \ \sup_{\xb \in \bX} \frac{1}{n} \sum_{i \in [n]} \mathrm{Var}\{f_i(\xb, X_i)\} \leq \tau^2.
   \end{equation*}
   Define
   \begin{equation*}
       Y = \sup_{\xb \in \bX} \bigg| \frac{1}{n} \sum_{i \in [n]} [f_i(\xb, X_i) - \mathbb{E}\{f_i(\xb, X_i)\}]\bigg|.
   \end{equation*}
   Then for any \( u > 0 \),
   \begin{equation} \label{equ:bousquet_epsilon}
       \mathbb{P} \Big[Y \geq \mathbb{E}[Y] + u \sqrt{2 \{\tau ^ 2 + 2 \eta \mathbb{E}[Y]\}} + 2 u ^ 2 \eta / 3 \Big] \leq \exp{(-nu^2)}.
   \end{equation}
\end{lemma}

\begin{proof}
    To prove this lemma, we  verify the conditions of Theorem 2.1 from \citet{bousquet2002ineq}. Without loss of generality, we consider the case where \( \mathbb{E}[f_i] = 0 \) for any \( i \in [n] \). To align our notation with \citet{bousquet2002ineq}, we introduce the leave-one-out supremum for each \( k \in [n] \) with the corresponding parameter \( \xb_k \) defined as
    \begin{equation*}
        Z_k = \sup_{\xb \in \bX} \Bigg| \sum_{i \in [n]; i \neq k} f_i(\xb, X_i) \Bigg|, \ \ \ \xb_k = \argmax_{\xb \in \bX} \Bigg| \sum_{i \in [n]; i \neq k} f_i(\xb, X_i) \Bigg|.
    \end{equation*}
    For the overall supremum, we define
    \begin{equation*}
        Z = \sup_{\xb \in \bX} \Bigg| \sum_{i \in [n]} f_i(\xb, X_i) \Bigg|, \ \ \  \xb^* = \argmax_{\xb \in \bX} \Bigg| \sum_{i \in [n]} f_i(\xb, X_i) \Bigg|.
    \end{equation*}

    The proof requires us to verify conditions (2) and (3) of Theorem 2.1 in \citet{bousquet2002ineq}. For condition (2), under the construction described above, we have that
    \begin{equation}\label{equ:Zk_prime}
        Z_k' = \Bigg| \sum_{i \in [n]} f_i(\xb_k, X_i) \Bigg| - Z_k = \Bigg| \sum_{i \in [n]} f_i(\xb_k, X_i) \Bigg| - \sup_{\xb \in \bX} \Bigg| \sum_{i \in [n]; i \neq k} f_i(\xb, X_i) \Bigg|.
    \end{equation}
    From the definition of the supremum, we have
    \begin{equation}\label{equ:Zk_ineq}
        \Bigg| \sum_{i \in [n]} f_i(\xb_k, X_i) \Bigg| \leq \Bigg| \sum_{i \in [n]} f_i(\xb^*, X_i) \Bigg| \ \ \ \text{and} \ \ \ \Bigg| \sum_{i \in [n]; i \neq k} f_i(\xb^*, X_i) \Bigg| \leq \Bigg| \sum_{i \in [n]; i \neq k} f_i(\xb_k, X_i) \Bigg|.
    \end{equation}
    Applying \eqref{equ:Zk_prime} and \eqref{equ:Zk_ineq} yields that
    \begin{equation*}
        Z_k' \leq \Bigg| \sum_{i \in [n]} f_i(\xb^*, X_i) \Bigg| - \Bigg| \sum_{i \in [n]; i \neq k} f_i(\xb^*, X_i) \Bigg| \leq |f_k(\xb^*, X_k)| \leq \eta \ \ \text{almost surely}.
    \end{equation*}
    We then apply the conditional expectation \( \mathbb{E}^k_n[\cdot] \), defined as \( \mathbb{E}[\cdot \mid X_{i}; i \in [n], i \neq k] \), to \( Z_k' \), and we have
    \begin{align*}
        \mathbb{E}^k_n[Z_k'] & = \mathbb{E}_n^k\Bigg[\Bigg| \sum_{i \in [n]} f_i(\xb_k, X_i) \Bigg|\Bigg] - \mathbb{E}^k_n[Z_k] \\
        &\overset{(i)}{=} \mathbb{E}_n^k\Bigg[\Bigg| \sum_{i \in [n]} f_i(\xb_k, X_i) \Bigg|\Bigg] - Z_k \overset{(ii)}{\geq} \Bigg|\mathbb{E}_n^k \Bigg[ \sum_{i \in [n]} f_i(\xb_k, X_i) \Bigg]\Bigg| - Z_k \overset{(iii)}{=} 0.
    \end{align*}
    Here \( (i) \) holds since \( Z_k \) has no randomness with respect to \( X_k \); \( (ii) \) holds through Jensen's inequality where the function \( \phi(x) = |x| \) is convex, and  \( (iii) \) holds by \( \mathbb{E}[f_k] = 0 \) for any \( k \in [n] \).
    
    For condition (3), we conclude that
    \begin{align*}
        (n - 1)Z &= \Bigg| \sum_{i \in [n]} (n-1) f_i(\xb^*, X_i) \Bigg| \\
        &\overset{(i)}{=}  \Bigg|\sum_{k \in [n]} \sum_{i \in [n], i \neq k} f_i(\xb^*, X_i) \Bigg| \overset{(ii)}{\leq} \sum_{k \in [n]} \Bigg|\sum_{i \in [n], i \neq k} f_i(\xb^*, X_i) \Bigg| \leq \sum_{k = 1}^n Z_k.
    \end{align*}
    Here, (i) follows from the leave-one-out structure; (ii) holds by the triangle inequality, and (iii) holds by the definition of \( Z_k \) as a supremum.
    
    Finally, applying the triangle inequality to \( Z_k' \), we get \( |Z_k'| \leq |f_k(\xb_k, X_k)| \). Thus, a bound \( \sup_{\xb \in \bX} \frac{1}{n} \sum_{i \in [n]} \mathrm{Var}\{f_i(\xb, X_i)\} \leq \tau^2 \) guarantees  the desired rate.
\end{proof}